\documentclass[twoside,11pt]{article}
\PassOptionsToPackage{inline,shortlabels}{enumitem}

\usepackage{macros}

\usepackage{newtxtext}
\usepackage[varvw,cmintegrals,varbb,smallerops]{newtxmath}

\usepackage[
cal=cm,
]
{mathalfa}

\usepackage[nohyperref]{jmlr2e}
\setenumerate{itemsep=\smallskipamount,topsep=\smallskipamount}		

\ifdefined\proof
    \renewenvironment{proof}[1][\proofname]{\par\noindent{\bf #1~ }}{\hfill\BlackBox\\[2mm]}
\else
    \newenvironment{proof}{\par\noindent{\bf Proof\ }}{\hfill\BlackBox\\[2mm]}
\fi

\usepackage[utf8]{inputenc}

\newcommand{\para}[1]{\vspace{0.2em}\textbf{#1\afterhead}~}

\usepackage{cancel}

\usepackage{lastpage}
\jmlrheading{23}{2022}{1-\pageref{LastPage}}{12/20; Revised
7/21}{2/22}{20-1393}{Yu-Guan Hsieh, Franck Iutzeler, Jérôme Malick, and Panayotis Mertikopoulos}
\ShortHeadings{Multi-Agent Online Optimization with Delays}{Hsieh, Iutzeler, Malick, and Mertikopoulos}
\firstpageno{1}

\begin{document}
\allowdisplaybreaks

\title{Multi-Agent Online Optimization with Delays:\\
Asynchronicity, Adaptivity, and Optimism}

\author{\name Yu-Guan Hsieh \email yu-guan.hsieh@univ-grenoble-alpes.fr \\
       \name Franck Iutzeler \email franck.iutzeler@univ-grenoble-alpes.fr \\
       \addr Univ. Grenoble Alpes, LJK, Grenoble, 38000, France
       \AND
       \name Jérôme Malick \email jerome.malick@univ-grenoble-alpes.fr \\
       \addr 
       Univ. Grenoble Alpes, CNRS, Grenoble INP, LJK, 38000 Grenoble, France
       \AND 
    \name Panayotis Mertikopoulos \email panayotis.mertikopoulos@imag.fr \\
       \addr
       Univ. Grenoble Alpes, CNRS, Inria, Grenoble INP, LIG, 38000, Grenoble, France
       \&
       Criteo AI Lab
       }

\editor{Sebastien Bubeck}

\maketitle

\begin{abstract}%

In this paper, we provide a general framework for studying multi-agent online learning problems in the presence of delays and asynchronicities.
Specifically, we propose and analyze a class of adaptive dual averaging schemes in which agents only need to accumulate gradient feedback received from the whole system, without requiring any between-agent coordination.
In the single-agent case, the adaptivity of the proposed method allows us to extend a range of existing results to problems with potentially unbounded delays between playing an action and receiving the corresponding feedback.
In the multi-agent case, the situation is significantly more complicated because agents may not have access to a global clock to use as a reference point;
to overcome this, we focus on the information that is available for producing each prediction rather than the actual delay associated with each feedback.
This allows us to derive adaptive learning strategies with optimal regret bounds, even in a fully \emph{decentralized}, \emph{asynchronous} environment.
Finally, we also analyze an ``optimistic'' variant of the proposed algorithm which is capable of exploiting the predictability of problems with a slower variation and leads to improved regret bounds. \end{abstract}


\renewcommand{\qedsymbol}{{$\blacksquare$}}

\section{Introduction}
\label{sec:intro}

Online learning is a powerful paradigm for sequential decision-making,
with a range of diverse applications in portfolio selection, online auctions, recommender systems, and many other fields;
for a comprehensive introduction to the topic, see the textbooks by \citet{SS11}, \citet{BCB12}, \citet{Hazan16}, and references therein.
In the most basic online learning scenario,
the agent (or ``learner'') chooses an action,
the cost of this action is subsequently revealed to the agent (possibly along with some gradient-based feedback),
and the process repeats.

In this bare-bones model, the time-varying nature of the problem is reflected in the variability of the cost functions encountered by the agent, and the feedback received by the agent is assumed to be immediately available at the end of each time step.
However,
in many cases of practical interest, there could be a significant delay between playing an action and receiving the corresponding feedback;
for instance, this is typically the case in online ad auctions \citep{CAC20}, network traffic routing \citep{ABeA+06}, etc.

Our work concerns online learning setups where delays and asynchronicities play a major role;
these may be due to the computational overhead involved,
the communication latency between different learners in distributed multi-agent systems,
the prediction of long-term effects, or any other reason.
In the literature, the specifics of the delay model are often tailored to the targeted application:
for instance, in online ad placement problems, delays are caused by the lag between the impression of an ad and its conversion, which data suggests are often exponentially distributed \citep{Cha14}.
Instead of zooming in on a particular application, our paper aims at studying the impact of delays and stimulus-response asynchronicities from a generalist, application-agnostic standpoint.
To that end, we propose a flexible framework for distributed online optimization problems in which several agents collaborate asynchronously to enhance their individual/collective performance in an evolving environment with non-zero response times.
This allows us to provide a wide range of regret bounds extending existing results in the literature,
and to design novel adaptive methods that can be implemented in a fully distributed and decentralized manner.

\para{Our contributions in the context of related work}
There are three major underlying themes in our analysis.
As we discussed above, the first has to do with \textbf{delays:}
either due to a computing overhead or an inherent lag between ``action'' and ``reaction'', agents may have to update their actions based on feedback that is potentially stale and obsolete.
The second has to do with \textbf{multi-agent} systems:
in a network setting, 
learners may have to take decisions with very different information at their disposal, and with no realistic means of coordinating their decision-making mechanisms.
Expanding further on this point, the third has to do with \textbf{adaptivity:}
we are interested in learning algorithms that can be run with minimal information prerequisites at the agent end, while still achieving optimal regret bounds.

To take all this into account, we introduce in \cref{sec:framework} a novel, flexible framework that unifies several models of online learning in the presence of delays---including both single- and multi-agent setups.
To achieve no regret in this context, we employ the \acl{DA} template of \citet{Nes09} which we combine with adaptive learning rates inspired by the ``inverse root sum'' blueprint of \citet{ACG02}, \citet{MS10} and \citet{DHS11}.
We show that the resulting policies achieve optimal data- and delay-dependent guarantees even in a fully decentralized environment (\cref{sec:variable}).
In the literature, the closest antecedents to our result are the works of \cite{MS14} and \cite{JGS16,JGS19}, in which the authors also devised adaptive methods for delayed online learning problems.
However, all these papers dealt with the single-agent (shared-memory) setup---and while \citet{JGS19} makes the weakest assumptions among these three papers, the derived bounds are only data-dependent, not delay-dependent.

On the technical side, the multi-agent nature of the problem gives rise to two additional challenges that are not present in the single-agent setup:
\begin{enumerate*}[\itshape i\upshape)]
\item
the non-monotonicity of the total amount of information available to the decision-making agent;
and
\item
the lack of a global counter that indicates the number of updates performed in the entire network so far.
\end{enumerate*}
\footnote{To the best of our knowledge, the only work providing a partial answer to these challenges is that of \citet{JGS19}:
this work takes into account the first challenge and can partly address the second challenge, through an approach that is different from ours. Nonetheless, as mentioned in the previous paragraph, they focused on a setup that is fundamentally different from ours, and the obtained results hence also differ considerably.}
In face of these challenges, we introduce in \cref{sec:DDA} the notion of \emph{dependency graph}, a \ac{DAG} that encodes how the feedback is actually received and used in the network.
Each topological sorting of this \ac{DAG} represents a \emph{faithful permutation} of time that is compatible with the underlying decision-making process.
With help of the dependency graph we also provide a novel characterization
of the key quantities that are involved in the incurred regret.
Taken together, these elements allow us to design and analyze an adaptive algorithm
that achieves optimal data- and delay-dependent regret bounds in this completely decentralized setting.
As a bonus of the new characterizations, we derive for the single-agent setup the first data- and delay-adaptive algorithm that does not require a ``bounded delay'' assumption.

Finally, in \cref{sec:optimistic}, we focus on improving these worst-case bounds by introducing a more ``optimistic'' step-size policy in the spirit of \citet{RS13-NIPS}.
This approach exploits the slow variation of ``predictable'' sequences, thereby improving the regret guarantees of online algorithms.
However, when gradients arrive out of order, the predictability of a loss sequence may be compromised---and, indeed, in the presence of delays, we show that a crude implementation of optimistic methods cannot yield any obvious benefit.
To account for this, we introduce a ``separation of timescales'' between the ``sensing'' and ``updating'' steps of the optimistic \acl{DA} method, and we show that this variable step-size scaling leads to optimal data-dependent guarantees.

\section{A General Framework for Asynchronous Online Optimization}
\label{sec:framework}

In this section, we lay out the general asynchronous online optimization framework that we study throughout the paper.
We also highlight the two challenges that arise in our framework due to its multi-agent nature.

\subsection{Problem Setup}
\label{subsec:framework-prob}

Consider a set of agents $\workers = \{1,\dotsc,\nWorkers\}$ playing against a sequence of time-varying loss functions, with the goal of minimizing their regret.
Formally, at each time slot $\run=\running$, one of the agents becomes \emph{active},
they select an action $\vt[\state]$ from the constraint set $\points$,
and they incur a loss $\vt[\obj](\vt[\state])$.
\footnote{For simplicity, we assume throughout that only one agent is active at each time step.}
The performance of the agents is then measured by the cumulative regret
\begin{equation}
    \label{eq:reg-def}
    \reg_{\nRuns}(\comp) = \sum_{\run=\start}^{\nRuns} \vt[\obj](\vt[\state])
    - \sum_{\run=\start}^{\nRuns} \vt[\obj](\comp)
\end{equation}
where $\comp\in\points$ is an arbitrary comparator action.
In the above, 
$\points$ is assumed to be a closed convex subset of $\vecspace$,
and
each $\vt[\obj]\from\vecspace\to\R\cup\{+\infty\}$ is convex and
subdifferentiable on $\points$.
Unless otherwise stated, we assume that the agents receive first-order feedback $\vt[\gvec]\in\subd\vt[\obj](\vt)$
at some moment after $\vt$ is played
(namely, $\vt[\gvec]$ is a subgradient of $\vt[\obj]$ at $\vt$).
\footnote{In a slight abuse of terminology, the terms gradient and subgradient will be used interchangeably in the sequel.}
Irrespective of the nature of the problem, we will refer to $\vt[\state]$ interchangeably as the \emph{prediction} made by the active agent or the \emph{action} played by the active agent at time $\run$, and we will write $\activeworker{\run}$ for the agent that is active at time $\run$.

For visualization purposes, the above setup is illustrated in \cref{fig:setup}, where we also highlight the fact that we do not put any restriction are how information is exchanged between agents.

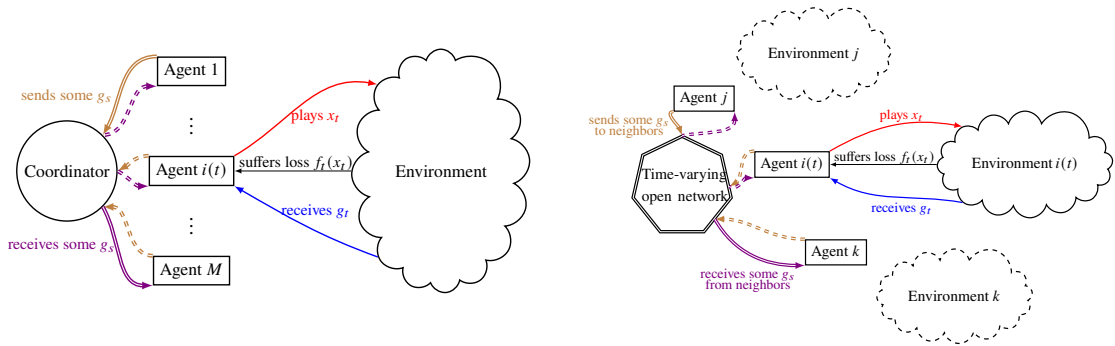
\begin{figure}[t]
    \centering
    \begin{subfigure}[b]{0.47\textwidth}
        \centering
        \resizebox{\textwidth}{!}{
        \begin{tikzpicture}[font=\large]
\tikzset{>=latex}

\node[rectangle,draw,thick] (rect1) at (-0.5,4)  {\small Agent $1$};
\node[rectangle,draw,thick] (rectM) at (-0.5,0)  {\small Agent $M$};
\node at (-0.5,3) {$\vdots$};
\node at (-0.5,1) {$\vdots$};

\node[rectangle,draw,thick] (rect) at (-0.5,2)  {\small Agent $\activeworker{\run}$};

\node[cloud, cloud puffs=15.7, aspect=0.65,align=center, draw,thick] (cloud) at (4.5,2) {\small Environment};

\node[circle,draw,thick] (coord) at (-3,2) {\small Coordinator};

\draw [->,red,thick] (rect.north east) to  [out=30,in=170] (cloud.north west);
\node[red] at (2.0,3.1) {\footnotesize plays $\point_\run$};

\draw [->,black] (cloud.west) to (rect.east);
\node[black] at (1.6,2.15) {\footnotesize suffers loss $\vt[\obj](\vt[\state])$};

\draw [->,blue,thick] (cloud.south west) to [out=160,in=-30] (rect.south east);
\node[blue] at (2.0,1.2) {\footnotesize receives $\vt[\gvec][\run]$ };

\draw [->,brown,double,thick] (rect1.north west) to  [out=180,in=60] (coord.north east);
\draw [<-,violet,double,dashed,thick] (rect1.south west) to  [out=180,in=0] (coord.north east);

\draw [->,brown,double,dashed,thick] (rect.north west) to  [out=180,in=30] (coord.east);
\draw [<-,violet,double,dashed,thick] (rect.south west) to  [out=180,in=-30] (coord.east);

\draw [->,brown,double,dashed,thick] (rectM.north west) to  [out=180,in=0] (coord.south east);
\draw [<-,violet,double,thick] (rectM.south west) to  [out=180,in=-60] (coord.south east);

\node[brown] at (-3,3.5) {\footnotesize sends some $\vt[\gvec][\runalt]$ };
\node[violet] at (-3.1,0.5) {\footnotesize receives some $\vt[\gvec][\runalt]$  };

\end{tikzpicture}         }
        \vspace{0.5em}
    \end{subfigure}
    \hspace{1em}
    \begin{subfigure}[b]{0.47\textwidth}
        \centering
        \resizebox{\textwidth}{!}{
        \begin{tikzpicture}[font=\large]
\tikzset{>=latex}

\node[rectangle,draw,thick] (rect1) at (-2.5,3.5)  {\small Agent $j$};
\node[rectangle,draw,thick] (rectM) at (0.5,0)  {\small Agent $k$};

\node[rectangle,draw,thick] (rect) at (-0.5,2)  {\small Agent $\activeworker{\run}$};

\node[cloud, cloud puffs=15.7, aspect=2,align=center, draw,thick] (cloud) at (4.8,2) {\small Environment $i(t)$};

\node[cloud, cloud puffs=15.7, aspect=2,align=center, draw,thick,dashed] (cloud2) at (0,4.5) {\small Environment $j$};

\node[cloud, cloud puffs=15.7, aspect=2,align=center, draw,thick,dashed] (cloud3) at (3.2,-1.0) {\small Environment $k$};

\node[regular polygon,regular polygon sides=7,thick,draw,text width=1.2cm,double] (coord) at (-3,1.5) {\small };
\node[text width=2.5cm] at (-2.67,1.5) {\small Time-varying open network};

\draw [->,red,thick] (rect.north east) to  [out=30,in=170] (cloud.north west);
\node[red] at (2.0,3.1) {\footnotesize plays $\point_\run$};

\draw [->,black] (cloud.west) to (rect.east);
\node[black] at (1.6,2.15) {\footnotesize suffers loss $\vt[\obj](\vt[\state])$};

\draw [->,blue,thick] (cloud.south west) to [out=170,in=-30] (rect.south east);
\node[blue] at (2.0,1) {\footnotesize receives $\vt[\gvec][\run]$ };

\draw [->,brown,double,dashed,thick] (rect.north west) to  [out=180,in=30] (coord.east);
\draw [<-,violet,double,dashed,thick] (rect.south west) to  [out=180,in=-30] (coord.east);

\draw [->,brown,double,thick] (rect1.south west) to  [out=180,in=110] (coord.north);
\draw [<-,violet,double,dashed,thick] (rect1.south east) to  [out=-90,in=10] (coord.north);

\draw [->,brown,double,dashed,thick] (rectM.north west) to  [out=180,in=0] (coord.south east);
\draw [<-,violet,double,thick] (rectM.south west) to  [out=180,in=-60] (coord.south east);

\node[brown] at (-4.2,3) {\footnotesize sends some $\vt[\gvec][\runalt]$ };
\node[brown] at (-4.2,2.75) {\footnotesize to neighbors };

\node[violet] at (-1.5,-0.5) {\footnotesize receives some $\vt[\gvec][\runalt]$  };
\node[violet] at (-1.5,-0.75) {\footnotesize from neighbors  };

\end{tikzpicture}         }
        \end{subfigure}
    \vspace{-0.6em}
    \caption{Illustration of the considered setup: a network of agents collaborate to minimize the total regret. We do not put any restriction on how the feedback is actually communicated. This can for example be done either through a coordinator-worker structure (left) or a decentralized open network (right).}
    \vspace{-1.4em}
    \label{fig:setup}
\end{figure}

\para{The delay model}
In environments with \emph{delayed feedback},
$\vt[\gvec]$ is only received by all the agents $\worker\in\workers$ a certain amount of time after the generating action $\vt[\state]$ was played.
In this regard, we will focus on the following sources of delay:
\begin{enumerate*}
[\itshape i\upshape)]
\item
\emph{inherent delays} that arise when the effect of a decision requires some time to be observed;
\item
\emph{computation delays} that arise when processing the action takes time (\eg due to gradient computations);
and
\item
\emph{communication delays} that arise in network setups where multiple workers share first-order information among themselves.
\end{enumerate*}

To express this formally, we write $\oneto{\run}\defeq\intinterval{\start}{\run}$
and we write $\vwt[\set]\subseteq\oneto{\run-1}$ for the set of gradient timestamps that are available to agent $\worker$ at time $\run$;
in other words, at time $\run$, the $\worker$-th agent only has $\setdef{\vt[\gvec][\runalt]}{\runalt\in\vwt[\set]}$ at their disposal. 
Clearly, at each stage $\run = \running$, the active agent $\activeworker{\run}$ can only compute $\vt[\state]$ based on $\setdef{\vt[\gvec][\runalt]}{\runalt\in\vwtATt[\set]}$, \ie the set of subgradients available for it at time $\run$.
This quantity is of utmost importance in our framework, so we also define
\begin{equation}
\vt[\set]
	= \vwtATt[\set]
	\qquad
	\text{and}
	\qquad
\vt[\setout]
	=\setexclude{\oneto{\run-1}}{\vt[\set]}
\end{equation}
for the set of timestamps that are available (resp.\;unavailable) to the active agent at time $\run$.

In a slight abuse of terminology, we will refer to both $(\vwt[\set])_{\run\in\oneto{\nRuns}}$ and $(\vt[\set])_{\run\in\oneto{\nRuns}}$ as feedback sequences although, strictly speaking, they only contain the \emph{timestamps} of the corresponding feedback.
Clearly, the non-delayed setting corresponds to the  case $\vt[\set]=\vwt[\set]=\oneto{\run-1}$ and $\vt[\setout]=\varnothing$.

\begin{figure}[t]
\setlength{\leftskip}{0.6cm}
\hspace{-2mm}\textbf{Single-agent ($\nWorkers=1$)}\\[-0.5em]

\resizebox{0.93\textwidth}{!}{
\begin{tikzpicture}
[framed,
background rectangle/.style={fill=gray!5,draw=black},
scale=1.0,
nodestyle/.style={circle,draw=black,thick,double,fill=yellow!20!white, inner sep=2pt},
nodestyleG/.style={,draw=black,fill=black!20!white, inner sep=2pt},
nodestyleGA/.style={,draw=black,thick,double,fill=yellow!20!white, inner sep=2pt},
edgestyle/.style={-,very thick,color=Goldenrod!20!black,opacity=0.4},
edgestyleC/.style={-,very thick, color=blue!20!black,opacity=0.4},
actstyle/.style={,fill=yellow!30!white, inner sep=2pt},
usestyle/.style={dashed,thick, color=white!20!black,opacity=0.4},
>=stealth]

\def\dx{1.6}
\def\dy{1}

\coordinate (T0) at (0*\dx,0);
\coordinate (A0) at (0*\dx,-0.5*\dy);
\coordinate (X0) at (0*\dx,-\dy);
\coordinate (GI0) at (0*\dx,-2.6\dy);
\coordinate (SI0) at (0*\dx,-3.1*\dy);
\coordinate (GII0) at (0*\dx,-2.6\dy);
\coordinate (SII0) at (0*\dx,-3.1*\dy);
\coordinate (S0) at (0*\dx,-5.5*\dy);

\coordinate (T1) at (1*\dx,0);
\coordinate (A1) at (1*\dx,-0.5*\dy);
\coordinate (X1) at (1*\dx,-\dy);
\coordinate (GI1) at (1*\dx,-2.6\dy);
\coordinate (SI1) at (1*\dx,-3.1*\dy);
\coordinate (GII1) at (1*\dx,-2.6\dy);
\coordinate (SII1) at (1*\dx,-3.1*\dy);
\coordinate (S1) at (1*\dx,-5.5*\dy);

\coordinate (T2) at (2*\dx,0);
\coordinate (A2) at (2*\dx,-0.5*\dy);
\coordinate (X2) at (2*\dx,-\dy);
\coordinate (GI2) at (2*\dx,-2.6\dy);
\coordinate (SI2) at (2*\dx,-3.1*\dy);
\coordinate (GII2) at (2*\dx,-2.6\dy);
\coordinate (SII2) at (2*\dx,-3.1*\dy);
\coordinate (S2) at (2*\dx,-5.5*\dy);

\coordinate (T3) at (3*\dx,0);
\coordinate (A3) at (3*\dx,-0.5*\dy);
\coordinate (X3) at (3*\dx,-\dy);
\coordinate (GI3) at (3*\dx,-2.6\dy);
\coordinate (SI3) at (3*\dx,-3.1*\dy);
\coordinate (GII3) at (3*\dx,-2.6\dy);
\coordinate (SII3) at (3*\dx,-3.1*\dy);
\coordinate (S3) at (3*\dx,-5.5*\dy);

\coordinate (T4) at (4*\dx,0);
\coordinate (A4) at (4*\dx,-0.5*\dy);
\coordinate (X4) at (4*\dx,-\dy);
\coordinate (GI4) at (4*\dx,-2.6\dy);
\coordinate (SI4) at (4*\dx,-3.1*\dy);
\coordinate (GII4) at (4*\dx,-2.6\dy);
\coordinate (SII4) at (4*\dx,-3.1*\dy);
\coordinate (S4) at (4*\dx,-5.5*\dy);

\coordinate (T5) at (5*\dx,0);
\coordinate (A5) at (5*\dx,-0.5*\dy);
\coordinate (X5) at (5*\dx,-\dy);
\coordinate (GI5) at (5*\dx,-2.6\dy);
\coordinate (SI5) at (5*\dx,-3.1*\dy);
\coordinate (GII5) at (5*\dx,-2.6\dy);
\coordinate (SII5) at (5*\dx,-3.1*\dy);
\coordinate (S5) at (5*\dx,-5.5*\dy);

\coordinate (T6) at (6*\dx,0);
\coordinate (A6) at (6*\dx,-0.5*\dy);
\coordinate (X6) at (6*\dx,-\dy);
\coordinate (GI6) at (6*\dx,-2.6\dy);
\coordinate (SI6) at (6*\dx,-3.1*\dy);
\coordinate (GII6) at (6*\dx,-2.6\dy);
\coordinate (SII6) at (6*\dx,-3.1*\dy);
\coordinate (S6) at (6*\dx,-5.5*\dy);

\coordinate (Tlast) at (7*\dx,0);
\coordinate (Alast) at (7*\dx,-0.5*\dy);
\coordinate (Xlast) at (7*\dx,-\dy);
\coordinate (GIlast) at (7*\dx,-2.6\dy);
\coordinate (SIlast) at (7*\dx,-3.1*\dy);
\coordinate (GIIlast) at (7*\dx,-2.6\dy);
\coordinate (SIIlast) at (7*\dx,-3.1*\dy);
\coordinate (Slast) at (7*\dx,-5.5*\dy);

\coordinate (Gend) at (7*\dx,-1.5*\dy);

\node (T0) at (T0) [left] {Time $\run$};
\node (X0) at (X0) [left] {Point played $\vt[\state]$};

\node (GI0) at (GI0)  [left,text = black]  {Gradients received};
\node (SI0) at (SI0)  [left,text = black]  {$\vt[\set]$};

\node (T1) at (T1)  {$1$};
\node (X1) at (X1)  [nodestyle] {$\vt[\state][1]$};

\node (GII1) at (GII1) [nodestyleGA] {};
\node (SII1) at (SII1) [actstyle] {$\emptyset$};

\node (T2) at (T2)  {$2$};

\node (X2) at (X2)  [nodestyle] {$\vt[\state][2]$};

\node (GII2) at (GII2) [nodestyleGA] {};
\node (SII2) at (SII2)  [actstyle] {$\emptyset$};

\node (T3) at (T3)  {$3$};
\node (X3) at (X3)  [nodestyle] {$\vt[\state][3]$};

\node (GII3) at (GII3) [nodestyleGA] {$\boldsymbol{\vt[\gvec][1]}$};
\node (SII3) at (SII3)  [actstyle] {$\{1\}$};

\node (T4) at (T4)  {$4$};
\node (X4) at (X4)  [nodestyle] {$\vt[\state][4]$};

\node (GII4) at (GII4) [nodestyleGA] {$\vt[\gvec][1],\boldsymbol{\vt[\gvec][3]}$};
\node (SII4) at (SII4) [actstyle] {$\{1,3\}$};

\node (T5) at (T5)  {$5$};

\node (X5) at (X5)  [nodestyle] {$\vt[\state][5]$};

\node (GII5) at (GII5) [nodestyleGA] {$\vt[\gvec][1],\vt[\gvec][3],\boldsymbol{\vt[\gvec][2]}$};
\node (SII5) at (SII5)  [actstyle] {$\{1,3,2\}$};

\node (Tlast) at (Tlast) {$\dots$};
\node (Xlast) at (Xlast) {$\dots$};
\node (GIlast) at (GIlast) {$\dots$};

\draw [edgestyle,->,out=-45,in=-225] (X1.south east) to (GII3.north west);

\draw [edgestyle,->,out=-45,in=-225] (X2.south east) to (GII5.north west);

\draw [edgestyle,->,out=-45,in=-225] (X3.south east) to (GII4.north west);

\draw [edgestyle,->,out=-45,in=-180] (X4.south east) to (Gend);
\draw [edgestyle,->,out=-45,in=-180] (X5.south east) to (Gend);

\draw [usestyle,->] (GII1.north) to (X1.south);
\draw [usestyle,->] (GII2.north) to (X2.south);
\draw [usestyle,->] (GII3.north) to (X3.south);
\draw [usestyle,->] (GII4.north) to (X4.south);
\draw [usestyle,->] (GII5.north) to (X5.south);

\end{tikzpicture}}
 \\[0.5em]
\hspace*{-2mm}\textbf{Multi-agent ($\nWorkers=2$)}\\[-0.5em]

\resizebox{0.934\textwidth}{!}{
\begin{tikzpicture}
[framed,
background rectangle/.style={fill=gray!5,draw=black},
scale=1.0,
nodestyle/.style={circle,draw=black,thick,double,fill=yellow!20!white, inner sep=2pt},
nodestyleG/.style={,draw=black,fill=black!20!white, inner sep=2pt},
nodestyleGA/.style={,draw=black,thick,double,fill=yellow!20!white, inner sep=2pt},
edgestyle/.style={-,very thick,color=Goldenrod!20!black,opacity=0.4},
edgestyleC/.style={-,very thick, color=green!20!black,opacity=0.4},
actstyle/.style={,fill=yellow!30!white, inner sep=2pt},
actstyleA/.style={,fill=yellow!20!white, inner sep=2pt},
usestyle/.style={dashed,thick, color=white!20!black,opacity=0.4},
>=stealth]

\def\dx{1.6}
\def\dy{1}

\coordinate (T0) at (0*\dx,0);
\coordinate (A0) at (0*\dx,-0.5*\dy);
\coordinate (X0) at (0*\dx,-1.1*\dy);
\coordinate (GI0) at (0*\dx,-2.6\dy);
\coordinate (SI0) at (0*\dx,-3.1*\dy);
\coordinate (GII0) at (0*\dx,-4.1\dy);
\coordinate (SII0) at (0*\dx,-4.6*\dy);
\coordinate (S0) at (0*\dx,-5.5*\dy);

\coordinate (T1) at (1*\dx,0);
\coordinate (A1) at (1*\dx,-0.5*\dy);
\coordinate (X1) at (1*\dx,-1.1\dy);
\coordinate (GI1) at (1*\dx,-2.6\dy);
\coordinate (SI1) at (1*\dx,-3.1*\dy);
\coordinate (GII1) at (1*\dx,-4.1\dy);
\coordinate (SII1) at (1*\dx,-4.6*\dy);
\coordinate (S1) at (1*\dx,-5.5*\dy);

\coordinate (T2) at (2*\dx,0);
\coordinate (A2) at (2*\dx,-0.5*\dy);
\coordinate (X2) at (2*\dx,-1.1\dy);
\coordinate (GI2) at (2*\dx,-2.6\dy);
\coordinate (SI2) at (2*\dx,-3.1*\dy);
\coordinate (GII2) at (2*\dx,-4.1\dy);
\coordinate (SII2) at (2*\dx,-4.6*\dy);
\coordinate (S2) at (2*\dx,-5.5*\dy);

\coordinate (T3) at (3*\dx,0);
\coordinate (A3) at (3*\dx,-0.5*\dy);
\coordinate (X3) at (3*\dx,-1.1\dy);
\coordinate (GI3) at (3*\dx,-2.6\dy);
\coordinate (SI3) at (3*\dx,-3.1*\dy);
\coordinate (GII3) at (3*\dx,-4.1\dy);
\coordinate (SII3) at (3*\dx,-4.6*\dy);
\coordinate (S3) at (3*\dx,-5.5*\dy);

\coordinate (T4) at (4*\dx,0);
\coordinate (A4) at (4*\dx,-0.5*\dy);
\coordinate (X4) at (4*\dx,-1.1\dy);
\coordinate (GI4) at (4*\dx,-2.6\dy);
\coordinate (SI4) at (4*\dx,-3.1*\dy);
\coordinate (GII4) at (4*\dx,-4.1\dy);
\coordinate (SII4) at (4*\dx,-4.6*\dy);
\coordinate (S4) at (4*\dx,-5.5*\dy);

\coordinate (T5) at (5*\dx,0);
\coordinate (A5) at (5*\dx,-0.5*\dy);
\coordinate (X5) at (5*\dx,-1.1\dy);
\coordinate (GI5) at (5*\dx,-2.6\dy);
\coordinate (SI5) at (5*\dx,-3.1*\dy);
\coordinate (GII5) at (5*\dx,-4.1\dy);
\coordinate (SII5) at (5*\dx,-4.6*\dy);
\coordinate (S5) at (5*\dx,-5.5*\dy);

\coordinate (T6) at (6*\dx,0);
\coordinate (A6) at (6*\dx,-0.5*\dy);
\coordinate (X6) at (6*\dx,-1.1\dy);
\coordinate (GI6) at (6*\dx,-2.6\dy);
\coordinate (SI6) at (6*\dx,-3.1*\dy);
\coordinate (GII6) at (6*\dx,-4.1\dy);
\coordinate (SII6) at (6*\dx,-4.6*\dy);
\coordinate (S6) at (6*\dx,-5.5*\dy);

\coordinate (Tlast) at (7*\dx,0);
\coordinate (Alast) at (7*\dx,-0.5*\dy);
\coordinate (Xlast) at (7*\dx,-1.1\dy);
\coordinate (GIlast) at (7*\dx,-2.6\dy);
\coordinate (SIlast) at (7*\dx,-3.1*\dy);
\coordinate (GIIlast) at (7*\dx,-4.1\dy);
\coordinate (SIIlast) at (7*\dx,-4.6*\dy);
\coordinate (Slast) at (7*\dx,-5.5*\dy);

\coordinate (Gend) at (7*\dx,-2.3*\dy);
\coordinate (Aux1) at (3.5*\dx,-3.4*\dy);
\coordinate (Aux2) at (4.1*\dx,-3.4*\dy);

\node (T0) at (T0) [left] {Time $\run$};
\node (A0) at (A0) [left] {Active agent $i(\run)$};
\node (X0) at (X0) [left] {Point played $\vt[\state]$};

\node (GI0) at (GI0)  [left,text = black]  {Gradients received by $1$};
\node (SI0) at (SI0)  [left,text = black]  {$\vwt[\set][1]$};

\node (GII0) at (GII0)  [left,text = black]  {Gradients received by $2$};
\node (SII0) at (SII0)  [left,text = black]  {$\vwt[\set][2]$};
\node (S0) at (S0)  [left,text = black]  {$\vt[\set]$};

\node (T1) at (T1)  {$1$};
\node (A1) at (A1)  [actstyleA] {$2$};
\node (X1) at (X1)  [nodestyle] {$\vt[\state][1]$};

\node (GI1) at (GI1) [nodestyleG] {};
\node (SI1) at (SI1) {$\emptyset$};

\node (GII1) at (GII1) [nodestyleGA] {};
\node (SII1) at (SII1) [actstyle] {$\emptyset$};

\node (S1) at (S1)  [actstyle] {$\emptyset$};

\node (T2) at (T2)  {$2$};
\node (A2) at (A2) [actstyleA] {$1$};
\node (X2) at (X2)  [nodestyle] {$\vt[\state][2]$};

\node (GI2) at (GI2) [nodestyleGA] {};
\node (SI2) at (SI2) [actstyle] {$\emptyset$};

\node (GII2) at (GII2) [nodestyleG] {};
\node (SII2) at (SII2) {$\emptyset$};

\node (S2) at (S2) [actstyle] {$\emptyset$};

\node (T3) at (T3)  {$3$};
\node (A3) at (A3) [actstyleA] {$1$};
\node (X3) at (X3)  [nodestyle] {$\vt[\state][3]$};

\node (GI3) at (GI3) [nodestyleGA] {$\boldsymbol{\vt[\gvec][2]}$};
\node (SI3) at (SI3) [actstyle] {$\{2\}$};

\node (GII3) at (GII3) [nodestyleG] {$\boldsymbol{\vt[\gvec][1]}$};
\node (SII3) at (SII3) {$\{1\}$};

\node (S3) at (S3) [actstyle] {$\{2\}$};

\node (T4) at (T4)  {$4$};
\node (A4) at (A4)  [actstyleA] {$2$};
\node (X4) at (X4)  [nodestyle] {$\vt[\state][4]$};

\node (GI4) at (GI4) [nodestyleG] {$\vt[\gvec][2],\boldsymbol{\vt[\gvec][3]}$};
\node (SI4) at (SI4) {$\{2,3\}$};

\node (GII4) at (GII4) [nodestyleGA] {$\vt[\gvec][1]$};
\node (SII4) at (SII4) [actstyle] {$\{1\}$};

\node (S4) at (S4)   [actstyle] {$\{1\}$};

\node (T5) at (T5)  {$5$};
\node (A5) at (A5)  [actstyleA] {$1$};
\node (X5) at (X5)  [nodestyle] {$\vt[\state][5]$};

\node (GI5) at (GI5) [nodestyleGA] {$\vt[\gvec][2],\vt[\gvec][3],\boldsymbol{\vt[\gvec][1]}$};
\node (SI5) at (SI5) [actstyle] {$\{2,3,1\}$};

\node (GII5) at (GII5) [nodestyleG] {$\vt[\gvec][1],\boldsymbol{\vt[\gvec][2]},\boldsymbol{\vt[\gvec][3]}$};
\node (SII5) at (SII5) {$\{1,2,3\}$};

\node (S5) at (S5) [actstyle] {$\{2,3,1\}$};

\node (Tlast) at (Tlast) {$\dots$};
\node (Xlast) at (Xlast) {$\dots$};
\node (GIlast) at (GIlast) {$\dots$};
\node (GIIlast) at (GIIlast) {$\dots$};
\node (Slast) at (Slast) {$\dots$};

\draw [edgestyle,->,out=-45,in=-225] (X1.south east) to (GII3.north west);
\draw [edgestyleC,->,out=35,in=255] (GII3.north east) to (GI5.south west);

\draw [edgestyle,->,out=-45,in=-225] (X2.south east) to (GI3.north west);
\draw [edgestyleC,->,out=-55,in=-215] (GI3.south east) to (GII5.north west);

\draw [edgestyle,->,out=-45,in=-225] (X3.south east) to (GI4.north west);
\draw [edgestyleC,->,out=-80,in=-215] (GI4.south east) to (GII5.north west);

\draw [edgestyle,->,out=-45,in=-190] (X4.south east) to (Gend);
\draw [edgestyle,->,out=-45,in=-190] (X5.south east) to (Gend);

\draw [usestyle,->,out=110,in=-110] (GII1.north) to (X1.south);
\draw [usestyle,->] (GI2.north) to (X2.south);
\draw [usestyle,->] (GI3.north) to (X3.south);
\draw [usestyle,->,out=140,in=-120] (GII4.north west) to (X4.south west);
\draw [usestyle,->] (GI5.north) to (X5.south);

\end{tikzpicture}
} \vspace{-1em}
\caption{
Illustration of the type of feedback sequences that may occur in a multi-agent setting.
In the standard single-agent case, the feedback sequence $\vt[\set]$, $\run=\running$, is necessarily non-decreasing: even though the feedback may not arrive with the same order as the corresponding actions, the \emph{number} of available gradients can only grow. This no longer holds when multiple agents are involved in the optimization process.
}
\vspace{-1.6em}
\label{fig:feedback}
\end{figure}
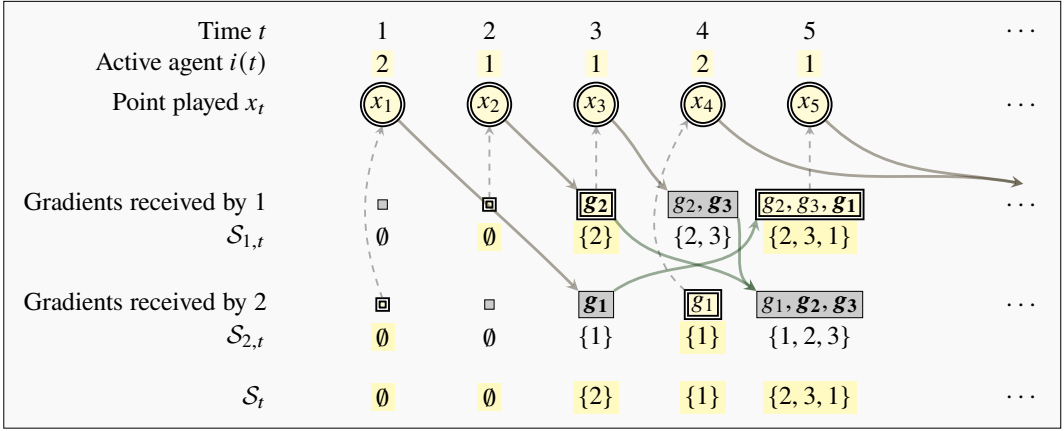

\subsection{Main Challenges: Non-Monotonicity of Feedback Sequence and Lack of Synchronization}

We now highlight two prominent features of our asynchronous online optimization framework that distinguish it from the large corpus of literature on \emph{single-agent} online learning with delays.
First, from the point of view of \emph{any} single agent $\worker$, the feedback sequence $(\vwt[\set])_{\run\in\oneto{\nRuns}}$ is non-decreasing by definition, \ie $\vwt[\set]\subseteq\vwt[\set][\worker][\run+1]$ for all $\run=\running$.
However, this may not be the case for the \emph{active} feedback sequence $(\vt[\set])_{\run\in\oneto{\nRuns}}$ which is in general \emph{non-monotone}.
In fact, due to communication delays, the same element of feedback may not arrive at each node at the same time.
Thus, as the active agent differs from one time slot to another, a timestamp contained in $\vt[\set]$ may not belong to $\update[\set]$ (see \cref{fig:feedback} for an illustration).
This leads to the first challenge we seek to overcome:

\begin{enumerate}
[leftmargin=*,label={Challenge \Roman*.}]
\item
Design learning algorithms capable of handling non-monotone feedback sequences.
\end{enumerate}

\begin{remark*}
We stress here that this issue is inextricably tied to the multi-agent character of our model.
In the single-agent case, $\vt[\set]$ is \emph{de facto} monotone, so this problem does not arise.
\end{remark*}

Second, in our model the agents only communicate when they exchange the received feedback.
Without additional coordination, the network does not maintain any global information about the evolution of the learning process.
In particular,
for reasons of privacy and information security,
we do not assume that agents have access to a global counter that indicates how many actions have been played at any given stage (as this could carry sensitive, identification-prone information).
Similarly, other quantities of interest, such as the current cumulative unavailability $\vt[\totaldelay]$ defined below, are also unavailable to each agent.
This leads to our second challenge:
\begin{enumerate}
[leftmargin=*,label={Challenge \Roman*.}]
\setcounter{enumi}{1}
\item
Dispense of the need to know $\run$ or other non-local information.
\end{enumerate}

As shown above, the lack of network synchronization, along with the non-monotonicity of the active feedback sequence, poses crucial challenges to both the design of the algorithms and the accompanying analysis.
In face of these, we introduce in \cref{subsec:faithful} an appropriate reordering of time that enables us to go beyond the algorithms developed for the single-agent setting.

\para{Quantifying the impact of delays}
As illustrated in \cref{fig:feedback}, having multiple agents also means that we can no longer associate a single delay to each individual feedback element.
This explains our choice of focusing on the available subgradients instead of the actual delays, which largely simplifies the description of the framework.
The delays, in turn, are still implicitly encoded in the sets $(\vwt[\set])$.
To quantify their effect, it will be convenient to consider the following measures:
\begin{itemize}[topsep=0.4em,itemsep=0.1em]
    \item The \emph{maximum delay} $\delaybound$ is the longest wait to receive an element of feedback:
    $\delaybound = \min \{ \tau : \oneto{\run-\delaybound-1}\subseteq\vt[\set] \text{ for all } {\run\in\oneto{\nRuns}} \}$.
    \item The \emph{maximum unavailability} $\outbound$  of the feedback is defined as $\outbound=\max_{\run\in\oneto{\nRuns}}\card(\vt[\setout])$.
    This is the maximum number of subgradients that could have---but otherwise \emph{haven't}---been communicated to an active agent at activation time.
    It is straightforward to see that $\outbound\le\delaybound$.\footnote{
    For any $\run \in \oneto{\nRuns} $, we have $\oneto{\run-\delaybound-1} \subseteq \vt[\set][\run]$  and thus $\vt[\setout][\run]=\setexclude{\oneto{\run-1}}{\vt[\set][\run]} \subseteq \intinterval{\run-\delaybound-1}{\run-1}$ which consists of $\delaybound$ elements.
    On the other hand, if, for some reason, one feedback is \emph{lost}, say the first one, then, the maximum delay is $\tau=\nRuns-1$ while the maximum unavailability is $\outbound=1$, in which case $\outbound \ll \delaybound$.}
    \item The \emph{cumulative unavailability} $\vt[\totaldelay]$ is given by
    $\vt[\totaldelay]=\sum_{\runalt=\start}^{\run}\card(\vt[\setout][\runalt])$.
    This generalizes the sum of delays to the multi-agent case;
    clearly, $\vt[\totaldelay]\le\outbound\run$.
\end{itemize}

\section{Delayed Dual Averaging and Faithful Permutations}
\label{sec:DDA}

In this section we present the main algorithmic template that we will use to address the limitations identified in the previous section, and which we call \acli{DDA}.
We also introduce the notion of ``faithful permuation'', which plays a major role in the analysis to come, as illustrated by the basic regret bound of \cref{thm:delay-regret} below.

\subsection{Delayed Dual Averaging}

To begin, recall that at each time $\run$, an agent computes the point $\vt[\state]$ using a collection of \emph{previously received subgradients} $\setdef{\vt[\gvec][\runalt]}{\runalt\in \vt[\set]}$ where $ \vt[\set]\subseteq\oneto{\run-1}$ represents the set of timestamps corresponding to the subgradients used by the active agent to produce $\vt[\state]$.
Put differently, if $\runalt \in \vt[\set]$, then  $\vt[\gvec][\runalt]\in\subd\vt[\obj][\runalt](\vt[\state][\runalt])$ has been used in the computation leading to playing $\vt[\state]$ at time $\run$.
On the other hand, $\vt[\setout]=\setexclude{\oneto{\run-1}}{\vt[\set]}$ collects the timestamps of the subgradients that are missing for the computation of $\vt[\state]$ due to delays.

Our candidate algorithm for this asynchronous setup builds on the \ac{DA} master template
\begin{equation}
\label{eq:dual-avg}
\tag{DA}
	\vt[\state] = \argmin_{\point\in\points}   
	\braces*{
	\sum_{\runalt <\run }\product{\vt[\gvec][\runalt]}{\point} + \frac{1}{\vt[\step]} \hreg(\point)
	}
\end{equation}
where
$\vt[\step]\geq0$ is a learning rate parameter
and
$\hreg\from\points\to\R$ is the method's \emph{regularizer}, assumed itself to be continuous and $1$-strongly convex relative to some ambient norm $\norm{\cdot}$ on $\vecspace$.
This algorithm is a version of ``\ac{FTRL} with linearized losses'' \citep{SSS06,SS11};
our terminology instead follows \citet{Nes09} and \citet{Xiao10} and is meant to clarify that we will be working with first-order feedback.

\begin{examples}
The two most popular candidates for $\hreg$ are the squared $\ell_2$-norm $\hreg(\point)=\norm{\point}_2^2/2$ for arbitrary closed convex constrained set $\points$ and the negative entropy $\hreg(\point)=\sum_{\indg}\point[\indg]\log(\point[\indg])$ for simplex constraints $\points=\setdef{\point}{\sum_{\indg=1}^{\vdim}\point[\indg]=1}$ (here $\point[\indg]$ denotes the $\indg^{th}$ coordinate of $\point$).
The first example is $1$-strongly convex relative to the Euclidean norm $\norm{\cdot} = \norm{\cdot}_{2}$, while the second one is $1$-strongly convex relative to the $\ell^{1}$ norm $\norm{\cdot}_{1}$ on the simplex.
\hfill
\endenv
\end{examples}

Of course, as stated, \eqref{eq:dual-avg} is not a practical algorithm in our setup because the active agent $\activeworker{\run}$ only has at its disposal the subgradients $\setdef{\vt[\gvec][\runalt]}{\runalt\in \vt[\set]}$ at time $\run$.
In view of this, we will consider the \acdef{DDA} policy
\begin{equation}
\label{eq:delayed-dual-avg}
\tag{DDA}
\vt[\state]
	= \argmin_{\point\in\points} \braces*{ \sum_{\runalt\in\vt[\set]}\product{\vt[\gvec][\runalt]}{\point} + \frac{\hreg(\point)}{\vt[\step]} }
	= \prox\left(- \vt[\step] \sum_{\runalt\in\vt[\set]}\vt[\gvec][\runalt]\right),
\end{equation}
where, for concision, we have now set
\begin{equation}
\notag
\prox(\dvec)
	= \argmax_{\point\in\points} \{ \braket{\dvec}{\point} - \hreg(\point) \},
	\quad
	\forall\dvec\in\dspace,
\end{equation}
for the \emph{mirror map} induced by $\hreg$.
An intuitive motivation for our algorithmic choice is that every feedback/gradient is put on a equal footing no matter which agent generated the corresponding action or the delay it suffers.\footnote{See \cref{subsec:MD-DA} for more discussion.}
Moreover, as long as $\vt[\step]$ can be computed locally,
\eqref{eq:delayed-dual-avg} can indeed be implemented independently by each agent of the network,
without requiring a global clock;
for a pseudocode implementation, see \cref{algo:DDA-dist}.

\begin{algorithm}[t]
    \caption{\eqref{eq:delayed-dual-avg} -- from the point of view of agent $\worker$  }
     \label{algo:DDA-dist}
 \begin{algorithmic}[1]
     \STATE {\bfseries Initialize:}
         $\gvecs_\worker \subs \emptyset$, $\run\subs1$.
    \WHILE{not stopped}
    \STATE {\bfseries asynchronously} {receive feedback $\vt[\gvec][\runalt]$ from time $\runalt$
    \STATE \hspace{\algorithmicindent} $\gvecs_\worker \subs \gvecs_\worker \union\{\runalt\}$
    } \label{algo:DDA-dist:receive}
    \STATE \hspace{\algorithmicindent} Relay $\vt[\gvec][\runalt]$ if necessary
    \IF{the agent becomes active, \ie $\activeworker{\run}=\worker$}
     \STATE $\vt[\set] \subs \gvecs_\worker$
     \STATE Update $\vt[\step]$ and play $\vt[\state] =  \argmin_{\point\in\points} \sum_{\runalt\in\vt[\set]}\product{\vt[\gvec][\runalt]}{\point} + \frac{\hreg(\point)}{\vt[\step]}$
    \ENDIF
    \ENDWHILE
 \end{algorithmic}
 \end{algorithm}

\subsection{Dependencies and Faithful Permutations}
\label{subsec:faithful}

A crucial challenge in \eqref{eq:delayed-dual-avg} is the choice of $\vt[\step]$.
Indeed, the standard analysis of \ac{DA} requires the learning rate sequence to be non-increasing, a property that can hardly be ensured in our situation due to the non-monotonicity of the active feedback sequence and the lack of network synchronization.
To sidestep this issue, we need to rethink what ``time'', or the ordering of the timestamps means to \eqref{eq:delayed-dual-avg}, and how this can be leveraged to construct a valid algorithm.

Our starting point will be to redefine the algorithm's internal clock (and corresponding learning rate) based exclusively on the active timestamp sets $\vt[\set]$, $\run=\running,\nRuns$.
To that end, we will start by viewing each timestamp as a node in a ``causal graph'', and we will include a directed edge from $\runalt$ to $\run$ if and only if $\runalt\in\vt[\set]$:
this represents a ``causal dependency'' of $\run$ on $\runalt$ in the sense that the gradient $\vt[\gvec][\runalt]$ has been used to define $\vt[\state]$ (cf.~\cref{fig:DAG}).
We will refer to this graph as the \emph{dependency graph} associated to the active feedback sequence $\vt[\set]$, $\run=\running,\nRuns$, and we will denote it by $\graph$;
for clarity, we also stress here that we do not assume that this structure is known to the agents.

A first important observation is that the default time ordering $\run=\running$ represents a topological sort of $\graph$, \ie a linear ordering of its vertices such that $\runalt < \run$ if there exists a directed edge $\runalt\leadsto\run$ in $\graph$.
\footnote{In particular, this property implies that $\graph$ is a \acf{DAG}.}
Second, since the update structure of \eqref{eq:delayed-dual-avg} is determined entirely by $\graph$ and the value of $\vt[\step]$ at each vertex of $\graph$, it follows that any reshuffling of time that respects the causal structure of $\graph$ should be an equally viable alternative for the algorithm.
We formalize this idea below via the notion of a \emph{faithful permutation}.

\begin{figure}[t]
\centering
\begin{tikzpicture}
[scale=1.0,
nodestyle/.style={circle,draw=black,thick,inner sep=2pt},
edgestyle/.style={-,thick,color=black,opacity=1},
>=stealth]

\def\dx{1}
\def\dy{1.4}

\coordinate (N1) at (0*\dx,0*\dy);
\coordinate (N2) at (1*\dx,1*\dy);
\coordinate (N3) at (2*\dx,0*\dy);
\coordinate (N4) at (3*\dx,1*\dy);
\coordinate (N5) at (4*\dx,0*\dy);

\node (N1) at (N1) [nodestyle] {$1$};
\node (N2) at (N2) [nodestyle] {$2$};
\node (N3) at (N3) [nodestyle] {$3$};
\node (N4) at (N4) [nodestyle] {$4$};
\node (N5) at (N5) [nodestyle] {$5$};

\draw [edgestyle,->] (N1) to (N4);
\draw [edgestyle,->] (N1) to (N3);
\draw [edgestyle,->] (N2) to (N5);
\draw [edgestyle,->] (N3) to (N4);
\draw [edgestyle,->] (N3) to (N5);
\draw [edgestyle,->,out=-30,in=-150] (N1) to (N5);

\end{tikzpicture}
 \hspace{1.5cm}
\begin{tikzpicture}
[scale=1.0,
nodestyle/.style={circle,draw=black,thick,inner sep=2pt},
edgestyle/.style={-,thick,color=black,opacity=1},
>=stealth]

\def\dx{1}
\def\dy{1.4}

\coordinate (N1) at (0*\dx,0*\dy);
\coordinate (N2) at (1*\dx,1*\dy);
\coordinate (N3) at (2*\dx,0*\dy);
\coordinate (N4) at (3*\dx,1*\dy);
\coordinate (N5) at (4*\dx,0*\dy);

\node (N1) at (N1) [nodestyle] {$1$};
\node (N2) at (N2) [nodestyle] {$2$};
\node (N3) at (N3) [nodestyle] {$3$};
\node (N4) at (N4) [nodestyle] {$4$};
\node (N5) at (N5) [nodestyle] {$5$};

\draw [edgestyle,->] (N1) to (N4);
\draw [edgestyle,->] (N2) to (N3);
\draw [edgestyle,->] (N2) to (N5);
\draw [edgestyle,->] (N3) to (N5);
\draw [edgestyle,->,out=-30,in=-150] (N1) to (N5);

\end{tikzpicture}
 \vspace{-0.4em}
\caption{The dependency graphs for the two examples of \cref{fig:feedback};
the left and right graphs correspond respectively to the single- and multi-agent examples presented therein.
The active feedback at time $\run$ is exactly the set of in-neighbors of the corresponding vertex.
}
\label{fig:DAG}
\vspace{-1.4em}
\end{figure}
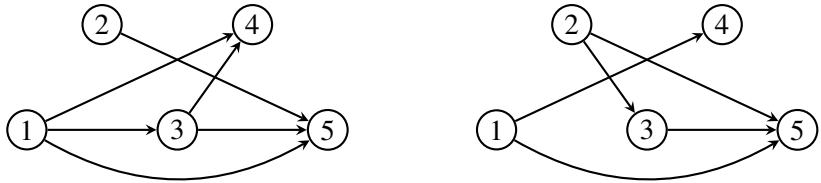

\begin{definition}
[Faithful permutations]
\label{def:faithful}
A permutation $\Dpermu$ of $\{\running,\nRuns\}$ is \emph{faithful} if and only if, for all $\runalt,\run=1,\dotsc,\nRuns$, we have
\begin{equation}
\label{eq:faithful}
\runalt\in\vt[\set]
	\implies
	\Dpermu^{-1}(\runalt)<\Dpermu^{-1}(\run).
\end{equation}
Equivalently, $\Dpermu$ is faithful if and only if $\Dpermu(1),\dotsc,\Dpermu(\nRuns)$ is a topological ordering of $\graph$.
\end{definition}

\Cref{def:faithful} means that the feedback used at time $\Dpermu(\run)$ (whose time indices are in $\vt[\set][\Dpermu(\run)]$) form a subset of  $\{\vt[\gvec][\Dpermu(\start)],\ldots,\vt[\gvec][\Dpermu(\run-1)]\}$. Indeed, if $\Dpermu(\runalt)\in\vt[\set][\Dpermu(\run)]$, then $\runalt=\Dpermu^{-1}(\Dpermu(\runalt)) < \Dpermu^{-1}(\Dpermu(\run)) = \run $, \ie $s\in\{\start,\ldots,\run-1\}$.
Thus, a faithful permutation can be seen as a reordering of the time that would still be compatible with the feedback used by each agent at every time.
We illustrate this notion with two examples below:

\begin{examples}
Clearly, the identity permutation $\run\mapsto\run$ is always faithful.
More interestingly, in the single-agent setting, we can define the \emph{ordering-by-arrival} as follows:
if the $\grun$-th received subgradient originates from round $\run$---\ie $\vt[\gvec] \in \subd\vt[\obj](\vt[\state])$---we set $\Dpermu(\grun) = \run$, so $\vt[\gvec]$ is the $\Dpermu^{-1}(\run)$-th received gradient.
\footnote{If multiple gradients arrive at a given round, we resolve ties arbitrarily;
this ambiguity in the definition of $\Dpermu$ plays no role in the analysis.}
In this notation, the timestamps of all feedback received \emph{before} $\vt[\gvec][\run]$ can be written as $\vt[\setrec] \defeq \{ \Dpermu(1),\dotsc,\Dpermu(\Dpermu^{-1}(\run)-1) \}$ for that $\vt[\gvec][\run]$ is the $\Dpermu^{-1}(\run)$-th feedback.
Along with the inclusion $\vt[\set]\subseteq\vt[\setrec]$, which holds because $\vt[\gvec]$ is necessarily computed with gradients arriving before itself, we see that $\Dpermu$ is indeed a faithful permutation.
\hfill
\endenv
\end{examples}

\begin{remark*}
A similar notion was considered by \citet{ZS20}, but for a completely different purpose.
There, the authors aimed to provide optimal algorithms for \emph{single-agent} adversarial bandits with delays.
They defined a ``dependency-preserving permutation'' exactly as the inverse of what we call a faithful permutation, and they used this notion to analyze an algorithm that can ``skip'' certain rounds of feedback
when tuning the algorithm's learning rate.
Our definition is motivated by---and tailored to---the multi-agent setting, where the non-monotonicity of the active feedback sequence $\vt[\set]$ plays a major role (we recall that this phenomenon cannot arise in the single-agent case).
These elements are altogether absent in the single-agent considerations of \citet{ZS20}.
\end{remark*}

\subsection{Bounding the Regret of Delayed Dual Averaging}
\label{subsec:reg-DDDA}

We are now in a position to state and prove our main, data-dependent regret guarantee for \eqref{eq:delayed-dual-avg} when run with a learning rate that is non-increasing \emph{along a faithful permutation}.
For simplicity, we assume throughout the sequel that  $\hreg$ is non-negative.
This is possible because $\hreg$ is strongly convex and we can thus always replace $\hreg$ by the non-negative function $\hreg-\min\hreg$ without affecting our algorithms.

Similar to $\oneto{\run}$ and $\vt[\setout]$, for a faithful permutation $\Dpermu$, we also define the set of the first $\run$ elements under the new ordering and the set of unavailable elements induced by this ordering as
\begin{equation}
	\notag
	\oneto{\run}^{\Dpermu} = \{\Dpermu(1),\ldots,\Dpermu(\run)\}
	\quad
	\text{and}
	\quad
	\vt[\setout^{\Dpermu}] = \setexclude{\oneto{\run-1}^{\Dpermu}}{\vt[\set][\Dpermu(\run)]}.
\end{equation}
We have the following theorem concerning the regret of \eqref{eq:delayed-dual-avg}.

\begin{restatable}{theorem}{DelayRegret}
\label{thm:delay-regret}
Let $\Dpermu$ be a faithful permutation of $\{\start,\ldots,\nRuns\}$, and assume that \acl{DDA} \eqref{eq:delayed-dual-avg} is run with a learning rate $\vt[\step]$, $\run=\running$, such that $\vt[\step][\Dpermu(\run+1)]\le\vt[\step][\Dpermu(\run)]$ for all $\run$.
Then the algorithm enjoys the regret bound
\begin{equation}
\label{eq:delay-bound}
	\reg_{\nRuns}(\comp)
	\le
	\frac{\hreg(\comp)}{\vt[\step][\Dpermu(\nRuns)]}
	+ \frac{1}{2}\sum_{\run=\start}^{\nRuns}
	\vt[\step][\Dpermu(\run)]
	\left(\dnorm{\vt[\gvec][\Dpermu(\run)]}^2
		+2\dnorm{\vt[\gvec][\Dpermu(\run)]}
		\sum_{\runalt\in
		\vt[\setout^{\Dpermu}]}\dnorm{\vt[\gvec][\runalt]}\right).
\end{equation}
\end{restatable}

\cref{thm:delay-regret} provides a template regret bound that forms the basis of all the upcoming analysis.
To begin, we note that the bound \eqref{eq:delay-bound} consists of 
the usual online dual averaging bound (\cf \cref{apx:DA}) plus
a term containing $ \sum_{\runalt\in\vt[\setout^{\Dpermu}]}\dnorm{\vt[\gvec][\runalt]}$ that reflects the impact of delay.
Similar decompositions can be found in \citet{MS14}, \citet{JGS16} and \citet{JGS19} respectively for online gradient descent, online mirror descent, and dual averaging.
\footnote{In \cite{MS14}, the authors work with the specific setting of coordinate-wise unconstrained gradient methods. Therefore, instead of products of norms they have products of scalars in their analysis.}
These papers focused on the single-agent (shared-memory) setting and conducted the analysis by either choosing $\Dpermu$ as the identity or the ordering by arrival.
\cref{thm:delay-regret} thus extends these results by providing a larger class of possible learning rate policies, which enables us to devise efficient and truly implementable learning rate update schemes for the fully decentralized setting in \cref{sec:variable}.

\vskip 0.5em
\begin{proof}[Proof of \cref{thm:delay-regret}]
As usual, the first step is to bound the algorithm's regret by its linearized counterpart, viz.
\begin{equation}
	\notag
	\reg_{\nRuns}(\comp)
	=\sum_{\run=\start}^{\nRuns}\vt[\obj](\vt)-\vt[\obj](\comp)
	\le\sum_{\run=\start}^{\nRuns}\product{\vt[\gvec]}{\vt[\state]-\comp}.
\end{equation}
To proceed, we leverage the so-called ``perturbed iterate'' framework for analyzing asynchronous algorithms in the spirit of \cite{MPPR+15} and \cite{JGS19}.
Formally, we define the following virtual iterate sequence
\begin{equation}
\notag
\vt[\virtual] = \argmin_{\point\in\points}
\sum_{\runalt=1}^{\run-1}\product{\vt[\gvec][\Dpermu(\runalt)]}{\point} + \frac{\hreg(\point)}{\vt[\step][\Dpermu(\run)]}.
\end{equation}
and decompose the sum as:
\begin{equation}
\label{eq:delay-proof-decomp}
	\sum_{\run=\start}^{\nRuns}\product{\vt[\gvec]}{\vt[\state]-\comp}
	= 
  \underbrace{ \sum_{\run=\start}^{\nRuns}\product{\vt[\gvec]}{\vt[\virtual][\Dpermu^{-1}(\run)]-\comp} }_{(a)}
	+ \underbrace{\sum_{\run=\start}^{\nRuns}\product{\vt[\gvec]}{\vt[\state]-\vt[\virtual][\Dpermu^{-1}(\run)]}}_{(b)}.
\end{equation}
We now proceed to bound each term separately.

\para{Term (\emph{a})}
The first term is exactly the linearized regret of the iterates $\vt[\virtual][\start],\ldots,\vt[\virtual][\nRuns]$ that is constructed with the feedback $\vt[\gvec][\Dpermu(\start)],\ldots,\vt[\gvec][\Dpermu(\nRuns)]$.
Thus, by analyzing the regret incurred by the dual averaging algorithm \eqref{eq:dual-avg} \emph{without} delays, we show in \cref{apx:DA} that this term can be bounded as
\begin{equation}
	\label{eq:delay-proof-virtual-regret}
	\sum_{\run=\start}^{\nRuns}\product{\vt[\gvec]}{\vt[\virtual][\Dpermu^{-1}(\run)]-\comp}
	=\sum_{\run=\start}^{\nRuns}\product{\vt[\gvec][\Dpermu(\run)]}{\vt[\virtual]-\comp}
	\le
	\frac{\hreg(\comp)}{\vt[\step][\Dpermu(\nRuns)]}
	+ \frac{1}{2}\sum_{\run=\start}^{\nRuns}\vt[\step][\Dpermu(\run)]\dnorm{\vt[\gvec][\Dpermu(\run)]}^2.
\end{equation}
Note that the assumption on the learning rate sequence ($\vt[\step][\Dpermu(\run+1)]\leq \vt[\step][\Dpermu(\run)]$) is crucial for the derivation of this bound.
\smallskip

\para{Term (\emph{b})}
For the second term, we would like to bound the distance between $\vt$ and $\vt[\virtual][\Dpermu^{-1}(\run)]$, or equivalently, the distance between $\vt[\state][\Dpermu(\run)]$ and $\vt[\virtual]$ (since we shall consider all the $\run\in\{\start,\ldots,\nRuns\}$).
To that end, we write
\begin{align}
\notag
\vt[\state][\Dpermu(\run)]
	= \prox\parens*{-\vt[\step][\Dpermu(\run)]\sum_{\runalt\in\vt[\set][\Dpermu(\run)]}\vt[\gvec][\runalt]}
	\quad
	\text{and}
	\quad
\vt[\virtual]
	= \prox\parens*{-\vt[\step][\Dpermu(\run)]\sum_{\runalt\in\oneto{\run-1}^{\Dpermu}}\vt[\gvec][\runalt]}.
\end{align}
Since the permutation $\Dpermu$ is faithful, we have $\vt[\set][\Dpermu(\run)]\subseteq\{\Dpermu(1),..,\Dpermu(\run-1)\}=\oneto{\run-1}^{\Dpermu}$.
We can then use the non-expansivity of the mirror map (\cref{lem:prox-nonexp} in \cref{apx:DA}) to get
\begin{equation}
	\notag
	\norm{\vt[\state][\Dpermu(\run)]-\vt[\virtual]}
	\le
	\dnorm{	\vt[\step][\Dpermu(\run)]
	\sum_{\runalt\in\vt[\setout^{\Dpermu}]}\vt[\gvec][\runalt]}
	\le
	\vt[\step][\Dpermu(\run)]
	\sum_{\runalt\in\vt[\setout^{\Dpermu}]}\dnorm{\vt[\gvec][\runalt]}.
\end{equation}
Subsequently,
\begin{align}
	\notag
	\sum_{\run=\start}^{\nRuns}\product{\vt[\gvec]}{\vt[\state]-\vt[\virtual][\Dpermu^{-1}(\run)]}
	&=
	\sum_{\run=\start}^{\nRuns}\product{\vt[\gvec][\Dpermu(\run)]}{\vt[\state][\Dpermu(\run)]-\vt[\virtual][\run]}\\
	\notag
	&\le
	\sum_{\run=\start}^{\nRuns}
	\dnorm{\vt[\gvec][\Dpermu(\run)]}
	\norm{\vt[\state][\Dpermu(\run)]-\vt[\virtual][\run]}\\
	\label{eq:delay-proof-diff}
	&\le
	\sum_{\run=\start}^{\nRuns}
	\vt[\step][\Dpermu(\run)]
	\dnorm{\vt[\gvec][\Dpermu(\run)]}
	\sum_{\runalt\in\vt[\setout^{\Dpermu}]}\dnorm{\vt[\gvec][\runalt]}.
\end{align}
Combining \eqref{eq:delay-proof-decomp}, \eqref{eq:delay-proof-virtual-regret} and \eqref{eq:delay-proof-diff}, we obtain the desired result.
\end{proof}

\subsection{Constant Learning Rate and Lag}

To get an idea of the optimal regret that the algorithm can achieve, we fix a constant learning rate  $\vt[\step]\equiv\step$, which we subsequently optimize to minimize the upper-bound on the regret.
To proceed, we define the \emph{cumulative lag} as
\begin{equation}
\label{eq:gsum-sigma}
\gsum^{\Dpermu}_{\run}
	= \sum_{\runalt=\start}^{\run}
	\left(
		\dnorm{\vt[\gvec][\Dpermu(\runalt)]}^2 
		+ 2\dnorm{\vt[\gvec][\Dpermu(\runalt)]}
		\sum_{\runano\in\vt[\setout^{\Dpermu}][\runalt]}\dnorm{\vt[\gvec][\runano]}
	\right)
	 = 
	\sum_{\runalt\in\oneto{\run}^{\Dpermu}}
	\dnorm{\vt[\gvec][\runalt]}^2 + 2
		\sum_{\{\runalt,\runano\}\in \vt[\setdel^{\Dpermu}]}
	\dnorm{\vt[\gvec][\runalt]}\dnorm{\vt[\gvec][\runano]},
\end{equation}
where
\begin{equation}
    \notag
    \vt[\setdel^{\Dpermu}] = \setdef{\{\Dpermu(\runalt),\runano\}}{\runalt\in\oneto{\run},\runano\in\vt[\setout^{\Dpermu}][\runalt]}.
\end{equation}
In words, $\{\alt{\runalt},\runano\}\in\vt[\setdel^{\Dpermu}]$ if 
\begin{enumerate*}[\itshape i\upshape)]
\item $\vt[\gvec][\runano]$ is not used to define $\vt[\state][\alt{\runalt}]$; and
\item after reordering by $\Dpermu$, $\runano$ comes before $\alt{\runalt}$ and $\alt{\runalt}$ comes before $\Dpermu(\run)$.
\end{enumerate*}
We also write $\vt[\gsum] = \vt[\gsum^{\idp}]$ for the lag associated to the standard time ordering
and define $\vt[\totaldelay^{\Dpermu}]=\card(\vt[\setdel^{\Dpermu}])=\sum_{\run=\start}^{\nRuns}\card(\vt[\setout^{\Dpermu}])$.

In the above, while $\vt[\totaldelay^{\Dpermu}]$ captures the ``total delay'' in terms of the reordering induced by $\Dpermu$,
the cumulative lag $\gsum^{\Dpermu}_{\run}$
regroups the actual errors caused by the inability of the learners to compensate the missing feedback, and gives the most fine-grained characterization of the effect of delayed feedback on the regret.
In the single-agent setting, \cite{JGS16} and \citet{MS14} also considered the same quantity but in the special case where $\Dpermu$ is the
ordering-by-arrival permutation discussed in \cref{subsec:faithful}.
In general, it is clear that $\vt[\gsum^{\Dpermu}]\le(\run+2\vt[\totaldelay^{\Dpermu}])\gbound^2$ provided that all subgradients are bounded in norm by $\gbound$;
moreover, if $\Dpermu$ is the identity permutation, we further have $\vt[\totaldelay^{\Dpermu}]=\vt[\totaldelay]\le\outbound\run$.
With all this in mind, a direct application of \cref{thm:delay-regret} gives the following series of more explicit bounds.

\begin{corollary}
\label{cor:delay-regret}
Let $\Dpermu$ be a faithful permutation and assume that \acl{DDA} \eqref{eq:delayed-dual-avg} is run with a constant learning rate $\step > 0$.
Then:
\begin{itemize}[leftmargin=1em]
	\item If $\dnorm{\vt[\gvec]}$ is uniformly bounded and $\step=\Theta\left(1/\sqrt{\max(1,\outbound)\nRuns}\right)$, then $\reg_{\nRuns}(\comp)=\bigoh\left(\sqrt{\max(1,\outbound)\nRuns}\right)$.
    \item If 
	$\dnorm{\vt[\gvec]}$ is uniformly bounded and $\step=\Theta\left(1/\sqrt{\max(\nRuns,\vt[\totaldelay^{\Dpermu}][\nRuns])}\right)$, then $\reg_{\nRuns}(\comp)=\bigoh\left(\sqrt{\max(\nRuns,\vt[\totaldelay^{\Dpermu}][\nRuns])}\right)$.
	\item If $\step=\Theta\left(1/\sqrt{\vt[\gsum^{\Dpermu}][\nRuns]}\right)$, then $\reg_{\nRuns}(\comp)=\bigoh\left(\sqrt{\vt[\gsum^{\Dpermu}][\nRuns]}\right)$.
\end{itemize}
\end{corollary}

\cref{cor:delay-regret} recapitulates several types of regret bounds that we can expect from \eqref{eq:delayed-dual-avg}, depending on the tuning of $\vt[\step]$ (either by using a pessimistic upper bound on the delays and the norms of the gradients, or using the actual delays and/or received gradients).
In particular, if we focus on the standard time ordering $\Dpermu = \idp$, \cref{cor:delay-regret} allows us to recover the optimal data-dependent bound of $\bigoh(\sqrt{\vt[\gsum][\nRuns]})$ that was previously obtained for the single-agent setting by \citet{JGS16} and \citet{MS14}.
Moreover, if we further assume that $\dnorm{\vt[\gvec]}\le\gbound$ for all $\run\in\oneto{\nRuns}$, we have $\gsum_{\nRuns}\le(\nRuns+2\vt[\totaldelay][\nRuns])\gbound^2$, which leads to the well-known $\bigoh(\sqrt{\vt[\totaldelay][\nRuns]})$ bound on the regret (see \eg \citealp{QD15}).

On the downside, \cref{cor:delay-regret} would seem to suggest that the derived regret bounds depend on the choice of the permutation $\Dpermu$, a concept that is relevant for the analysis, but which is otherwise devoid of physical meaning (at least, relative to the sequence of events as it unfolds in real time).
Because of this, the computation of the optimal learning rates required by \cref{cor:delay-regret} seems beyond reach in practice \textendash\ even if we assume that the various quantities involved are somehow known to the agents.
However, as we show below, \emph{this is not the case:}
the values of both $\vt[\totaldelay^{\Dpermu}][\nRuns]$ and $\vt[\gsum^{\Dpermu}][\nRuns]$ are independent of $\Dpermu$, and hence, so are the bounds of \cref{cor:delay-regret}.
To prove this, we first provide a new characterization of the set $\vt[\setdel^\Dpermu]$ which is of independent interest:

\begin{proposition}
\label{prop:setdel-charac}
Let $\Dpermu$ be a faithful permutation. Then
\begin{equation}
\vt[\setdel^{\Dpermu}]
	= \setdef*{\{\runalt,\runano\} \subseteq \oneto{\run}^{\Dpermu}}{\text{$\runalt$ and $\runano$ are not adjacent in $\graph$}}.
\end{equation}
\end{proposition}

\begin{proof}
By definition of the dependency graph, $\runalt$ and $\runano$ are not adjacent in $\graph$ if and only if  $\{\runalt\notin\vt[\set][\runano], \runano\notin\vt[\set][\runalt]\}$.
We will thus show that
    \begin{equation}
    \notag
    \vt[\setdel^{\Dpermu}]
	= \setdef*{\{\runalt,\runano\} \subseteq \oneto{\run}^{\Dpermu}}{\runalt\notin\vt[\set][\runano], \runano\notin\vt[\set][\runalt]}.
\end{equation}
This relies on a two-way inclusion argument.

\para{Inclusion (``\,$\subseteq$\,'')}
Let $\runalt\in\oneto{\run}$ and $\runano\in\vt[\setout^\Dpermu][\runalt]=\setexclude{\oneto{\runalt-1}^{\Dpermu}}{\vt[\set][\Dpermu(\runalt)]}$. By definition of $\oneto{\run}^{\Dpermu}$ we have $\Dpermu(\runalt)\in\oneto{\run}^{\Dpermu}$ and $\runano\in\oneto{\runalt-1}^{\Dpermu}\subseteq\oneto{\run}^{\Dpermu}$. It remains to prove that $\Dpermu(\runalt)\notin\vt[\set][\runano]$. We exploit the equivalence
\begin{equation}
\label{eq:apx-charac-D-eqv}
\runano\in\oneto{\runalt-1}^{\Dpermu}
	\iff
	\Dpermu^{-1}(\runano)\le\runalt-1
	\iff
	\Dpermu^{-1}(\runano) < \Dpermu^{-1}(\Dpermu(\runalt))
	\iff
	\Dpermu(\runalt)\notin
	\oneto{\Dpermu^{-1}(\runano)}^{\Dpermu}.
\end{equation}
To conclude, we use the fact that $\Dpermu$ is a faithful permutation and accordingly $\vt[\set][\runano]\subseteq\oneto{\Dpermu^{-1}(\runano)-1}^{\Dpermu}\subseteq\oneto{\Dpermu^{-1}(\runano)}^{\Dpermu}$. Along with \eqref{eq:apx-charac-D-eqv} we deduce that $\Dpermu(\runalt)\notin\vt[\set][\runano]$.

\para{Containment (``\,$\supseteq$\,'')}
Let $\{\runalt,\runano\}\subset\oneto{\run}^{\Dpermu}$ such that  $\runalt\notin\vt[\set][\runano]$ and $\runano\notin\vt[\set][\runalt]$. We assume without loss of generality $\Dpermu^{-1}(\runano)<\Dpermu^{-1}(\runalt)$. This is indeed equivalent to $\runano\in\oneto{\Dpermu^{-1}(\runalt)-1}^{\Dpermu}$ and therefore $\runano\in\vt[\setout^\Dpermu][\Dpermu^{-1}(\runalt)]$. We complete the proof by noting that $\runalt\in\oneto{\run}^{\Dpermu}$ if and only if $\Dpermu^{-1}(\runalt)\in\oneto{\run}$.
\end{proof}

In contrast to the original definition of $\vt[\setdel^{\Dpermu}]$, the characterization of \cref{prop:setdel-charac}---\ie the non-adjacency of the vertices---is independent of the ordering of the timestamps.
By defining $\graph_{\run}^{\Dpermu}$ as the subgraph of $\graph$ spanned by the vertices of $\oneto{\run}^{\Dpermu}$ in $\graph$,
the proposition says that $\vt[\setdel^{\Dpermu}]$ contains exactly the non-adjacent vertex pairs of $\vt[\graph^{\Dpermu}]$.
With this in mind, we readily obtain the following important corollary:

\begin{corollary}
\label{cor:D-gsum-sigma}
For any two faithful permutations $\Dpermu$ and $\Dpermualt$, we have $\vt[\setdel^{\Dpermu}][\nRuns]=\vt[\setdel^{\Dpermualt}][\nRuns]$, and, a fortiori, $\vt[\totaldelay^{\Dpermu}][\nRuns]=\vt[\totaldelay^{\Dpermualt}][\nRuns]$ and $\vt[\gsum^{\Dpermu}][\nRuns]=\vt[\gsum^{\Dpermualt}][\nRuns]$.
In other words, the regret bounds of \cref{cor:delay-regret} are independent of $\Dpermu$.
\end{corollary}
\begin{proof}
Simply note that $\oneto{\nRuns}^{\Dpermu}=\oneto{\nRuns}^{\Dpermualt}=\oneto{\nRuns}$.
\end{proof}

\cref{cor:D-gsum-sigma} shows that the regret bounds of \cref{cor:delay-regret} are indeed meaningful, as they do not depend on any ``virtual'' reordering of time by a faithful permutation.
However, given that the quantities $\vt[\gsum][\nRuns]$ and $\vt[\totaldelay][\nRuns]$ cannot be assumed known beforehand, the agents might need to employ a much more conservative learning rate of the order of $\Theta(1/\sqrt{\outbound\nRuns})$ to minimize their regret.
We address this important issue via the design of suitable adaptive learning methods in the next section. 

\section{Tuning the Learning Rate in the Presence of Delays}
\label{sec:variable}

In this section, we exploit the template bound of \cref{thm:delay-regret} to design efficient leaning rates that provably achieve low regret.
To clarify our objective, we begin by identifying the main desiderata that
we seek to achieve:
\begin{enumerate}
[\itshape i\upshape)]
\item
\emph{Anytime / Restart-free:}
the algorithm should not require the knowledge of the horizon $\nRuns$ and/or include a restart schedule where previous information is discarded.
\item
\emph{Coordination-free:}
the learning rate of each agent must be computable based \emph{exclusively} on local information without any need for agent coordination.
\item
\emph{Data-dependent bounds:}
the algorithm's regret guarantees should feature the actual gradients observed instead of an upper bound thereof.
\item
\emph{Adaptivity to delays:}
the algorithm's regret should depend on the observed delays and not only on a pessimistic, worst-case estimate thereof.
\end{enumerate}

To derive a learning rate with the above properties,
we will employ an ``inverse-root-sum-square'' policy in the spirit of AdaGrad and other adaptive algorithms.
This is perhaps the easiest to illustrate in the case $\Dpermu=\idp$:
here, to obtain an $\bigoh(\sqrt{\vt[\gsum][\nRuns]})$ regret, we could employ the policy $\vt[\step] = 1/\sqrt{\vt[\gsum]} = 1/\sqrt{\sum_{\runalt=\start}^{\run}\vt[\gsumpar][\runalt]}$ where
\[
\vt[\gsumpar][\runalt]
	= \dnorm{\vt[\gvec][\runalt]}^2 + 2\dnorm{\vt[\gvec][\runalt]}
		\sum_{\runano\in\vt[\setout][\runalt]}\dnorm{\vt[\gvec][\runano]}.
\]
The key in the analysis of this policy is provided by the following standard lemma (dating back at least to \citealp{ACG02}, and proven for completeness in \cref{apx:adaptive}):

\begin{restatable}{lemma}{adaptive}
\label{lem:adaptive}
For any sequence of real numbers $\seq{\scalar}{1}{\nRuns}$ with $\sum_{\runalt=1}^{\run}\scalar_{\runalt}>0$ for all $\run\in\oneto{\nRuns}$, we have
\begin{equation}
    \notag
    \sum_{\run=1}^{\nRuns}
    \frac{\scalar_{\run}}{\sqrt{\sum_{\runalt=1}^{\run}\scalar_{\runalt}}}
    \le 2 \sqrt{\sum_{\run=1}^{\nRuns}\scalar_{\run}}.
\end{equation}
\end{restatable}

Based on this lemma, it is straightforward to show that \eqref{eq:delayed-dual-avg} with learning rate $\vt[\step] = 1/\sqrt{\vt[\gsum]}$ incurs at most $\bigoh(\sqrt{\vt[\gsum][\nRuns]})$ regret.
However, this policy is not implementable because it involves unobserved feedback---and hence violates one of our principal desiderata.
In the rest of this section, we show how this difficulty can be circumvented in many relevant scenarios.

\subsection{Pessimistic Non-Adaptive Learning Rate}
\label{subsec:nonada}

To set the stage for the analysis to come, we begin by assuming that the agents know $\delaybound$ an upper bound on the maximum delay and $\gbound$ an upper bound on the norms of the observed gradients.
This leads to $\vt[\gsumpar][\runalt] \leq \gbound^2 (1+2\outbound) \leq \gbound^2 (1+2\delaybound)$, and subsequently $\vt[\gsum]\le\gbound^2 (1+2\delaybound)\run$.
Given this preliminary result, it is tempting to choose $\vt[\step] = \Theta(1/\gbound\sqrt{\run(1+2\delaybound)})$.
This is however still unrealistic as the agents do not know the exact value of $\run$, and may only estimate it by using $\run\le\card(\vt[\set])+\delaybound+1$.
To justify this strategy, we need to prove that the corresponding learning rate is indeed non-increasing along some faithful permutation in order to apply \cref{thm:delay-regret}.
For this, we make the following assumption.

\begin{assumption}
\label{asm:card-in-order}
If $\runalt\in\vt[\set]$, then $\card(\vt[\set][\runalt])<\card(\vt[\set])$.
\end{assumption}

In words, the assumption requires that if $\vt[\gvec][\runalt]$ is used to compute $\vt[\state]$, then $\vt[\state][\runalt]$ is computed with fewer gradients than $\vt[\state]$.
This is a fairly mild requirement which is in turn implied by the upcoming \cref{asm:in-order} (see the accompanying discussion).
In particular, if the agents also relay the information $\card(\vt[\set])$ as well, \cref{asm:card-in-order} can be ensured by delaying the actual usage of a received feedback when necessary.
\footnote{In this case, $\vt[\set]$ refers to the timestamps of the gradients that are used for the computation of $\vt[\state]$;
however, this does not necessarily contain all the gradients that the active agent $\activeworker{\run}$ has received by time $\run$.}
Then, when the actual delays are bounded by $\delaybound$, the gradients $\{\vt[\gvec][\start],\ldots,\vt[\gvec][\run-\delaybound-1]\}$ can always be used for computing $\vt[\state]$. Therefore, introducing this extra delay will not increase the maximum delay and has no effect on the regret bound of the following proposition.

\begin{restatable}{proposition}{DecenDecrRegret}
\label{prop:decen-decr-regret}
Suppose that \cref{asm:card-in-order} holds, the maximum delay is bounded by $\delaybound$, and the norm of the observed gradients is bounded by $\gbound$.
Assume further that \acl{DDA} \eqref{eq:delayed-dual-avg} is run with the learning rate
\begin{align}
\notag
	\vt[\step] = \frac{\radius}{\gbound\sqrt{(1+2\delaybound)(\card(\vt[\set])+\delaybound+1)}}.
\end{align}
Then, for any $\comp$ such that $\hreg(\comp)\le\radius^2$, the generated points $\seq{\state}{\start}{\nRuns}$ enjoy the regret bound
\begin{equation}
\notag
	\reg_{\nRuns}(\comp)
	\le
   2 \radius \gbound \sqrt{(\nRuns+\delaybound)(1+2\delaybound)}.
\end{equation}
\end{restatable}
\begin{proof}
We will in fact prove a stronger variant for which it is sufficient to assume that $\delaybound$ is an upper bound on the maximum unavailability (denoted by $\outbound$ previously).
Let $\vt[\gsumupper]=\gbound^2(1+2\delaybound)(\card(\vt[\set])+\delaybound+1)$ so that $\vt[\step]=\radius/\sqrt{\vt[\gsumupper]}$.
We choose a permutation $\Dpermu$ that satisfies if $\vt[\gsumupper][\runalt]<\vt[\gsumupper]$ then $\Dpermu^{-1}(\runalt)<\Dpermu^{-1}(\run)$
(we just need to sort the time indices using $\vt[\gsumupper]$ and map to this new order).
From \cref{asm:card-in-order} and the definition of $\vt[\gsumupper]$ we know that $\Dpermu$ is a faithful permutation.
Moreover, $(\vt[\gsumupper])_{\run}$ is non-decreasing along $\Dpermu$:
indeed, if this were not the case---that is, if $\vt[\gsumupper][\Dpermu(\run+1)]<\vt[\gsumupper][\Dpermu(\run)]$ for some $\run$---we would have $\run+1=\Dpermu^{-1}(\Dpermu(\run+1))<\Dpermu^{-1}(\Dpermu(\run))=\run$, a contradiction.

We now proceed to prove $\card(\vt[\setout^{\Dpermu}])\le\delaybound$, or equivalently $\card(\vt[\set][\Dpermu(\run)])\ge\run-1-\delaybound$.
For this we show $\oneto{\run}^{\Dpermu}\subseteq\oneto{\card(\vt[\set][\Dpermu(\run)])+\delaybound+1}$, which implies $\run\le\card(\vt[\set][\Dpermu(\run)])+\delaybound+1$ and thus the above inequality.
Provided that $\vt[\gsumupper]$ is non-decreasing along $\Dpermu$, for $\runalt\le\run$ we have $\card(\vt[\set][\Dpermu(\runalt)])\le\card(\vt[\set][\Dpermu(\run)])$.
Using the bounded unavailability assumption we get $\card(\setexclude{\oneto{\Dpermu(\runalt)-1}}{\vt[\set][\Dpermu(\runalt)]})\le\delaybound$ so that 
$\Dpermu(\runalt)-1-\card(\vt[\set][\Dpermu(\runalt)])\le\delaybound$ and subsequently $\Dpermu(\runalt)\le\card(\vt[\set][\Dpermu(\run)])+\delaybound+1$. This proves $\oneto{\run}^{\Dpermu}\subseteq\oneto{\card(\vt[\set][\Dpermu(\run)])+\delaybound+1}$.

From $\card(\vt[\setout^{\Dpermu}])\le\delaybound$ it follows immediately $\vt[\gsumpar^\Dpermu]
\defeq\dnorm{\vt[\gvec][\Dpermu(\run)]}^2
		+2\dnorm{\vt[\gvec][\Dpermu(\run)]}
		\sum_{\runalt\in
		\vt[\setout^{\Dpermu}]}\dnorm{\vt[\gvec][\runalt]}\le\gbound^2(1+2\delaybound)$ for all $\run$.
Along with $\run\le\card(\vt[\set][\Dpermu(\run)])+\delaybound+1$ we deduce $\vt[\gsum^{\Dpermu}]\le\gbound^2(1+2\delaybound)(\card(\vt[\set][\Dpermu(\run)])+\delaybound+1)=\vt[\gsumupper][\Dpermu(\run)]$.
Applying \cref{thm:delay-regret} and the AdaGrad lemma (\cref{lem:adaptive}), we obtain
\begin{equation}
\notag
\begin{aligned}
	\reg_{\nRuns}(\comp)
	&\le
	\frac{\hreg(\comp)}{\vt[\step][\Dpermu(\nRuns)]}
	+ \frac{1}{2}\sum_{\run=\start}^{\nRuns}
	\vt[\step][\Dpermu(\run)]
	\Bigg(\dnorm{\vt[\gvec][\Dpermu(\run)]}^2
		+2\dnorm{\vt[\gvec][\Dpermu(\run)]}
		\sum_{\runalt\in
		\vt[\setout^{\Dpermu}]}\dnorm{\vt[\gvec][\runalt]}\Bigg)\\
	&\le\radius\sqrt{\vt[\gsumupper][\Dpermu(\nRuns)]}
	+ \frac{\radius}{2}\sum_{\run=\start}^{\nRuns}
	\frac{\vt[\gsumpar^\Dpermu]}{\sqrt{\vt[\gsumupper][\Dpermu(\run)]}}\\
	&\le\radius\sqrt{\vt[\gsumupper][\Dpermu(\nRuns)]}
	+ \frac{\radius}{2}\sum_{\run=\start}^{\nRuns}
	\frac{\vt[\gsumpar^\Dpermu]}{\sqrt{\vt[\gsum^{\Dpermu}]}}\\
	&\le\radius\sqrt{\vt[\gsumupper][\Dpermu(\nRuns)]}+\radius\sqrt{\vt[\gsum^\Dpermu][\nRuns]}
	\le 2\radius\sqrt{\vt[\gsumupper][\Dpermu(\nRuns)]}.
\end{aligned}
\end{equation}
Our assertion then follows by noting that $\card(\vt[\set][\Dpermu(\nRuns)])\le\Dpermu(\nRuns)-1\le\nRuns-1$.
\end{proof}

\cref{prop:decen-decr-regret} shows that, even in the fully decentralized case where no global clock is available, it is \emph{still} possible to design implementable algorithms that retain the optimal $\bigoh(\sqrt{\delaybound\nRuns})$ regret bound.
Our next step is to further improve the algorithm so that it can adapt to both the \emph{data} and the \emph{delay} of the feedback.
The aforementioned characterization of delay will turn out to be crucial for this task.

\subsection{Adaptation to Delays in Distributed Systems}
\label{subsec:ada-dist}

To design a learning rate policy that adapts to both data and delays, we have to find a way to estimate $\vt[\gsum]$ by only using local information of each agent.
To that end, define for each agent $\worker$ the \emph{individual} ordering-by-arrival as a permutation $\Dpermu_\worker$ of $\{1,\ldots,\nRuns\}$ such that the $\grun$-th received feedback of $\worker$ comes from $\vt[\state][\Dpermu_i(\grun)]$ (played by $\worker$ or another player), \ie the $\grun$-th received feedback of $\worker$ is $\vt[\gvec][\Dpermu_i(\grun)] \in  \subd\vt[\obj][\Dpermu_i(\grun)](\vt[\state][\Dpermu_i(\grun)])$. 
With this notation, we can define the set of all feedback received \emph{before  $\vt[\gvec][\run]$ by agent $\worker$}; since $\vt[\gvec][\run]$ is the $\Dpermu_i^{-1}(\run)$-th feedback, this set is defined as $\vwt[\setrec] \defeq \{ \Dpermu_i(1),\Dpermu_i(2),\ldots,\Dpermu_i(\Dpermu_i^{-1}(\run)-1) \}$.

Using these definitions and looking closely at the definition of the lag\;\eqref{eq:gsum-sigma}, we notice that:
\begin{enumerate}
\item
The quantity $\sum_{\runalt=\start}^{\run} \dnorm{\vt[\gvec][\Dpermu(\runalt)]}^2$ cannot be known at instant $\Dpermu(\run)$ since the set of gradients available at that time is $\vt[\set][\Dpermu(\run)]$. It is thus natural to consider approximating it by $ \sum_{\runalt\in\vt[\set][\Dpermu(\run)]} \dnorm{\vt[\gvec][\runalt]}^2 $.
\item
For each $\run$ the quantity $\sum_{\{\runalt,\runano\}\in \vt[\setdel^{\Dpermu}]}\dnorm{\vt[\gvec][\runalt]}\dnorm{\vt[\gvec][\runano]}$, gathering the pairs of feedback of $\oneto{\run}^{\Dpermu}$ satisfying the relation $\{\runalt\notin\vt[\set][\runano],\runano\notin\vt[\set][\runalt]\}$ (\cref{prop:setdel-charac}), is generally unknown.
Building on the works of \citet{JGS16} and \citet{MS14}, this sum can be approximated by
\(
\sum_{\runalt\in\vt[\set][\Dpermu(\run)]}
	(\dnorm{\vt[\gvec][\runalt]}
	\sum_{\runano\in\setexclude{\vt[\setrec][\worker(\run),\runalt]}{\vt[\set][\runalt]}}\dnorm{\vt[\gvec][\runano]}).
\)
In words, for all $\runalt\in\vt[\set][\Dpermu(\run)]$, the worker $\activeworker{\run}$ aggregates the feedback received before $\vt[\gvec][\runalt]$ but was not used to generate $\vt[\gvec][\runalt]$.
\end{enumerate}

Putting these two points together, a reasonable surrogate for $\vt[\gsum^\Dpermu]$ would be $\vt[\Gsumbck][\Dpermu(\run)]$, where for all $\run\in\oneto{\nRuns}$, we define
\begin{align}
\notag
	\vt[\Gsumbck] = \sum_{\runalt\in\vt[\set]}
		\left(
		\dnorm{\vt[\gvec][\runalt]}^2 + 2\dnorm{\vt[\gvec][\runalt]}
		\sum_{\runano\in\setexclude{\vt[\setrec][\worker(\run),\runalt]}{\vt[\set][\runalt]}}\dnorm{\vt[\gvec][\runano]}
			\right). 
\end{align}
To make $\vt[\Gsumbck]$ a valid approximation, we would need $\runalt$ to satisfy $\runalt\notin\vt[\set][\runano]$ whenever $\runano\in\setexclude{\vt[\setrec][\worker(\run),\runalt]}{\vt[\set][\runalt]}$ given the characterization of \cref{prop:setdel-charac}.
In particular, this is true if $\vt[\gvec][\runalt]$ is not used to generate $\vt[\state][\runano]$ whenever $\vt[\gvec][\runano]$ arrives before $\vt[\gvec][\runalt]$ at node $\activeworker{\run}$.
This leads to the following mild assumption:
when an agent receives a gradient $\vt[\gvec]$, they must have already received all the feedback that was used to compute it.

\begin{assumption}
\label{asm:in-order}
For every worker $\worker\in\workers$ and all $\run=\running$, we have $\vt[\set]\subseteq\vwt[\setrec]$.
\end{assumption}

The above assumption is notably verified in the following scenarios:
\begin{enumerate*}[\itshape i\upshape)]
\item
a coordinator-worker scheme in which the transmission of the gradients occurs \emph{in order}, in first-come, first-serve manner;
\item
broadcasting of newly received and computed gradient over a fixed communication network;
\item
whenever two agents communicate their gradient pools are synchronized and the gradients are exchanged in the order they become available to the agents.
\end{enumerate*}
As a consequence, \cref{asm:in-order} is satisfied in many relevant setups and can otherwise be enforced by imposing \emph{iii)} at the price of a slightly higher communication cost.

Now, since the active agent $\worker(\run)$ at time $\run$ knows $\vt[\set]$ (by definition) and $\vwt[\setrec][\activeworker{\run}][\runalt]$ for $\runalt\in\vt[\set]$ (by construction), the quantity $\vt[\Gsumbck]$ is indeed computable with purely local information.
The agents can thus run $\eqref{eq:delayed-dual-avg}$ with a learning rate of the form $\vt[\step]=\Theta(1/\sqrt{\vt[\Gsumbck]})$. 
The obtained algorithm, which we call AdaDelay-Dist, is detailed in \cref{algo:adadelay-dist};
its principal regret guarantee is given below:

\begin{restatable}{theorem}{AdaDelayDistRegret}
\label{thm:adadelayi}
Suppose that the maximum delay is bounded by $\delaybound$, the norm of the gradients are bounded by $\gbound$, and that \cref{asm:in-order} holds.
Assume further that \acl{DDA} \eqref{eq:delayed-dual-avg} is run with the learning rate
\begin{align}
\label{eq:adadelayDist}
\tag{AdaDelay--Dist}
	\vt[\step] = \frac{\radius}{\sqrt{  \vt[\Gsumbck]  + \regparinit}}.
\end{align}
where $\regparinit > 0$ is a positive constant.
Then, for all $\comp$ such that $\hreg(\comp)\le\radius^2$,
the algorithm
enjoys the regret bound
\begin{equation}
\notag
	\reg_{\nRuns}(\comp)
	\le
	2\radius\sqrt{\vt[\gsum][\nRuns]}
	+2\radius\sqrt{\regparinit}
	+\frac{\radius}{\sqrt{\regparinit}}
	\gbound^2(2\delaybound+1)^2.
\end{equation}
\end{restatable}

The bound of \cref{thm:adadelayi} differs from the optimal data-dependent bound by at most a time-independent constant, and this is achieved at the worst-case cost of transmitting an additional scalar (\ie $\sum_{\runalt\in\vt[\set]}\dnorm{\vt[\gvec][\runalt]}$) per element of feedback sent. 
Moreover, we should also stress that the algorithm does not use the global time
as in the case of \cref{prop:decen-decr-regret},
time indices are present in \cref{algo:adadelay-dist} only for ease of comprehension, notably to highlight the fact that a worker knows (and keeps track) of the feedback used to produce past points (\ie $\sum_{\runalt\in\vt[\set]}\dnorm{\vt[\gvec][\runalt]}$ for each point $\vt[\state]$ played by the worker).
Finally, notice that although the theorem assumes the gradients and delays to be bounded, the algorithm itself does \emph{not} require any knowledge of these bounds.
A bad estimate of these quantities would only cause the method to suffer from higher regret at the first iterations.

\begin{algorithm}[tb]
    \caption{\ref{eq:adadelayDist} -- from the point of view of agent $\worker$  }
     \label{algo:adadelay-dist}
 \begin{algorithmic}[1]
     \STATE {\bfseries Initialize:}
         $\gvecs_\worker \subs \emptyset$, $\vw[\regpar] \subs \regparinit>0$
    \WHILE{not stopped}
    \STATE {\bfseries asynchronously} {receive  $\vt[\gvec]$ along with $\sum_{\runalt\in\vt[\set]}\dnorm{\vt[\gvec][\runalt]}$ \underline{from other agents}}  
     \STATE \hspace{\algorithmicindent} $\vw[\regpar] \subs \vw[\regpar] + \dnorm{\vt[\gvec]}^2 + 2\dnorm{\vt[\gvec]}
    (\sum_{\runalt\in\vw[\gvecs]}\dnorm{\vt[\gvec][\runalt]}
    -\sum_{\runalt\in\vt[\set]}\dnorm{\vt[\gvec][\runalt]})$
    \STATE \hspace{\algorithmicindent} $\vw[\gvecs] \subs \vw[\gvecs]\union\{\vt[\gvec]\}$
    \STATE \hspace{\algorithmicindent} Relay the information if necessary \\[0.2cm]
    \STATE {\bfseries asynchronously} {receive  $\vt[\gvec]$ \underline{as a feedback}}
    \STATE \hspace{\algorithmicindent} Retrieve $\sum_{\runalt\in\vt[\set]}\dnorm{\vt[\gvec][\runalt]}$ \underline{from the memory}
    \STATE \hspace{\algorithmicindent} $\vw[\regpar] \subs \vw[\regpar] + \dnorm{\vt[\gvec]}^2 + 2\dnorm{\vt[\gvec]}
    (\sum_{\runalt\in\vw[\gvecs]}\dnorm{\vt[\gvec][\runalt]}
    -\sum_{\runalt\in\vt[\set]}\dnorm{\vt[\gvec][\runalt]})$
    \STATE \hspace{\algorithmicindent} $\vw[\gvecs] \subs \vw[\gvecs]\union\{\vt[\gvec]\}$
    \STATE \hspace{\algorithmicindent} Send $\vt[\gvec]$ and $\sum_{\runalt\in\vt[\set]}\dnorm{\vt[\gvec][\runalt]}$ \underline{to other agents} \\[0.2cm]
    \IF{the agent becomes active, \ie $\activeworker{\run}=\worker$}
     \STATE $\vt[\set] \subs \gvecs_\worker$
     \STATE $\vt[\step] \subs \radius/\sqrt{\vw[\regpar] }$
      \STATE Play $\vt[\state] =  \argmin_{\point\in\points} \sum_{\runalt\in\vt[\set]}\product{\vt[\gvec][\runalt]}{\point} + \frac{\hreg(\point)}{\vt[\step]}$
     \ENDIF
     \ENDWHILE
 \end{algorithmic}
 \end{algorithm}

We will now proceed to prove \cref{thm:adadelayi}.
To that end, let
$\vwt[\setbck] \defeq
\setdef{\{\runalt, \runano\}}
{\runalt\in\vt[\set],\runano\in
\setexclude{\vwt[\setrec][\worker][\runalt]}{\vt[\set][\runalt]}}$ so that
\begin{align}
\label{eq:gsumbcki-setbck}
	\vt[\Gsumbck] = \sum_{\runalt\in\vt[\set]}
		\left(
		\dnorm{\vt[\gvec][\runalt]}^2 + 2\dnorm{\vt[\gvec][\runalt]}
		\sum_{\runano\in\setexclude{\vt[\setrec][\worker(\run),\runalt]}{\vt[\set][\runalt]}}\dnorm{\vt[\gvec][\runano]}
			\right) 
		= 
		\sum_{\runalt\in\vt[\set]}
		\dnorm{\vt[\gvec][\runalt]}^2 + 2
			   \sum_{\{\runalt,\runano\}\in  \vwt[\setbck][\activeworker{\run}]}
		\dnorm{\vt[\gvec][\runalt]}
\dnorm{\vt[\gvec][\runano]}
\end{align}
To simplify the notation, we will write $\vt[\setbck]=\vwt[\setbck][\activeworker{\run}]$.
In the following proposition, we show that $\vt[\setbck]$ can be characterized in the same way as $\vt[\setdel^{\Dpermu}]$.

\begin{restatable}{proposition}{SetbckCharac}
\label{prop:setbck-charac}
Let $\Dpermu$ be a faithful permutation and let \cref{asm:in-order} hold. Then
\begin{equation}
\notag
	\vt[\setbck]=
	\setdef{\{\runalt,\runano\}\subseteq\vt[\set]}{\text{$\runalt$ and $\runano$ are not adjacent in $\graph$}}
\end{equation}
\end{restatable}
\begin{proof}
The proof is similar to that of \cref{prop:setdel-charac}, and we defer it to \cref{apx:adaptive}.
\end{proof}

Thanks to \cref{prop:setdel-charac} and \cref{prop:setbck-charac}, comparing $\vt[\setdel^{\Dpermu}]$ with $\vt[\setbck][\Dpermu(\run)]$ amounts to comparing $\oneto{\run}^{\Dpermu}$ with $\vt[\set][\Dpermu(\run)]$.
Using the bounded delay assumption, we can prove the following properties on a faithful permutation.

\begin{restatable}{proposition}{FaithfulDelayBound}
\label{prop:decent-delay-bound}
Let $\Dpermu$ be a faithful permutation and assume that the maximum delay is bounded by $\delaybound$.
We have \emph{(}a\emph{)} $\oneto{\run}^{\Dpermu}\subseteq\oneto{\Dpermu(\run)+\delaybound}$;  \emph{(}b\emph{)} $\setexclude{\oneto{\run}^{\Dpermu}}{\vt[\set][\Dpermu(\run)]}\subseteq\intinterval{\Dpermu(\run)-\delaybound}{\Dpermu(\run)+\delaybound}$; and \emph{(}c\emph{)} $|\Dpermu(\run)-\run|\le\delaybound$.
\end{restatable}

\begin{proof}[Main idea of the proof]
Proving (\emph{a}) and (\emph{b}) simply uses the definition of faithful permutations and the maximum delay, while to prove (\emph{c}) we also leverage the fact that $\Dpermu$ is a permutation.
To streamline our discussion, the complete proof is again deferred to \cref{apx:adaptive}.
\end{proof}

Interestingly, \cref{prop:decent-delay-bound}(\emph{c}) shows that when the delays are bounded by $\delaybound$, a faithful permutation can at most move an element $\delaybound$ steps away from its original position.
We are now ready to provide the complete proof of \cref{thm:adadelayi}.

\begin{proof}[Proof of \cref{thm:adadelayi}]
Let $\vt[\gsumupper]=\vt[\Gsumbck]+\regparinit$ so that $\vt[\step]=\radius/\sqrt{\vt[\gsumupper]}$ and $\Dpermu$ be a permutation such that \emph{i}) if $\vt[\gsumupper][\runalt]<\vt[\gsumupper]$ then $\Dpermu^{-1}(\runalt)<\Dpermu^{-1}(\run)$; \emph{ii}) if $\vt[\gsumupper][\runalt]=\vt[\gsumupper]$ and $\runalt\in\vt[\set]$ then $\Dpermu^{-1}(\runalt)<\Dpermu^{-1}(\run)$. $(\vt[\gsumupper])_{\run}$ is obviously non-decreasing along $\Dpermu$ (see proof of \cref{prop:decen-decr-regret}). We claim that this is a faithful permutation. 
For this, let $\runalt\in\vt[\set]$ and we would like to show  $\Dpermu^{-1}(\runalt)<\Dpermu^{-1}(\run)$. By \cref{asm:in-order} we have
$\vt[\set][\runalt]\subseteq\vwt[\setrec][\activeworker{\run}][\runalt]$ and from $\runalt\in\vt[\set]$ it holds $\vwt[\setrec][\activeworker{\run}][\runalt]\subseteq\vt[\set]$; accordingly, $\vt[\set][\runalt]\subseteq\vt[\set]$.
Invoking \cref{prop:setbck-charac} we deduce $\vt[\setbck][\runalt]\subseteq\vt[\setbck]$.
Using \eqref{eq:gsumbcki-setbck} we then get $\vt[\gsumupper][\runalt]\le\vt[\gsumupper]$. This inequality along with $\runalt\in\vt[\set]$ imply $\Dpermu^{-1}(\runalt)<\Dpermu^{-1}(\run)$.

In the remainder of the proof, we will use the notation $\vt[\Gsumbck][\run]=\vt[\gsum][\nRuns]=\vt[\gsum^\Dpermu][\nRuns]$ for $\run>\nRuns$.
Let us prove that $\vt[\Gsumbck][\Dpermu(\run)+2\delaybound+1]\ge\vt[\gsum^\Dpermu]$ for $\run\in\oneto{\nRuns}$.
This is the case when $\Dpermu(\run)+2\delaybound+1>\nRuns$ by the previous definition.
Otherwise, with \eqref{eq:gsum-sigma}, \eqref{eq:gsumbcki-setbck}, \cref{prop:setdel-charac,prop:setbck-charac}, this is equivalent to proving that $\oneto{\run}^{\Dpermu}\subseteq\vt[\set][\Dpermu(\run)+2\delaybound+1]$.
The inclusion holds since on one hand, by \cref{prop:decent-delay-bound}(a) we have $\oneto{\run}^{\Dpermu}\subseteq\oneto{\Dpermu(\run)+\delaybound}$ and on the other hand $\oneto{\Dpermu(\run)+\delaybound}\subseteq\vt[\set][\Dpermu(\run)+2\delaybound+1]$ by the definition of maximum delay. 

As we have proved that $\Dpermu$ is a faithful permutation, it holds $\vt[\set][\Dpermu(\run)]\subseteq\oneto{\run}^{\Dpermu}$. The above hence also implies $\vt[\set][\Dpermu(\run)]\subseteq\vt[\set][\Dpermu(\run)+2\delaybound+1]$, and accordingly, $\vt[\gsumupper][\Dpermu(\run)+2\delaybound+1]\ge\vt[\gsumupper][\Dpermu(\run)]$.
The inequality is still true when $\Dpermu(\run)+2\delaybound+1>\nRuns$ as $\vt[\Gsumbck]\le\vt[\gsum][\nRuns]$ always holds by \cref{prop:setdel-charac,prop:setbck-charac} and $\vt[\set]\subseteq\oneto{\nRuns}$.
Applying \cref{thm:delay-regret} gives
\begin{align*}
	\reg_{\nRuns}(\comp)
	&\le
	\frac{\hreg(\comp)}{\vt[\step][\Dpermu(\nRuns)]}
	+ \frac{1}{2}\sum_{\run=\start}^{\nRuns}
	\vt[\step][\Dpermu(\run)]
	\Bigg(\dnorm{\vt[\gvec][\Dpermu(\run)]}^2
		+2\dnorm{\vt[\gvec][\Dpermu(\run)]}
		\sum_{\runalt\in
		\vt[\setout^{\Dpermu}]}\dnorm{\vt[\gvec][\runalt]}\Bigg)\\
	&\le
	\radius\sqrt{\vt[\gsumupper][\Dpermu(\nRuns)]}
	+ \frac{\radius}{2}\sum_{\run=\start}^{\nRuns}
	\frac{\vt[\gsumpar^\Dpermu]}{\sqrt{\vt[\gsumupper][\Dpermu(\run)]}}\\
	&=
	\radius\sqrt{\vt[\gsumupper][\Dpermu(\nRuns)]}
	+ \frac{\radius}{2}
	\sum_{\run=\start}^{\nRuns}
	\left(
	\frac{1}{\sqrt{\vt[\gsumupper][\Dpermu(\run)+2\delaybound+1]}}
	+\frac{1}{\sqrt{\vt[\gsumupper][\Dpermu(\run)]}}
	-\frac{1}{\sqrt{\vt[\gsumupper][\Dpermu(\run)+2\delaybound+1]}}
	\right)\vt[\gsumpar^\Dpermu],
\end{align*}
where as in the proof of \cref{prop:decen-decr-regret} we write $\vt[\gsumpar^\Dpermu]=\dnorm{\vt[\gvec][\Dpermu(\run)]}^2
		+2\dnorm{\vt[\gvec][\Dpermu(\run)]}
		\sum_{\runalt\in
		\vt[\setout^{\Dpermu}]}\dnorm{\vt[\gvec][\runalt]}$.
From \cref{prop:decent-delay-bound}(b) we know that $\setexclude{\oneto{\run}^{\Dpermu}}{\vt[\set][\Dpermu(\run)]}\subseteq\intinterval{\Dpermu(\run)-\delaybound}{\Dpermu(\run)+\delaybound}$. Since $\oneto{\run-1}^{\Dpermu}=\setexclude{\oneto{\run}^{\Dpermu}}{\{\Dpermu(\run)\}}$ and $\Dpermu(\run)\notin\vt[\set][\Dpermu(\run)]$, we deduce that $\card(\vt[\setout^{\Dpermu}])\le2\delaybound$ and hence $\vt[\gsumpar^{\Dpermu}] \le \gbound^2(1+4\delaybound)$.
With the non-negativity of $1/\sqrt{\vt[\gsumupper][\Dpermu(\run)]}-1/\sqrt{\vt[\gsumupper][\Dpermu(\run)+2\delaybound+1]}$
and the fact that $\vt[\gsum^\Dpermu]\le\vt[\Gsumbck][\Dpermu(\run)+2\delaybound+1]<\vt[\gsumupper][\Dpermu(\run)+2\delaybound+1]$ we then get

\begin{align*}
    \reg_{\nRuns}(\comp)
	&\le\radius\sqrt{\vt[\gsumupper][\Dpermu(\nRuns)]}
	+ \frac{\radius}{2}\sum_{\run=\start}^{\nRuns}
	\frac{\vt[\gsumpar^\Dpermu]}{\sqrt{\vt[\gsum^{\Dpermu}]}}
	+ \frac{\radius}{2}\sum_{\run=\start}^{\nRuns}
	\left(
	\frac{1}{\sqrt{\vt[\gsumupper][\Dpermu(\run)]}}
	-\frac{1}{\sqrt{\vt[\gsumupper][\Dpermu(\run)+2\delaybound+1]}}
	\right)\gbound^2(1+4\delaybound)
	\\
	&\le\radius\sqrt{\vt[\gsumupper][\Dpermu(\nRuns)]}+\radius\sqrt{\vt[\gsum^\Dpermu][\nRuns]}
	+\frac{\radius}{2}\sum_{\run=\start}^{\nRuns}
	\left(
	\frac{1}{\sqrt{\vt[\gsumupper][\run]}}
	-\frac{1}{\sqrt{\vt[\gsumupper][\run+2\delaybound+1]}}
	\right)\gbound^2(1+4\delaybound)
	\\
	&\le
	2\radius\sqrt{\vt[\gsum][\nRuns]+\regparinit}
	+\frac{\radius}{2\sqrt{\regparinit}}
	(2\delaybound+1)(4\delaybound+1)\gbound^2\\
	&\le
	2\radius\sqrt{\vt[\gsum][\nRuns]}
	+2\radius\sqrt{\regparinit}
	+\frac{\radius}{\sqrt{\regparinit}}
	(2\delaybound+1)^2\gbound^2
\end{align*}
The second inequality uses \cref{lem:adaptive} and reorders the timestamps of the sum;
the third inequality upper bounds both $\vt[\gsumupper][\Dpermu(\nRuns)]$ and $\vt[\gsum^\Dpermu][\nRuns]=\vt[\gsum][\nRuns]$ by $\vt[\gsum][\nRuns]+\regparinit$ for the first term, and uses telescoping and lower bounds $\vt[\gsumupper]$ by $\regparinit$ for the second term;
in the last inequality we employ the fact that $\sqrt{a+b}\le\sqrt{a}+\sqrt{b}$ for all $a,b\ge0$.
This concludes the proof.
\end{proof}

\subsection{Adaptation to Unbounded Delays in the Single-Agent Setting}
\label{subsec:adadelayplus}

In this part, we will show that when there is only one agent (\ie $\nWorkers=1$), we can extend the ideas developed in the previous section to cope even with \emph{unbounded} delays.
In fact, in this situation the agent knows exactly the delay of each feedback and how each iterate is computed, so they can tune their learning rate accordingly.
This is in sharp contrast with the decentralized case in which the agents are in general unable to estimate the number of actions that have been played in the network but for which they have not received the corresponding feedback (\ie $\card(\vt[\setout])$).

To put all this in motion, let $\gbound$ be an upper bound on the norms of gradients that we assume to be known by the agent, and let $\vt[\setrec]=\vt[\setrec][1,\run]$ denotes the set of feedback (represented by their timestamps) received before $\vt[\gvec]$.
Our goal is to provide an upper bound of $\vt[\gsum]=\vt[\gsum^{\idp}]$ that is as tight as possible. As in \cref{subsec:ada-dist}, this is done in two steps (we write below $\vt[\setdel]=\vt[\setdel^{\idp}]$ for simplicity)

\begin{enumerate}
\item
The quantity $\sum_{\runalt=\start}^{\run}
		\dnorm{\vt[\gvec][\runalt]}^2$ can be approximated by $ \sum_{\runalt\in\vt[\set][\run]}
		\dnorm{\vt[\gvec][\runalt]}^2 $. Clearly,
		\[\sum_{\runalt=\start}^{\run}
		\dnorm{\vt[\gvec][\runalt]}^2\le
		\sum_{\runalt\in\vt[\set][\run]}
		\dnorm{\vt[\gvec][\runalt]}^2 + \gbound^2(\card(\vt[\setout])+1);\]
\item 
A proxy for $\sum_{\{\runalt,\runano\}\in \vt[\setdel]}\dnorm{\vt[\gvec][\runalt]}\dnorm{\vt[\gvec][\runano]}$,
	is $\sum_{\runalt\in\vt[\set]}
	(\dnorm{\vt[\gvec][\runalt]}
	\sum_{\runano\in\setexclude{\vt[\setrec][\runalt]}{\vt[\set][\runalt]}}\dnorm{\vt[\gvec][\runano]})$.
	Thanks to \cref{prop:setdel-charac} and \cref{prop:setbck-charac}, we have indeed
	\begin{equation}
		\notag
		\sum_{\{\runalt,\runano\}\in \vt[\setdel]}\dnorm{\vt[\gvec][\runalt]}\dnorm{\vt[\gvec][\runano]}
		\le
		\sum_{\runalt\in\vt[\set]}
		\bigg(\dnorm{\vt[\gvec][\runalt]}
		\sum_{\runano\in\setexclude{\vt[\setrec][\runalt]}{\vt[\set][\runalt]}}\dnorm{\vt[\gvec][\runano]}\bigg)
		+ \gbound^2 (\card(\vt[\setdel]) - \card(\vt[\setbck])).
	\end{equation}
\end{enumerate}

\begin{algorithm}[tb]
   \caption{\ref{eq:adadelayOplus}}
    \label{algo:adadelayOplus}
\begin{algorithmic}[1]
    \STATE {\bfseries Initialize:}
        $\gvecs \subs \emptyset$, $\run\subs1$, $\delayres\subs0$, $\regparint \subs 0$.\\[0.2em]
    \WHILE{not stopped}
    \IF{receive feedback $\vt[\gvec]$}
    \STATE $\delayres \subs \delayres - 1 - 2(\card(\gvecs)-\card(\vt[\set]))$
    \STATE $\regparint \subs \regparint + \dnorm{\vt[\gvec]}^2 + 2\dnorm{\vt[\gvec]}
    (\sum_{\runalt\in\gvecs}\dnorm{\vt[\gvec][\runalt]}
    -\sum_{\runalt\in\vt[\set]}\dnorm{\vt[\gvec][\runalt]})$ \label{algo:adadelayOplus:subtract}
    \STATE $\gvecs \subs \gvecs\union\{\vt[\gvec]\}$
    \ELSIF{requested to play an action $\vt[\state]$}
    \STATE $\vt[\set] \subs \gvecs$
    \STATE $\delayres \subs \delayres + 1 + 2((\run-1)-\card(\vt[\set]))$
    \STATE $\regpar \subs \max(\regpar, \regparint+\gbound^2\delayres)$
    \STATE $\vt[\state] \subs \argmin_{\point\in\points} \sum_{\runalt\in\vt[\set]}\product{\vt[\gvec][\runalt]}{\point} + (\sqrt{\regpar}/\radius)\hreg(\point)$
    \STATE $\run\subs\run+1$
    \ENDIF
    \ENDWHILE
\end{algorithmic}
\end{algorithm} 
In summary, we have shown that $\vt[\gsum]\le\vt[\Gsumbck]+\gbound^2\vt[\delayres]$ where $\vt[\delayres] \defeq \run + 2\vt[\totaldelay] - \card(\vt[\set]) - 2\card(\vt[\setbck])$.
This has the following immediate consequences:

\begin{restatable}{proposition}{PropAdaDelayOPlus}
\label{prop:adadelayoplus}
Assume that the norms of the gradients are bounded by $\gbound$ and the sequence of active feedback is non-decreasing, \ie $ \vt[\set] \subseteq  \vt[\set][\run+1]$.
Assume further that \acl{DDA} \eqref{eq:delayed-dual-avg} is run with the learning rate sequence
\begin{align}
\label{eq:adadelayOplus}
\tag{AdaDelay+}
	\vt[\step] = \min\left(  \vt[\step][\run-1] ,  \frac{\radius}{\sqrt{  \vt[\Gsumbck]  + \gbound^2 \vt[\delayres] }} \right)
\end{align}
where $\vt[\delayres] = \run + 2\vt[\totaldelay] - \card(\vt[\set]) - 2\card(\vt[\setbck])$.
Then, for any $\comp$ such that $\hreg(\comp)\le\radius^2$, the generated points $\seq{\state}{\start}{\nRuns}$ enjoy the regret bound
\begin{equation}
\notag
	\reg_{\nRuns}(\comp)
	\le
	2\radius\max_{1\le\run\le\nRuns}
	\sqrt{\vt[\Gsumbck] + \gbound^2\vt[\delayres]}
	\le
	2\radius
	\min\left(\max_{1\le\run\le\nRuns}
	\sqrt{\vt[\gsum]+\gbound^2\vt[\delayres]},
	\gbound\sqrt{\nRuns+2\vt[\totaldelay][\nRuns]}\right).
\end{equation}
\end{restatable}
\begin{proof}
The proof is detailed in \cref{apx:adaptive}.
We apply \cref{thm:delay-regret} with the choice $\Dpermu=\idp$ and conclude by using the inequality $\vt[\gsum]\le\vt[\Gsumbck]+\gbound^2\vt[\delayres]$ and the AdaGrad lemma (\cref{lem:adaptive}). 
\end{proof}

We refer to this new adaptive scheme as \ref{eq:adadelayOplus} and we provide one possible pseudocode implementation as \cref{algo:adadelayOplus}.
Notice that we do not use directly $\vt[\step]=\radius/\sqrt{\vt[\Gsumbck]+\gbound^2\vt[\delayres]}$ since we want the learning rate to be non-increasing.
To the best of our knowledge, \ref{eq:adadelayOplus} is the first online algorithm with regret guarantees that are both data- \emph{and} delay-dependent, all the while bypassing the bounded delay assumption.
In particular, its regret bound achieves the best of both worlds:
\begin{enumerate}
\item
When the delays are bounded by $\delaybound$, we have $\delayres\le2\delaybound^2+3\delaybound+1$ (proved in \cref{apx:adaptive}), so this worst-case bound still outperforms (by an additive constant) the data-dependent bound of \cref{thm:adadelayi}.
In the same setting, \citet{JGS16} also proposed an adaptive algorithm based on FTRL-Prox with a regret bound of the same order.
\item
It also achieves the optimal square-root dependence on the cumulative unavailability $\vt[\totaldelay][\nRuns]$ no matter whether the delays are bounded or not.
\end{enumerate}

\section{An Optimistic Variant}
\label{sec:optimistic}

In previous sections, we have established regret guarantees with respect to the worst case scenario.
In particular, the losses that we encounter can be arbitrary, and even adversarial.
Nonetheless, the environment can have a much more benign nature: there may be patterns of loss functions which can be exploited to achieve a smaller regret (\eg losses generated by a game mechanism, slowly-varying function sequence).
In this spirit, optimistic algorithms exploit the predictability of the loss sequence to obtain improved regret bounds.
In the unconstrained Euclidean setup ($\points = \vecspace$, $\hreg = 1/2\|\cdot\|^2$) that we will focus on in the following, the algorithm writes
\begin{equation}
\label{eq:OptGD}
\tag{OptGD}
\begin{aligned}
    \vt[\state] &= \last - \step\,\past[\gvec],
    ~~~~
    \inter &= \vt[\state] - \step\,\inter[\appr].
\end{aligned}
\end{equation}
The first update $ \vt[\state] = \last - \step\past[\gvec]$ is a classical online gradient step. However, for optimistic methods, the point $ \vt[\state]$ \emph{is not played} at time $\run$; instead, the agent plays $ \inter = \vt[\state] - \step\,\inter[\appr]$ after sensing the gradient of $\vt[\obj]$ by designing a gradient \emph{guess}  $\inter[\appr]=\inter[\appr](\vt[\state][\start],\vt[\gvec][\frac{3}{2}],\ldots,\past[\gvec])$. This is the \emph{optimistic step}. Following this action, the player suffers a loss $\vt[\obj](\inter)$ and receives the feedback $\inter[\gvec]\in\partial\vt[\obj](\inter)$.

The regret of 
\eqref{eq:OptGD} was shown 
\citep{CYLM+12,JGS17,MY16,RS13-NIPS} to be bounded by
\begin{equation}
\label{eq:OptGD-regret}
\reg_{\nRuns}(\comp) \le \frac{\norm{\comp-\vt[\state][\start]}^2}{2\step} + \sum_{\run=1}^{\nRuns}\frac{\step}{2}\norm{\inter[\gvec]-\inter[\appr[\gvec]]}^2.
\end{equation}
By optimally choosing $\step$, we attain a regret in $\bigoh\left(\sqrt{\sum_{\run=1}^{\nRuns}\norm{\inter[\gvec]-\inter[\appr[\gvec]]}^2}\right)$.

This bound gets smaller as $\inter[\appr]$ gets closer to $\inter[\gvec]$ (\ie when the {optimistic} guess is good), while we recover the regret of vanilla online gradient descent for $\inter[\appr]=0$ (no optimistic guess). A possible choice in practice is to use the last received feedback as a guess, \ie $\inter[\appr]=\past[\gvec]$, in which case, favorable guarantees can be derived when the function sequence has a small total variation and when these functions are smooth (see \eg \citealp{CYLM+12,JGS17}).

In this section, we present how Delayed Dual Averaging can be extended to incorporate an optimistic step in the unconstrained Euclidean setup. 
Importantly, we show that the dual averaging step has to be done with a smaller learning rate than the optimistic step.

\subsection{Delayed Optimistic Dual Averaging}

While \ac{OptGD} successfully leverages the predictability of the loss sequence for achieving a smaller regret, the effect of delay on this algorithm remains, as far as we are aware, unknown.

By extending \eqref{eq:delayed-dual-avg} to incorporate an optimistic step, \acl{DOptDA} can then be stated as follows:\footnote{The same algorithm (in a more general form) is called optimistic FTRL in \cite{JGS17}. We choose to employ the term optimistic dual averaging to maintain consistency with preceding sections.}
\begin{equation}
\label{eq:D-ODA}
\tag{DOptDA}
\begin{aligned}
    \vt[\state] &= \argmin_{\point\in\vecspace} \sum_{\runalt\in\vt[\set]}\product{\inter[\gvec][\runalt]}{\point} + \frac{\norm{\point-\vt[\point][\start]}^2}{2\vt[\step]}
    = \vt[\state][\start] - \vt[\step] \sum_{\runalt\in\vt[\set]} \inter[\gvec][\runalt],
    \\
    \inter &= \argmin_{\point\in\vecspace}\thinspace
    \product{\inter[\appr]}{\point} + \frac{\norm{\point-\vt[\point]}^2}{2\vt[\stepalt]}
    = \vt[\state] - \vt[\stepalt]\,\inter[\appr].
\end{aligned}
\end{equation}
Following our delay framework, $\current$ is computed using gradients from time moments $\vt[\set]$. Similarly, $\inter[\appr]$ must be derived solely based on information available to the active agent $\activeworker{\run}$ at time $\run$.

One key feature of our algorithm is we allow the optimistic step (\ie the step that leads to $\inter$) of \eqref{eq:D-ODA} to use a larger learning rate than the actual update step (\ie the step that obtains $\update$), \ie $\vt[\stepalt]\ge\vt[\step]$.
This additional flexibility allows us to compensate the missing information that have not arrived due to delays and provides the following regret bound
proved in\;\cref{apx:optimistic-delay-regret}.

\begin{restatable}{theorem}{OptDelayRegret}
\label{thm:optimistic-delay-regret}
Assume that the maximum delay is bounded by $\delaybound$. 
Let \acl{DOptDA} \eqref{eq:D-ODA} be run with learning rate sequences $\seqinf[\oneto{\nRuns}]{\step}{\run}$, $\seqinf[\oneto{\nRuns}]{\stepalt}{\run}$ satisfying $\update[\step]\le\current[\step]$ and $(2\delaybound+1)\vt[\step]\le\vt[\stepalt]$ for all $\run$. Then the regret of the algorithm (evaluated at the points $\vt[\state][\frac{3}{2}],\ldots,\vt[\state][\nRuns+\frac{1}{2}]$) satisfies
\begin{equation}
\notag
    \reg_{\nRuns}(\comp)
    \le
    \frac{\norm{\comp-\vt[\state][\start]}^2}{2\vt[\step][\nRuns]}
    +\sum_{\run=\start}^{\nRuns}
    \frac{\vt[\stepalt]}{2}\left(\norm{\inter[\gvec]-\inter[\appr[\gvec]]}^2-\norm{\inter[\appr]}^2\right).
\end{equation}
\end{restatable}

In \cref{thm:optimistic-delay-regret}, we successfully show that \eqref{eq:D-ODA} retains the desired property of undelayed optimistic gradient descent:
the regret of the algorithm is solely determined by the distance between $\inter[\gvec]$ and $\inter[\appr]$ (see \cref{eq:OptGD-regret}).
Precisely, the theorem guarantees a regret in $\bigoh\left(\sqrt{\delaybound\sum_{\run=1}^{\nRuns}\norm{\inter[\gvec]-\inter[\appr[\gvec]]}^2}\right)$ for fix learning rate sequences $\vt[\step]\equiv\step$, $\vt[\stepalt]\equiv(2\delaybound+1)\step$ that are optimally chosen.
Similar to the case of delayed mirror descent and delayed dual averaging, an additional factor of $\sqrt{\delaybound}$ appears in the regret bound, and their regret is recovered tightly by setting $\inter[\appr]=0$.

\begin{remark*}
The bounded delay assumption can in fact be relaxed in \cref{thm:optimistic-delay-regret}.
Nonetheless, we choose to adopt this assumption for ease of understanding.
Otherwise, denoting $\vt[\delay]=\card(\vt[\setout])+\card(\setdef{\runalt\in\oneto{\nRuns}}{\run\in\vt[\setout][\runalt]})+1$ and employing a constant update learning rate $\vt[\step]\equiv\step$ and $\vt[\stepalt]=\vt[\delay]\step$, we achieve a regret in $\bigoh\left(\sqrt{\sum_{\run=\start}^{\nRuns}\vt[\delay]\norm{\inter[\gvec]-\inter[\appr[\gvec]]}^2}\right)$.
Note that $\sum_{\run=\start}^\nRuns\vt[\delay]=2\totaldelay+\nRuns$ and when $\,\inter[\appr[\gvec]]=0$ the bound can be inferred from \cref{thm:delay-regret} with the choice $\Dpermu=\idp$.
\end{remark*}

\subsection{The Necessity of Scale Separation for Robustness to Delays}

In the following, we discuss the \emph{necessity} of having a relatively aggressive optimistic step compared to the update ($\vt[\stepalt]\ge\vt[\step]$) in order to be robust to delay.\footnote{The optimistic step is also called \emph{extrapolation} step to mirror the vocabulary of the extragradient method \cite{Kor76}.} Note that taking a more aggressive extrapolation update compared to the actual state update was shown to clearly improve the robustness of the extragradient method with respect to both rates and convergence itself in \cite{hsieh2020explore}.

For this, we consider linear losses $\vt[\obj]=\product{\vt[\gvec]}{\cdot}$ and uniform delay $\delaybound$ (\ie every feedback becomes available after a delay of $\delaybound$ time steps).\footnote{For linear losses, the gradient does not depend on the calling point and thus $\inter[\gvec] = \nabla\vt[\obj](\inter[\state]) = \vt[\gvec]$.}
We define the \emph{$\delaybound$-variation} of the loss sequence by $\vt[\variation^{\delaybound}][\nRuns]=\sum_{\run=\start}^{\nRuns}\norm{\vt[\gvec]-\vt[\gvec][\run-\delaybound]}^2$ where we set $\vt[\gvec]=0$ for $\run\le0$. For ease of notation we further denote $\vt[\variation^{\delaybound^+}][\nRuns]=\vt[\variation^{\delaybound+1}][\nRuns]$.
The following corollary is immediate from \cref{thm:optimistic-delay-regret}. 

\begin{corollary}
\label{cor:delay-optimistic-linear}
In the context of linear losses $\vt[\obj]=\product{\vt[\gvec]}{\cdot}$ and uniform delay $\delaybound$ ($\vt[\set]=\oneto{\run-\delaybound-1}$ for all $\run$), running \acl{DOptDA} \eqref{eq:D-ODA} with $\inter[\appr]=\vt[\gvec][\run-\delaybound-1]$ and constant learning rates $\step=\radius/\sqrt{(2\delaybound+1)\vt[\variation^{\delaybound^+}][\nRuns]}$ and $\stepalt=(2\delaybound+1)\step$ where $\radius\ge\norm{\comp-\vt[\state][\start]}$ guarantees the regret bound
\begin{equation}
    \notag
    \reg_{\nRuns}(\comp)\le\radius\sqrt{(2\delaybound+1)\vt[\variation^{\delaybound^+}][\nRuns]}.
\end{equation}
\end{corollary}

This results indicates that with an optimistic learning rate $\stepalt$ taken  $(2\delaybound+1)$ times bigger than the update learning rate $\step$, one can guarantee a regret bound of the order of the square root of the $(\delaybound+1)$-variation. 
In contrast, we now demonstrate the impossibility to obtain a regret  that is sub-linear in $\vt[\variation^{\delaybound^+}][\nRuns]$ when $\stepalt=\step$ (or even when $\stepalt\le\delaybound\step$).

\begin{restatable}{theorem}{OptRegretLB}
\label{thm:optimistic-regret-lower-bound}
Consider the setup of \cref{cor:delay-optimistic-linear}. Let $\step=\step(\radius,\nRuns,\delaybound,\vt[\variation^{\delaybound^+}][\nRuns])$ be uniquely determined by $\radius\ge \norm{\comp-\vt[\state][\start]}$, the time horizon $\nRuns$, the uniform delay $\delaybound$, and the $(\delaybound+1)$-variation $\vt[\variation^{\delaybound^+}][\nRuns]$.
If we run \acl{DOptDA} 
\eqref{eq:D-ODA} with $\inter[\appr]=\vt[\gvec][\run-\delaybound-\frac{1}{2}]$ and $\stepalt\le\delaybound\step$, it is impossible to guarantee a regret in $\smalloh(\max(\vt[\variation^{\delaybound^+}][\nRuns], \sqrt{\nRuns}))$.
\end{restatable}

\begin{proof}
The proof is reported in \cref{apx:optimistic-regret-lower-bound}; its construction is partially inspired by \citet{CYLM+12}, and as a special case, in the undelayed setting, we recover the result that the optimistic step is necessary to guarantee a regret in $\bigoh\Big(\sqrt{\sum_{\run=\start}^{\nRuns}\norm{\vt[\gvec]-\vt[\gvec][\run-1]}^2}\Big)$.

Nonetheless, in the original proof of \cite{CYLM+12}, the learning rate was first fixed and then a loss sequence was constructed to yield large regret, which could possibly also prevent optimistic algorithms to achieve low regret. Our approach fixes this fallacy by informing the algorithm of the variation in advance so that optimistic algorithms provably obtain low regrets on these sequences (cf. \cref{cor:delay-optimistic-linear}).
\end{proof}

Finally, we also show that among all the online algorithms with the same prior information, the bound achieved in \cref{cor:delay-optimistic-linear} is tight in its dependence on $\delaybound$ and $\vt[\variation^{\delaybound^+}][\nRuns]$.

\begin{restatable}{proposition}{VarRegretLB}
\label{prop:lowerboundopt}
For any online learning algorithm with prior knowledge of $\nRuns$, $\delaybound$ and $\overline{\variation^{\delaybound}}\ge\vt[\variation^{\delaybound^+}][\nRuns]$, there exists a sequence of linear losses such that if the feedback is subject to constant delay $\delaybound$, then the regret of the algorithm on this sequence with respect to a vector\;$\comp$ with $\norm{\comp-\vt[\state][\start]}\le1$ is $\Omega(\sqrt{\delaybound\overline{\variation^{\delaybound}}})$.
\end{restatable}

\begin{proof}
The proof is reported in \cref{apx:lowerboundopt}. It combines the standard $\Omega(\sqrt{\nRuns})$ lower bound of undelayed online learning with idea from \cite{LSZ09}.
\end{proof}

Thus, in this section we showed that using \eqref{eq:D-ODA} \emph{with a double learning rate strategy} enables to achieve a  $\bigoh(\sqrt{\delaybound \vt[\variation^{\delaybound^+}][\nRuns]})$ regret which is tight among online learning methods and out of reach of single learning rate \eqref{eq:D-ODA}.

\subsection{Delayed Online Learning with Slow Variation}

Now that we have laid out our main results concerning the optimistic variant of delayed dual averaging, we investigate the choice of $\inter[\appr]$ for slowly varying loss functions $(\vt[\obj])_{\run\in\oneto{\nRuns}}$.

For this, we consider the case where the full gradient $\grad\vt[\obj]$ is obtained as feedback (and not only $\vt[\gvec] = \grad\vt[\obj](\vt[\state])$).
Using this kind of feedback, we can compute the gradient of the last received function at the current point immediately,\footnote{\ie without any delay, the delays considered here are either due to communication between agents or inherent to the feedback mechanism.} and use it as a guess for the current function's gradient. Formally, we make the following assumption.

\begin{assumption}
\label{asm:whole-vecfield}
The feedback associated to time step $\run$ is the whole vector field $\vt[\vecfield]=\grad\vt[\obj]$, the evaluation of which at any point $\point\in\vecspace$ is immediate and does not induce any delay.
\end{assumption}

The first part of the assumption is sometimes referred to as the ``full-information'' online learning model, and is typically satisfied when the learning system is used for prediction (\eg classification, regression). 
In fact, in such problems, the actions of the agents represent the model parameters, for which the whole loss and its gradient can be computed once the corresponding data is observed \citep{SS11}.

With this assumption, we can set  $\inter[\appr[\gvec]]=\vt[\appr[\vecfield]](\vt[\state]) $ where $\vt[\appr[\vecfield]] $ is some \emph{past} vector field (\ie $\vt[\appr[\vecfield]] = \vt[\vecfield][\runalt]$ for some $\runalt\in\vt[\set]$).
Now, for smooth losses, the following regret bound can be derived.

\begin{restatable}{theorem}{OptDelayRegretV}
\label{thm:optimistic-delay-regret-V}
Let the maximum delay be bounded by $\delaybound$ and that \cref{asm:whole-vecfield} holds. Assume in addition that the vector fields  $\vt[\vecfield]$ are $\lips$-Lipschitz continuous.
Take $\inter[\appr[\gvec]]=\vt[\appr[\vecfield]](\vt[\state])$, 
$\update[\step]\le\current[\step]$, $(2\delaybound+1)\vt[\step]\le\vt[\stepalt]$, and $2\vt[\stepalt^2]\lips^2\le1$.
Then, the regret of \acl{DOptDA} \eqref{eq:D-ODA} (evaluated at the points $\vt[\state][\frac{3}{2}],\ldots,\vt[\state][\nRuns+\frac{1}{2}]$) satisfies
\begin{equation}
\notag
    \reg_{\nRuns}(\comp)
    \le
    \frac{\norm{\comp-\vt[\state][\start]}^2}{2\vt[\step][\nRuns]}
    +\sum_{\run=\start}^{\nRuns}
    \vt[\stepalt]\norm{\vt[\vecfield](\vt)-\vt[\appr[\vecfield]](\vt)}^2.
\end{equation}
\end{restatable}
\begin{proof}
The proof is immediate from \cref{thm:optimistic-delay-regret} and is deferred to \cref{apx:optimistic-slow-var}.
\end{proof}

\cref{thm:optimistic-delay-regret-V} reduces the problem of choosing an adequate vector $\inter[\appr]$ to that of choosing a vector field $\vt[\appr[\vecfield]]$ which approximates well $\vt[\vecfield]$.
In our setup of full gradient feedback with a loss sequence evolving slowly over time, one natural option is reuse some recent function for the constitution of $\vt[\appr[\vecfield]]$.
Since we are in a distributed setting, the evolution of the loss functions may have both global and local components. We discuss these two typical cases below.

\begin{example}[Global variation]
\label{ex:global-variation}
If the loss functions vary slowly following a global trend, we can timestamp every gradient field which makes it possible to choose $\vt[\appr[\vecfield]]=\vt[\vecfield][\tilde{\run}]$ where $\tilde{\run}=\max\vt[\set]$, \ie the active agent $\activeworker{\run}$ uses the most recent data available at hand (independent of its source) when playing $\vt[\state]$. 
This would however require the agents to share the whole vector field $\vt[\vecfield]$.
\end{example}

\begin{example}[Local variation]
\label{ex:local-variation}
If the loss functions vary slowly for all the agents, the active agent $\activeworker{\run}$ can choose the last feedback corresponding to a point it played, \ie
$\vt[\appr[\vecfield]]=\vt[\vecfield][\tilde{\run}]$ where $\tilde{\run}=\max\setdef{\runalt\in\vt[\set]}{\activeworker{\runalt}=\activeworker{\run}}$.
Compared to \cref{ex:global-variation}, we gain in terms of both data privacy and communication efficiency since only the gradients $\inter[\gvec]$ need to be shared among the agents in this scenario.
\end{example}

Denoting the total deviation of our approximation by $\vt[\variation][\nRuns]=\sum_{\run=\start}^{\nRuns}\norm{\vt[\vecfield](\vt)-\vt[\appr[\vecfield]](\vt)}^2$, \cref{thm:optimistic-delay-regret-V} guarantees a regret in $\bigoh(\radius^2\delaybound\lips+\radius\sqrt{\delaybound\vt[\variation][\nRuns]})$ for suitably chosen constant learning rate sequences $\vt[\step]\equiv\step$ and $\vt[\stepalt]\equiv\stepalt$.
In both \cref{ex:global-variation,ex:local-variation}, $\vt[\variation][\nRuns]$ characterizes some variation of the loss sequence over time.
However, the optimal choice of the $\step$ and $\stepalt$ allowing us to obtain the aforementioned regret guarantee depends on $\vt[\variation][\nRuns]$, which cannot be known in advance.
To circumvent this issue, we can again design an adaptive learning rate schedule in the spirit of AdaGrad by assuming knowledge on an universal bound for the difference $\norm{\vt[\vecfield](\vt)-\vt[\appr[\vecfield]](\vt)}^2$.
For the following result, we simply resort to the standard assumption of bounded gradients.

\begin{restatable}{proposition}{OptAdapt}
\label{prop:optimistic-adaptive}
Let the maximum delay be bounded by $\delaybound$ and let \cref{asm:in-order,asm:whole-vecfield} hold.
Further suppose that $\vt[\vecfield]$ are $\lips$-Lipschitz continuous and both $\vt[\vecfield], \vt[\appr[\vecfield]]$ have their norm bounded by $\gbound$.
Then for any $\comp$ such that $\norm{\comp-\vt[\state][\start]}\le\radius$, running \acl{DOptDA} \eqref{eq:D-ODA} with $\vt[\appr[\gvec]]=\vt[\appr[\vecfield]](\vt[\state])$,
\begin{gather}
    \notag
    \vt[\stepalt]=\min\left( \frac{\radius\sqrt{4\delaybound+1}}{2\sqrt{\left(\sum_{\runalt\in\vt[\set]}\norm{\vt[\vecfield][\runalt](\vt[\state][\runalt])-\vt[\appr[\vecfield]][\runalt](\vt[\state][\runalt])}^2+4\gbound^2(\delaybound+1)\right)}}, \frac{1}{\sqrt{2}\lips}\right),
\end{gather}
and
\begin{gather}
    \notag
    \vt[\step]=\min\left( \frac{\radius}{2\sqrt{(4\delaybound+1)\left(\sum_{\runalt\in\vt[\set]}\norm{\vt[\vecfield][\runalt](\vt[\state][\runalt])-\vt[\appr[\vecfield]][\runalt](\vt[\state][\runalt])}^2+4\gbound^2(3\delaybound+1)\right)}}, \frac{1}{\sqrt{2}\lips (4\delaybound+1)}\right)
\end{gather}
guarantees
\begin{equation}
\notag
    \reg_{\nRuns}(\comp)
    \le
    \max\left(
    \sqrt{2}\radius^2\lips(4\delaybound+1),
    2\radius\sqrt{(4\delaybound+1)(
    \vt[\variation][\nRuns]+4\gbound^2(3\delaybound+1))}
    \right).
\end{equation}
\end{restatable}
\begin{proof}
The proof is deferred to \cref{apx:optimistic-adaptive}. Notice that the adaptive learning rates are not necessarily non-increasing and therefore \cref{thm:optimistic-delay-regret-V} can not be directly applied.
To address this challenge, we rely on the use of faithful permutations and adapt both \cref{thm:optimistic-delay-regret} and \cref{thm:optimistic-delay-regret-V} to accommodate more flexible learning rate schedules.
\end{proof}

Compared to the optimal regret that can be achieved with prior knowledge of $\vt[\variation][\nRuns]$, the bound is only degraded by a constant factor.
To implement this learning rate schedule, the computation of $\vt[\stepalt]$ and $\vt[\step]$ needs to be made possible.
This would require the agents to relay $\norm{\vt[\vecfield](\vt)-\vt[\appr[\vecfield]](\vt)}$ in addition to $\inter[\gvec]=\vt[\vecfield](\inter)$ after receiving $\vt[\vecfield]$. 

\begin{remark*}
At the price of a worse dependence on the constants, we can use the difference $\norm{\vt[\vecfield](\inter)-\vt[\appr[\vecfield]](\current)}$ instead of $\norm{\vt[\vecfield](\vt)-\vt[\appr[\vecfield]](\vt)}$ in the computation of the learning rates, which prevents us from an extra evaluation of the vector field; see \eg \citealp[Corollary 9]{JGS17}.
\end{remark*} 

\section{Discussion}
\label{sec:discussion}

In this section, we discuss several links of our work to other existing results whose detailed presentation would have otherwise interrupted the flow of our paper.

\subsection{Related Work}
\label{subsec:related-work}

Our work lies at the interface between multiple active research areas, each tackling a specific aspect of the general framework considered in this paper.
We provide below a more focused view into each of these topics, namely:
\begin{enumerate*}
[\itshape i\upshape)]
\item online learning with delays;
\item multi-agent online learning;
\item distributed online optimization;
and
\item asynchronous optimization.
\end{enumerate*}

\para{Online learning with delays} 
The research on the delayed feedback problem in online learning was pioneered by \cite{WO02}, in which it was shown that running $\delaycst+1$ independent learners guaranteed the minimax regret $\bigoh(\sqrt{\delaycst\nRuns})$ when the feedback is uniformly delayed by $\delaycst$ time steps. The same strategy was further analyzed by \cite{JGS13} for more complex delay mechanisms.
However, maintaining a pool of learners can be prohibitively resource intensive. Therefore, another line of research focuses on investigating the effect of delays on gradient-based methods.

In \cite{LSZ09}, the same $\bigoh(\sqrt{\delaycst\nRuns})$ bound on the regret was first derived for a slowed-down version of online gradient descent (\ie running the algorithm with smaller learning rates) under the constant delay assumption.
Comprehensive studies were later provided by \citet{MS14}, \citet{QD15} and \citet{JGS16}. 
In more detail, denoting by $\totaldelay$ the aggregated feedback delay after $\nRuns$ rounds, \cite{QD15} established a regret bound in $\bigoh(\sqrt{\totaldelay})$ for online gradient descent and dual averaging, and suggested using the classical doubling trick to dynamically adjust the learning rate.\footnote{Due to a lack of consensus in the literature, \cite{QD15} used the name online mirror descent
to refer to dual averaging. See \cref{rem:mirror} for further discussion.
}
Assuming bounded delays, both \cite{MS14} and \cite{JGS16} devised delay-adaptive methods in order to obtain data-dependent bounds. The former centered on online gradient descent in the unconstrained case while the latter was based on online mirror descent and FTRL-prox.
Under the same setting, \cite{JGS19} also presented an adaptive method with a data-dependent bound which however has a worst-case dependence on the delay.
Recently, \cite{CZP20} extended the delayed feedback analysis to an online saddle-point algorithm which handled the constraints through Lagrangian relaxation.

Our work differs from the above in that we consider a multi-agent setup in which feedback does not arrive to the agents at the same time. To the best of our knowledge, this situation has never been considered before and gives rise to extra challenges that call for novel techniques.
In fact, even though both \cite{MS14} and \cite{JGS19} also dealt with \emph{asynchronous} online optimization, they focused on the coordinator-worker setting. It is thus possible there for the agents/workers to exploit a quantity stocked on the server (\eg an inexact global clock in \citealt{JGS19}). This is generally impossible in our setup.

As for the use of optimistic hints in online learning with delays,
the concurrent work by \cite{FOCM+21} appearing after the initial submission of our manuscript proposed a reduction of delayed online learning to optimistic online learning. Several practical algorithms developed based on this idea along with
empirical evidence of the benefit of optimism in online learning with delays were also presented.
Our \cref{sec:optimistic} complements this work by providing lower bounds, alternative analysis, and results for the multi-agent setup.

To complete the picture, we note that
the impact of delays has equally been studied in the literature on multi-armed bandits, both stochastic \citep{PBASG18,VCP17} and adversarial \citep{CGM18,LCG19,CGM19}.
Nonetheless, the settings of these papers are quite different from the online optimization problems we consider in our paper, so there is no overlap in results or techniques.

\para{Multi-agent online learning}
Multi-agent online learning encompasses a broad spectrum of problems, including distributed online optimization (discussed next),
multi-agent bandits
\citep{BM19,CGM19,SBHO+13,XTZV15},
and games \citep{CBL06,HMZ20}. 
In a recent paper, \cite{CCM20} considered a cooperative online learning problem in which a different set of agents is activated at each round, they encounter the same loss, and they receive immediately the relevant gradient feedback after playing.
While this setting is different from our own (there are no feedback delays and a fixed underlying communication graph is assumed), this is 
as far as we aware the first paper that considered asynchronous activation in multi-agent online convex optimization.
Very recently, the setup considered in \cite{CCM20} was further extended to cope with adversarial semi-bandits by \citet{DC21}.

\para{Distributed online optimization} 
In distributed online convex optimization, the agents cooperatively optimize a sequence of global costs which are defined as the sum of local loss functions, each associated with an agent. Under this setup, consensus-based distributed algorithms were proposed and shown to achieve sublinear regret \citep{HCM13,YSVQ12}. 
More recently, \cite{SJ17} and \cite{ZRZT19} further modified these algorithms to cope with dynamic regret, whereas the case of a time-varying network topology was examined in \cite{MC14} and \cite{AGL15}.

Nonetheless, all of the above works concern the \emph{synchronous} scenario, and this is true for both the activation of the agents (all the agents engage in each iteration) and the communication between the agents (which are performed without any delay).
In contrast, our framework allows for asynchronous \emph{activations} as well as asynchronous \emph{communication}.
To the best of our knowledge, such online optimization problems have only been considered in the concurrent work by \cite{JWJW21} that appeared after the initial submission of our paper, in which the authors introduced a push-sum strategy to solve the problem.

In addition to the above, the underlying communication topology is not modeled explicitly by our approach and it is thus possible to have agents that leave and join freely during the learning process.
For sake of concreteness, we further explain in \cref{subsec:global} how our method can be used to solve problems that are studied in
distributed online optimization.

\para{Asynchronous optimization}
For optimization problems that have a sum structure (\eg over different parts of some data set, or over several agents), a large part of the literature is based on a random sampling of one or several of the functions leading to a partial use of the data or of the links between agents. This stems from the study of randomized fixed point operators \citep{bianchi2015coordinate,combettes2015stochastic}, later extended to delayed settings \citep{MPPR+15,peng2016arock,leblond2017asaga}. This kind of randomized algorithms is incompatible with the setup considered in our paper in which the agents are \emph{activated}---not sampled.

In the case when a coordinator uses several workers to gather asynchronously gradient feedback, several variants of the proximal gradient algorithm were shown to be efficient, see \citet{aytekin2016analysis},
\citet{vanli2018global},
and \citet{mishchenko2020distributed}, the latter allowing for unbounded delays. However, the analyses of these methods are based on the study of the distance between the iterates and the minimizer of the problem which hinders their extension to the online setting.  

Finally, we are aware of very few works on open networks where agents can freely join and leave the system. These exceptions treat the simpler problem of averaging local values and focus on the system's stability \citep{HM17,franceschelli2020stability,de2020open}. These ideas were recently extended to study the stability of decentralized gradient descent in open networks \citep{hendrickx2020stability}, but, again, there is no overlap with our work.

\subsection{Online Algorithms are not Equally Robust to Delays}
\label{subsec:MD-DA}

In this paper, we have paid exclusive attention to variants of \acf{DA}.
Another family of algorithms that the agents may follow to minimize their regret is \ac{OMD} and its variants.
While these two types of methods achieve the same order of regret in many situations, they are not equally robust to delays in our setup, as explained below.

To define \ac{OMD}, we make the additional assumption that the subdifferential $\subd\hreg$ admits a continuous selection denoted by $\grad\hreg$.
The \emph{Bregman divergence} induced by $\hreg$ is then written as
\begin{equation}
    \notag
    \breg_{\hreg}(\point,\pointalt) = \hreg(\point) - \hreg(\pointalt) - \product{\grad\hreg(\pointalt)}{\point-\pointalt}.
\end{equation}
Subsequently, the update of \ac{OMD} is
\begin{equation}
\label{eq:mirror-desc}
\tag{OMD}
    \vt[\state] = \argmin_{\point\in\points}
    	\braces*{ \product{\vt[\gvec][\run-1] }{\point} + \frac{1}{\vt[\step][\run-1]} \breg_{\hreg}(\point,\vt[\state][\run-1]) }
    =
    \prox\left(\nabla\hreg(\vt[\state][\run-1]) - \vt[\step][\run-1]\vt[\gvec][\run-1]\right).
\end{equation}
The main difference between \eqref{eq:mirror-desc} and \eqref{eq:dual-avg} is that \eqref{eq:mirror-desc} generates a new point by combining the last gradient with the last prediction, while \eqref{eq:dual-avg} combines all past gradients and then generates a prediction, without explicitly using the last available prediction.

The two algorithms \eqref{eq:mirror-desc} and \eqref{eq:dual-avg}  are not equally robust to delays.
Indeed, if feedback from different rounds arrives out-of-order (due to the presence of delays), the natural extension of the methods would be to use them as if they corresponded to the last played point.
The sequence of points generated by the algorithms would then be different than with ordered feedback.
However, for \eqref{eq:dual-avg}, the final output after all feedback has arrived will be \emph{the same} for all agents, in contrast to that of \eqref{eq:mirror-desc}.
This is because, in dual averaging, all gradients enter the model with the \emph{same weight} \citep[Sec.~1.2]{Nes09};
this is a very appealing feature, especially when trying to incorporate delayed gradients or gradients generated by other agents.
As we show below in \cref{subsec:global}, this also helps guarantee that the points played by the agents do not deviate too much from each other when the delay is small.

\begin{remark}
\label{rem:mirror}
The origins of the above methods can be traced to \cite{NY83}, but there is otherwise no consensus on terminology in the literature.
The specific formulation \eqref{eq:mirror-desc} is sometimes referred to as ``eager'' mirror descent, in contrast to the method's ``lazy'' variant which outputs $\vt[\state] \gets \prox(-\sum_{\runalt < \run} \vt[\step][\runalt] \vt[\gvec][\runalt])$, see \eg \cite{Nes09} or \cite{MZ19}.
These variants coincide when $\hreg$ is infinitely ``steep'' at the boundary of $\points$, \ie $\dom\subd\hreg \cap \points = \relint\points$;
otherwise, they lead to different sequences of play \citep{KM17}.
The ``dual averaging'' variant is due to \cite{Nes09}, and differs from the lazy variant of \eqref{eq:mirror-desc} in that all gradients enter the algorithm with the same weight.
From an online learning viewpoint, \eqref{eq:dual-avg} can also be seen as a ``linearized'' version of the \ac{FTRL} class of algorithms \citep{SSS06}, and coincides with \ac{FTRL} when the loss functions encountered are linear.
For a survey, see \cite{juditsky2019unifying}, \cite{mcmahan2017survey}, \cite{Mer19}, and references therein.
\end{remark}

\subsection{Multi-Agent Online Learning for Minimization of Global Losses}
\label{subsec:global}

Throughout the paper, our analysis has focused on the agents' \emph{individual} losses ($\vt[\obj]$ being the loss of the active agent $\worker=\worker(\run)$), and thus \emph{leads to regret bounds that characterize how much the whole network actually pays.}
While these bounds have an interest, networks of agents may also want to monitor \emph{global} losses over the agents.
This is typically the case of distributed online optimization, where the agents cooperate to solve a time-varying global problem.

In this section, we demonstrate the flexibility of our framework by showing that the aforementioned algorithms and analyses can be easily extended to this setup.
This, on one hand, bridges the gap between our work and the broad corpus of literature on distributed online optimization, and, on the other hand, provides the occasion to directly address the case of open networks where agents can join and depart the optimization process freely.

\subsubsection{From Effective Regret to Collective Regret}

In distributed optimization, it is often assumed that multiple predictions are made in a same time slot.
Formally, we denote by $\vt[\nWorkers]$ the number of active agents at time $\run$ and identify these agents from $1$ to $\vt[\nWorkers]$ instead of identifying each agent independently. This notation clarifies the fact that the agents are anonymous with respect to the algorithm and each other.
The functions and the played points at time $\run$ are respectively denoted by $\vwt[\obj][1], \ldots, \vwt[\obj][\vt[\nWorkers]]$ and $\vwt[\state][1], \ldots, \vwt[\state][\vt[\nWorkers]]$.

By directly extending the regret defined by \eqref{eq:reg-def} to our current setup, we obtain the following:
\begin{equation}
    \tag{Effective Regret}
    \local[\reg]_{\nRuns}(\comp) = \sum_{\run=\start}^{\nRuns}\sum_{\worker=1}^{\vt[\nWorkers]} \vwt[\obj](\vwt[\state])
    - \sum_{\run=\start}^{\nRuns}\sum_{\worker=1}^{\vt[\nWorkers]} \vwt[\obj](\comp),
\end{equation}
where the superscript $\ell$ means that the regret sums over the \emph{local} costs of the learners. Each agent only pays for the function it serves and the ultimate goal for a single agent is to perform well on the functions that it encounters. 
As an example, on-device machine learning aims to equip users' personal devices with intelligent machine features such as conversational understanding and image recognition, for the purposes of providing a satisfying user experience to each individual \citep{SCZL+16,WHLN+20}.

In contrast, we can also define \emph{global} loss functions $\vt[\obj]=\sum_{\worker=1}^{\vt[\nWorkers]}\vwt[\obj]$ at every instant $\run$ and evaluate each active agents' action with respect to this function.
This leads to the following regret formulation:
\begin{equation}
    \tag{Collective Regret}
    \glob[\reg]_{\nRuns}(\comp) = \sum_{\run=\start}^{\nRuns}\sum_{\worker=1}^{\vt[\nWorkers]} \vwt[\obj](\vwt[\state][\workeraltG])
    - \sum_{\run=\start}^{\nRuns}\sum_{\worker=1}^{\vt[\nWorkers]} \vwt[\obj](\comp),
\end{equation}
where, instead of evaluating $\vwt[\obj]$ at the point $\vwt[\state]$ played by learner $\worker$, we now evaluate all the $\vwt[\obj]$ at a single point $\vwt[\state][\workeraltG]$ independently of the worker $\worker$.
The choice of the \emph{reference agent} can vary with time; it is however possible to fix its index to $\workeraltG$ in advance given that the attribution of the worker indices at each $\run$ is arbitrary.

When the number of agents are fixed, \emph{collective regret} reduces to the usual regret formulation employed in the distributed online optimization literature \citep{HCM13,SJ17,YSVQ12}.
This performance measure suits better the applications related to wireless sensor networks such as distributed estimation \citep{RN04} and data fusion \citep{NLF07,RCMP+15}.
In fact, sensor networks are mostly deployed for a common objective  shared by all the sensors.
To attain this objective, the sensor nodes may need to cooperate to track some unknown variable
or to collaborate to learn a global assessment of the situation.
The \acl{GPRg} then measures each agent's performance with respect to this \emph{collective} mission, hence the name thereof. 

Finally, our formulation also admits the additional flexibility of involving different sets and numbers of agents at each iteration.
This is of particularly interest for open multi-agent systems \citep{HM17} and elastic distributed learning \citep{NWMC+13}.
Building upon this observation, we further explored how our approach could be applied to optimization in open networks in the follow-up work \citep{HIMM21open}.

Now, provided that all the loss functions $\vwt[\obj]$ are $\gbound$-Lipschitz,  the relation between $\vt[\glob[\reg]][\nRuns]$ and $\vt[\local[\reg]][\nRuns]$ is quite direct as formulated in the following lemma.

\begin{restatable}{lemma}{RegLocalGlobal}
\label{lem:regret-local-global}
Assume that all the loss functions $\vwt[\obj]$ are $\gbound$-Lipschitz; then, 
\begin{equation}
\notag
    \vt[\glob[\reg]][\nRuns](\comp)
    \le \vt[\local[\reg]][\nRuns](\comp)
    +\sum_{\run=\start}^{\nRuns}\sum_{\worker=1}^{\vt[\nWorkers]}
    \gbound\norm{\vwt-\vwt[\state][\workeraltG]}.
\end{equation}
\end{restatable}

\subsubsection{Decentralized Delayed Dual Averaging}
\label{subsec:DDDA}

Thanks to \cref{lem:regret-local-global}, a bound on the effective regret can be directly translated into one on the collective regret as long as the distances between the agents' predictions for a same moment can be controlled.
To illustrate this idea, we adapt \ac{DDA} to the current setup and bound its induced collective regret for appropriately chosen learning rates.
Let us first slightly extend the previously introduced notations and concepts to the current framework: The set of available gradients at time $\run$ for a worker $\worker$, $\vwt[\set]$, now represents the set of the (learner, time) indices of the feedback available for playing $\vwt$ so that if $(\workeralt,\runalt)\in\vwt[\set]$ then necessarily $\runalt\in\oneto{\run-1}$. The maximum delay $\delaybound$ is to be understood with respect to the global time index $\run$.
That is, for every $\runalt\in\oneto{\run-\delaybound-1}$ and $\workeralt\in\oneto{\vt[\nWorkers][\runalt]}$ we must have $(\workeralt,\runalt)\in\vwt[\set]$.
We also introduce the (quadratic) mean number of active agents by $\rms{\nWorkers}=\sqrt{(1/\nRuns)\sum_{\run=\start}^\nRuns\vt[\nWorkers]^2}$.

With these notations, the update of \ac{D-DDA} writes at time $\run$ for an agent $\worker$ as
\begin{equation}
\label{eq:D-DDA}
    \tag{D-DDA}
    \vwt[\state] = \argmin_{\point\in\points} \sum_{(\workeralt,\runalt)\in\vwt[\set]}\product{\vwt[\gvec][\workeralt][\runalt]}{\point} + \frac{\hreg(\point)}{\vwt[\step]},
\end{equation}
where $\vwt[\gvec][\workeralt][\runalt] \in \partial \vwt[\obj][\workeralt][\runalt](\vwt[\state][\workeralt][\runalt])$.
In order to understand the mechanics of \acl{GPRg} in our setup, we restrict our self to the case of a fixed learning rate $\vwt[\step]\equiv\step$.\footnote{We bypass this limitation in \cref{apx:global} by providing an implementable variable learning rate strategy that provably achieves small collective regret.}
To bound the \acl{GPRg}, three elements come into play:
\begin{itemize}
    \item the \emph{effective} regret. For this part, we change the time indices to have exactly one point played at each time. We define  $\vt[\nSamples]=\sum_{\runalt=\start}^{\run}\vt[\nWorkers][\runalt]$ and $\nSamples=\vt[\nSamples][\nRuns]$; then, the index of worker\;$\worker$ at time  $\run$ is changed  to $\mapping(\worker,\run)=\vt[\nSamples][\run-1]+\worker$ (so that only one action is performed at that time). This maps our problem to the setting of \cref{thm:delay-regret} with $\vt[\step]\equiv\step$ and thus with $\Dpermu=\idp$ we get
\begin{align}
\label{eq:global-reg-fix-lr-local}
    \vt[\local[\reg]][\nRuns](\comp)
    &\le
    \frac{\hreg(\comp)}{\step}
    + \frac{1}{2}\sum_{\indsamp=\start}^{\nSamples}
    \step
    \left(\dnorm{\vt[\alt{\gvec}][\indsamp]}^2
        +2\dnorm{\vt[\alt{\gvec}][\indsamp]}
        \sum_{\runano\in\setexclude{\oneto{\indsamp-1}}{\vt[\alt{\set}][\indsamp]}}\dnorm{\vt[\alt{\gvec}][\runano]}\right)
\end{align}
where $\vt[\alt{\gvec}][\mapping(\worker,\run)]=\vwt[\gvec]$ and $\vt[\alt{\set}][\mapping(\worker,\run)]=\setdef{\mapping(\workeralt,\runalt)}{(\workeralt,\runalt)\in\vwt[\set]}$.
\item the maximal delay $\delaybound$. Bounding from above the number of unavailable gradients for a (learner, time) pair and translating this condition to bound $\card(\setexclude{\oneto{\indsamp-1}}{\vt[\alt{\set}][\indsamp]})$, we get
\begin{align}
    \label{eq:regret-network-local-text}
    \vt[\local[\reg]][\nRuns](\comp)
    \le
    \frac{\hreg(\comp)}{\step}
    + \step(\delaybound+1)\gbound^2\sum_{\run=\start}^{\nRuns}
    \vt[\nWorkers^2].
\end{align}
\item the non-expansiveness of the mirror map (\cref{lem:prox-nonexp}).
This part enables us to go from the effective regret to the \acl{GPRg} using \cref{lem:regret-local-global}. 
\end{itemize}

Putting together these points we manage to show the following bound on the \acl{GPRg}, the full proof being deferred to \cref{apx:global}.

\begin{restatable}{proposition}{GlobalRegretFixLR}
\label{prop:delay-regret-global}
Assume that the maximum delay is bounded by $\delaybound$ and that all the loss functions are $\gbound$-Lipschitz.
For any $\comp$ satisfying $\hreg(\comp)\le\radius^2$, running \acl{D-DDA} \eqref{eq:D-DDA} with constant stepsize
\[\vwt[\step]\equiv\step=\frac{\radius}{\gbound\rms{\nWorkers}\sqrt{(2\delaybound+1)\nRuns}}\]
guarantees the following upper bound on the \acl{GPRg}
\begin{equation}
\notag
    \vt[\glob[\reg]][\nRuns](\comp) \le 
    2\radius\gbound\rms{\nWorkers}\sqrt{(2\delaybound+1)\nRuns}
    =\bigoh(\rms{\nWorkers}\sqrt{\delaybound\nRuns}).
\end{equation}
\end{restatable}

As a sanity check, we can see that when there is no delay ($\delaybound=0$) and a fixed number of agents ($\vt[\nWorkers]\equiv\nWorkers$), the proposition ensures a regret in $\bigoh(\nWorkers\sqrt{\nRuns})$. This corresponds to the regret achieved by dual averaging on $\vt[\obj]=\sum_{\worker=1}^{\nWorkers}\vwt[\obj]$ which is $\nWorkers\gbound$-Lipschitz (\citealp[Section 5.2]{Hazan16}; \citealp{Xiao10}; see also \cref{apx:DA}).
Nonetheless, since the network of agents may be evolving, the average number of workers $\rms{\nWorkers}$ may often not be available in advance; neither is the time horizon $\nRuns$ nor the current time index $\run$.
Exploiting the ideas of \cref{sec:variable}, we provide in \cref{apx:global} an implementable learning rate scheme that achieves a regret in $\bigoh(\sqrt{\delaybound\nSamples\nWorkers_{\max}})$ where $\nWorkers_{\max}=\max_{\start\le\run\le\nRuns}\vt[\nWorkers]$.

\section{Concluding Remarks}
\label{sec:conclusion}

Our aim in this paper was to design adaptive and non-adaptive learning algorithms that can provably achieve low regret in the presence of delays and asynchronicities in both single- and multi-agent environments.
This was achieved by means of a general dual averaging framework for handling delays and deriving regret bounds under various learning rate policies including adaptive and data-dependent ones.
In addition, we paid special attention to the decentralized case (which includes open networks of agents collaborating to achieve a low \acl{GPRg}), and we showed how our analysis can be further improved through the use of optimistic policies in slowly-varying environments.

Our work provides the basis for a number of subsequent extensions of independent interest.
One particular direction concerns the case where the agents' gradient feedback is corrupted by noise, either exogenous (\eg stemming from environmental fluctuations) or endogenous (\eg from mini-batch sampling in the case of empirical risk objectives).
Equally important is the choice of target regret measure:
in addition to the agents' effective and collective regret, there is a fair number of network applications in which dynamic regret considerations could be equally relevant.
In this regard, it would be important to see if the proposed policies lead to low dynamic regret \textendash\ or how to modify them to achieve this more demanding benchmark.

Finally, if the agents only have access to their incurred losses at each stage, it is possible to reconstruct a \emph{biased} estimate of the corresponding subgradients using a stochastic approximation estimator---either \emph{single-point} \citep{FKM05} or \emph{two-point} \citep{ADX10}.
However, in addition to the bias introduced by this indirect sampling process, the variance of the single-point estimator also grows unbounded as the process unfolds;
moreover, in multi-agent settings, the agent performing an update must have access to both the loss incurred by another agent at a different (known) timestamp \emph{and} the actual sampling perturbation / direction employed by the agent that incurred said loss.
Phenomena such as these lead to significant difficulties---both technical and conceptual---in the analysis of adaptive algorithms, and require completely new techniques to handle.
We defer work on this fruitful research direction to the future. 

\acks{%
The authors would like to thank the action editor and the reviewers for their thoughtful comments that helped greatly in improving the organization and the presentation of the paper.
This research was partially supported by the COST Action CA16228 ``European Network for Game Theory'' (GAMENET),
and
the French National Research Agency (ANR) in the framework of
the ``Investissements d'avenir'' program (ANR-15-IDEX-02),
the LabEx PERSYVAL (ANR-11-LABX-0025-01),
MIAI@Grenoble Alpes (ANR-19-P3IA-0003),
and the grants ORACLESS (ANR-16-CE33-0004) and ALIAS (ANR-19-CE48-0018-01). }

\renewcommand{\proofname}{\normalfont\textbf{Proof}}
\appendix
\crefalias{section}{appendix}
\crefalias{subsection}{appendix}

\section{Undelayed Dual Averaging}
\label{apx:DA}
Our paper studies several variants of dual averaging in various delayed/distributed setups.
For sake of completeness, we include here an analysis of the vanilla dual averaging algorithm in the basic undelayed online learning setting.
For a thorough study of the algorithm the readers can refer to the textbook \citealp[Section 5]{Hazan16} and \citealp{Xiao10}.

Let us consider a sequence of first-order feedback $\seq{\gvec}{\start}{\nRuns}$. At time $\run$ dual averaging computes
\begin{equation}
    \tag{DA}
    \label{eq:dual-avg-apx}
    \vt[\state] = \argmin_{\point\in\points} \sum_{\runalt=\start}^{\run-1}\product{\vt[\gvec][\runalt]}{\point} + \frac{\hreg(\point)}{\vt[\step]}.
\end{equation}
We recall that the mirror map is defined as $\prox: \dvec \mapsto \argmin_{\point\in\points} ~ \product{-\dvec}{\point}+\hreg(\point)$. We can thus write $\vt[\state]=\prox(\vt[\dvec])$ where $\vt[\dvec]=-\vt[\step]\sum_{\runalt=\start}^{\run-1}\vt[\gvec][\runalt]$ may be viewed as the dual point of $\vt[\state]$.
We have the following standard result concerning the (linearized) regret achieved by the algorithm.

\begin{proposition}
\label{prop:dual-avg-regret}
Let online dual averaging \eqref{eq:dual-avg-apx} be run with non-increasing learning rates  $(\vt[\step])_{\run\in\oneto{\nRuns}}$. Then, the generated points $\seq{\state}{\start}{\nRuns}$ satisfy
\begin{equation}
\notag
    \linreg_{\nRuns}(\comp)
    :=
    \sum_{\run=1}^\nRuns
    \product{\vt[\gvec]}{\vt[\state]-\comp}
    \le
    \frac{\hreg(\comp)}{\vt[\step][\nRuns]}
    + \frac{1}{2}\sum_{\run=\start}^{\nRuns}\vt[\step][\run]\dnorm{\vt[\gvec][\run]}^2.
\end{equation}
\end{proposition}
\begin{proof}
Let us fix $\comp\in\points$ and define the associated estimate sequence
\begin{equation}
    \notag
    \vt[\estseq](\point) = \sum_{\runalt=1}^{\run-1} \product{\current[\gvec][\runalt]}{\point-\comp} + \frac{\hreg(\point)}{\vt[\step]}.
\end{equation}
We will show that
\begin{equation}
    \label{eq:estseq-diff}
    \displaystyle \current[\estseq](\current) \le
    \update[\estseq](\update)-\product{\vt[\gvec]}{\update[\state]-\comp}
    -\frac{1}{2\vt[\step]}\norm{\update-\current}^2.
\end{equation}
On one hand, by $\vt[\step][\run+1]\le\vt[\step]$ and the non-negativity of $\hreg$,
\begin{equation}
    \label{eq:estseq-diff-1}
    \update[\estseq](\update)
    =\current[\estseq](\update)
    + \product{\vt[\gvec]}{\update-\comp}
    + \left(\frac{1}{\update[\step]}-\frac{1}{\current[\step]}\right)\hreg(\update)
    \ge \current[\estseq](\update)
    + \product{\vt[\gvec]}{\update-\comp}.
\end{equation}
On the other hand, by the definition of $\vt[\state]$ we have $\vt[\dvec]\in\subd\hreg(\vt[\state])$. Thus
\begin{align}
    \notag
    \current[\estseq](\update)-\current[\estseq](\current)
    &= \sum_{\runalt=1}^{\run-1}\product{\vt[\gvec][\runalt]}{\update-\current}
    + \frac{\hreg(\update)}{\vt[\step]} - \frac{\hreg(\current)}{\vt[\step]}\\
    \label{eq:estseq-diff-2}
    &=
    -\frac{1}{\vt[\step]}\product{\vt[\dvec]}{\update-\current}
     + \frac{\hreg(\update)}{\vt[\step]} - \frac{\hreg(\current)}{\vt[\step]}
    \ge \frac{1}{2\vt[\step]}\norm{\update-\current}^2.
\end{align}
The inequality holds thanks to the $1$-strong convexity of $\hreg$.
Summing \eqref{eq:estseq-diff-1}, \eqref{eq:estseq-diff-2} and rearranging the terms we obtain \eqref{eq:estseq-diff}.

Next, let $\vt[\step][\nRuns+1]=\vt[\step][\nRuns]$ and define $\vt[\state][\nRuns+1]$ by \eqref{eq:dual-avg} (We can do this since $\vt[\state][\nRuns+1]$ is not used in the computation of $\linreg_{\nRuns}$). 
Leveraging on \eqref{eq:estseq-diff}, we bound the regret as follows:
\begin{align}
    \notag
    \linreg_{\nRuns}(\comp)
    & := \sum_{\run=1}^\nRuns
    \product{\vt[\gvec]}{\vt[\state]-\comp}\\
    \notag
    & =\sum_{\run=1}^\nRuns
    \left(\product{\vt[\gvec]}{\vt[\state]-\update}
    +\product{\vt[\gvec]}{\update-\comp}\right)\\
    \notag
    & \le\sum_{\run=1}^\nRuns
    \left(\frac{\vt[\step]}{2}\dnorm{\vt[\gvec]}^2
    +\frac{1}{2\vt[\step]}\norm{\update-\current}^2
    +\update[\estseq](\update)-\current[\estseq](\current)
    -\frac{1}{2\vt[\step]}\norm{\update-\current}^2\right)\\
    \notag
     & =
    \update[\estseq][\nRuns](\update[\state][\nRuns])
    -\vt[\estseq][\start](\vt[\state][\start])
    + \frac{1}{2}\sum_{\run=\start}^{\nRuns}\vt[\step][\run]\dnorm{\vt[\gvec][\run]}^2\\
    \label{eq:regret-dual-avg-proof}
    &\le
    \frac{\hreg(\comp)}{\vt[\step][\nRuns]}
    + \frac{1}{2}\sum_{\run=\start}^{\nRuns}\vt[\step][\run]\dnorm{\vt[\gvec][\run]}^2.
\end{align}
In the last inequality we used 
\[\update[\estseq][\nRuns](\update[\state][\nRuns])
=\min_{\point\in\points}\update[\estseq][\nRuns](\point)
\le\update[\estseq][\nRuns](\comp)
=\frac{\hreg(\comp)}{\vt[\step][\nRuns+1]}
=\frac{\hreg(\comp)}{\vt[\step][\nRuns]}\]
and $\vt[\estseq][\start](\vt[\state][\start])=\hreg(\vt[\state][\start])/\vt[\step][\start]\ge0$. \eqref{eq:regret-dual-avg-proof} is exactly what we want to prove, so this ends the proof.
\end{proof}

We next prove the non-expansiveness of the mirror map which are used multiple times in our analyses (for a reference, see \eg \citealp[Chapter E, Thm. 4.2.1]{HL01}, or \citealp[Cor. 3.5.11]{Zal02}).

\begin{lemma}
\label{lem:prox-nonexp}
The mirror map is non-expansive, \ie $\norm{\prox(\dvec)-\prox(\alt{\dvec})}\le\dnorm{\dvec-\alt{\dvec}}$ for all $\dvec,\alt{\dvec}\in\vecspace$.\footnote{Precisely, $\prox$ is non-expansive because we are assuming that the strong convexity constant of $\hreg$ is $1$. Otherwise it would just be Lipschitz continuous, and clearly this would only influence our results by a constant factor (that depends on the strong convexity constant of $\hreg$).}
\end{lemma}
\begin{proof}
Let $\point=\prox(\dvec)$ and $\alt{\point}=\prox(\alt{\dvec})$. By definition of the mirror map,
\begin{equation}
    \notag
    \point = \argmin_{\test\in\points}\product{-\dvec}{\test}+\hreg(\test), ~~~~
    \alt{\point} = \argmin_{\test\in\points}\product{-\alt{\dvec}}{\test}+\hreg(\test).
\end{equation}
The optimality condition implies that $\dvec\in\subd\hreg(\point)$ and $\alt{\dvec}\in\subd\hreg(\alt{\point})$.
Hence, with the Cauchy–Schwarz inequality and the $1$-strong convexity of $\hreg$ with respect to $\norm{\cdot}$, we have
\begin{equation}
\notag
    \dnorm{\dvec-\alt{\dvec}}\norm{\alt{\point}-\point}
    \ge
    \product{\alt{\dvec}-\dvec}{\alt{\point}-\point}
    \ge \norm{\point-\alt{\point}}^2.
\end{equation}
It follows immediately $\dnorm{\dvec-\alt{\dvec}}\ge\norm{\point-\alt{\point}}.$
\end{proof}

\section{Missing Proofs for Variable Learning Rate Methods}
\label{apx:adaptive}
In this part, we complete the proofs of the results presented in \cref{sec:variable}.
To begin, the well-known ``inverse-root-sum'' lemma (see \eg \citealp[Lemma 3.5]{ACG02}) is essential for proving the regret guarantees of these methods.

\adaptive*
\begin{proof}
The function $y\in\R^+\mapsto\sqrt{y}$ being concave and has derivative $y\mapsto1/(2\sqrt{y})$, it holds for every $z\ge0$,
\begin{equation}
    \notag
    \sqrt{z}\le\sqrt{y}+\frac{1}{2\sqrt{y}}(z-y).
\end{equation}
Take $y=\sum_{\runalt=1}^{\run}\vt[\scalar][\runalt]$ and $z=\sum_{\runalt=1}^{\run-1}\vt[\scalar][\runalt]$ gives
\begin{equation}
\notag
    2\sqrt{\sum_{\runalt=1}^{\run-1}\vt[\scalar][\runalt]}+\frac{\vt[\scalar][\run]}{\sqrt{\sum_{\runalt=1}^{\run}\vt[\scalar][\runalt]}}\le2\sqrt{\sum_{\runalt=1}^{\run}\vt[\scalar][\runalt]}.
\end{equation}
We conclude by summing the inequality from $\run=2$ to $\run=\nRuns$ and using $\sqrt{\vt[\scalar][1]}\le2\sqrt{\vt[\scalar][1]}$.
\end{proof}

Recall that $\vwt[\setbck] =
\setdef{\{\runalt, \runano\}}
{\runalt\in\vt[\set],\runano\in
\setexclude{\vwt[\setrec][\worker][\runalt]}{\vt[\set][\runalt]}}$
and $\vt[\setbck]=\vwt[\setbck][\activeworker{\run}]$.
The next two propositions were used in the proof of \cref{thm:adadelayi}.

\SetbckCharac*
\begin{proof}
We will prove
    \begin{equation}
    \notag
    \vt[\setbck]
	= \setdef*{\{\runalt,\runano\} \subseteq \vt[\set]}{\runalt\notin\vt[\set][\runano], \runano\notin\vt[\set][\runalt]}
\end{equation}
by a two-way inclusion argument.
\para{Inclusion (``\,$\subseteq$\,'')}
Let $\runalt\in\vt[\set]$ and $\runano\in\setexclude{\vwt[\setrec][\activeworker{\run}][\runalt]}{\vt[\set][\runalt]}$.
The inclusion $\runano\in\vwt[\setrec][\activeworker{\run}][\runalt]$ means that $\vt[\gvec][\runano]$ arrives earlier than $\vt[\gvec][\runalt]$ on node $\activeworker{\run}$. As all the available gradients are used when playing $\vt[\state]$ and $\runalt\in\vt[\set]$, we deduce $\runano\in\vt[\set]$.
On the other hand, $\runano\in\vwt[\setrec][\activeworker{\run}][\runalt]$ also implies $\runalt\notin\vwt[\setrec][\activeworker{\run}][\runano]$.
Using \cref{asm:in-order} we know that $\vt[\set][\runano]\subseteq\vwt[\setrec][\activeworker{\run}][\runano]$, and consequently $\runalt\notin\vt[\set][\runano]$.

\para{Containment (``\,$\supseteq$\,'')}
Let $\{\runalt, \runano\}\subseteq\vt[\set]$ such that $\runalt\notin\vt[\set][\runano]$ and $\runano\notin\vt[\set][\runalt]$. Since either $\runano\in\vwt[\setrec][\activeworker{\run}][\runalt]$
or $\runalt\in\vwt[\setrec][\activeworker{\run}][\runano]$ (but not both) we conclude immediately $\{\runalt, \runano\}\in\vwt[\setbck][\activeworker{\run}]=\vt[\setbck]$. 
\end{proof}

\FaithfulDelayBound*
\begin{proof}
(\emph{a})
Let $\runalt,\run\in\oneto{\nRuns}$ such that $\runalt\le\run$. We need to prove $\Dpermu(\runalt)\le\Dpermu(\run)+\delaybound$.
Assume the opposite, that is, $\Dpermu(\runalt)>\Dpermu(\run)+\delaybound$.
Then, from the bounded delay assumption, $\Dpermu(\run)\in\vt[\set][\Dpermu(\runalt)]$.
$\Dpermu$ being a faithful permutation, this implies $\run=\Dpermu^{-1}(\Dpermu(\run))<\Dpermu^{-1}(\Dpermu(\runalt))=\runalt$, a contradiction.
Finally, $\oneto{\run}^{\Dpermu}=\{\Dpermu(1),\ldots,\Dpermu(\run)\}=\setdef{\Dpermu(\runalt)}{\runalt\le\run}$ and hence $\oneto{\run}^{\Dpermu}\subseteq\oneto{\Dpermu(\run)+\delaybound}$.

(\emph{b}) This is immediate from (\emph{a}) and the inclusion $\oneto{\Dpermu(\run)-\delaybound-1}\subseteq\vt[\set][\Dpermu(\run)]$ which holds since the maximum delay is assumed to be bounded by $\delaybound$.

(\emph{c}) Fix $\run\in\oneto{\nRuns}$. For all $\runalt\le\run$, we have $\Dpermu(\runalt)\le\Dpermu(\run)+\delaybound$ and therefore $\max_{\runalt\le\run}\Dpermu(\runalt)\le\Dpermu(\run)+\delaybound$.
$\Dpermu$ being a permutation of $\oneto{\nRuns}$, it holds $\max_{\runalt\le\run}\Dpermu(\runalt)\ge\run$ and subsequently $\run\le\Dpermu(\run)+\delaybound$.
Similarly, we also have $\Dpermu(\run)-\delaybound\le\min_{\run\le\runalt}\Dpermu(\runalt)$ and $\min_{\run\le\runalt}\Dpermu(\runalt)\le\run$.
This implies $\Dpermu(\run)-\delaybound\le\run$. Combining the two we conclude $|\Dpermu(\run)-\run|\le\delaybound$.
\end{proof}

We close this section with the single-agent adaptive algorithm \eqref{eq:adadelayOplus}.
\PropAdaDelayOPlus*
\begin{proof}
Let $\vt[\gsumupper]=\radius^2/\vt[\step^2]$ so that $\vt[\step]=\radius/\sqrt{\vt[\gsumupper]}$.
It holds that $\vt[\gsumupper]\ge\vt[\Gsumbck]+\vt[\delayres]\gbound^2\ge\vt[\gsum]$.
The first inequality comes from the definition of $\vt[\step]$ and the second inequality was shown in \cref{subsec:adadelayplus}.
Applying \cref{thm:delay-regret} with $\Dpermu=\idp$ and \cref{lem:adaptive} yields
\begin{equation}
\notag
\begin{aligned}
    \reg_{\nRuns}(\comp)
    &\le
    \frac{\hreg(\comp)}{\vt[\step][\nRuns]}
    + \frac{1}{2}\sum_{\run=\start}^{\nRuns}
    \vt[\step][\run]
    \Bigg(\dnorm{\vt[\gvec][\run]}^2
        +2\dnorm{\vt[\gvec][\run]}
        \sum_{\runalt\in\vt[\setout]}\dnorm{\vt[\gvec][\runalt]}\Bigg)\\
    &\le\radius\sqrt{\vt[\gsumupper][\nRuns]}
    + \frac{\radius}{2}\sum_{\run=\start}^{\nRuns}
    \frac{1}{\sqrt{\vt[\gsum]}}
    \Bigg(\dnorm{\vt[\gvec][\run]}^2
        +2\dnorm{\vt[\gvec][\run]}
        \sum_{\runalt\in\vt[\setout]}\dnorm{\vt[\gvec][\runalt]}\Bigg)\\
    &\le\radius\sqrt{\vt[\gsumupper][\nRuns]}+\radius\sqrt{\vt[\gsum][\nRuns]}
    \le 2\radius\sqrt{\vt[\gsumupper][\nRuns]}.
\end{aligned}
\end{equation}
Since $\vt[\gsumupper][\nRuns] = \max_{1\le\run\le\nRuns}\vt[\Gsumbck]+\vt[\delayres]\gbound^2$, we have already proved the first inequality.
For the second inequality, we use both $\vt[\Gsumbck]\le\vt[\gsum]$ and $\vt[\Gsumbck]\le(\card(\vt[\set])+2\card(\vt[\setbck]))\gbound^2$ (\cf \eqref{eq:gsumbcki-setbck}).
\end{proof}

When the delays are bounded by a constant, it is possible to further bound\;$\delayres$ from above, as shown below.

\begin{proposition}
Assume that the maximum delay is bounded by $\delaybound$.
Then $\delayres\le2\delaybound^2+3\delaybound+1$.
\end{proposition}
\begin{proof}
To begin, we have $\run-\card(\vt[\set]) \le \delaybound+1$ as $\oneto{\run-\delaybound-1}\subseteq\vt[\set]$.
Next, let us consider a pair $\{\runalt, \runano\}\in\setexclude{\vt[\setdel]}{\vt[\setbck]}$.
From \cref{prop:setbck-charac} we know that $\{\runalt, \runano\}\not\subseteq\vt[\set]$, so we have either $\runalt\in\intinterval{\run-\delaybound}{\run}$ or $\runano\in\intinterval{\run-\delaybound}{\run}$.
Without loss of generality, we suppose $\runalt<\runano$, then $\runano\in\intinterval{\run-\delaybound}{\run}$.
By \cref{prop:setdel-charac} we have $\runalt\notin\vt[\set][\runano]$, and thus $\runalt\in\intinterval{\runano-\delaybound}{\runano-1}$.
This shows $\card(\setexclude{\vt[\setdel]}{\vt[\setbck]})\le\delaybound(\delaybound+1)$. We can therefore conclude $\vt[\delayres]\le2\delaybound(\delaybound+1)+\delaybound+1=2\delaybound^2+3\delaybound+1$.
\end{proof}

Therefore, the bound of \cref{prop:adadelayoplus}
potentially improves upon the bounds obtained in \cite{MS14} and \cite{JGS16}. 
\section{Proofs Related to the Optimistic Variant}
\subsection{Delayed Optimistic Dual Averaging}
\label{apx:optimistic-delay-regret}

\OptDelayRegret*

\begin{proof}
Let us consider the virtual iterates
\begin{equation}
    \notag
    \vt[\virtual] = \vt[\state][\start] - \current[\step]\sum_{\runalt=\start}^{\run-1}\inter[\gvec][\runalt].
\end{equation}
We define the estimate sequence
\begin{equation}
    \notag
    \vt[\estseq](\point) = \sum_{\runalt=1}^{\run-1} \product{\inter[\gvec][\runalt]}{\point-\comp} + \frac{\norm{\point-\vt[\point][\start]}^2}{2\vt[\step]}.
\end{equation}
Notice that the regret is measured with the leading states
\begin{align}
\label{eq:opt-regret-decomp}
    \vt[\obj](\inter) - \vt[\obj](\comp)
    \le \product{\inter[\gvec]}{\inter-\comp}
    = \product{\inter[\gvec]}{\inter-\vt[\virtual][\run+1]}
    + \product{\inter[\gvec]}{\vt[\virtual][\run+1]-\comp}
\end{align}
As shown in the proof of \cref{prop:dual-avg-regret}, we have
\begin{equation}
\label{eq:opt-regret-rec}
    \product{\inter[\gvec]}{\vt[\virtual][\run+1]-\comp}
    \le \update[\estseq](\update[\virtual]) - \current[\estseq](\current[\virtual]) - \frac{1}{2\vt[\step]}\norm{\update[\virtual]-\current[\virtual]}^2.
\end{equation}
For the other term, we recall the definition $\vt[\setout]=\setexclude{\oneto{\run-1}}{\vt[\set]}$ and define $\vt[\out]=\card(\vt[\setout])$. Then,
\begin{align}
    \notag
    \product{\inter[\gvec]}{\inter-\vt[\virtual][\run+1]}
    &=
    \product{\inter[\gvec]}{\inter-\current}
    + \product{\inter[\gvec]}{\current-\current[\virtual]}
    + \product{\inter[\gvec]}{\current[\virtual]-\update[\virtual]}\\
    \notag
    &=
    \product{\inter[\gvec]}{-\current[\stepalt]\inter[\appr]}
    + \product{\inter[\gvec]}{\current[\step]\sum_{\runalt\in\vt[\setout]}\inter[\gvec][\runalt]}
    + \product{\inter[\gvec]}{\current[\virtual]-\update[\virtual]}\\
    \notag
    &=
    \frac{\vt[\stepalt]}{2}\left(\norm{\inter[\gvec]-\inter[\appr[\gvec]]}^2-\norm{\inter[\gvec]}^2-\norm{\inter[\appr]}^2\right)\\
    \notag
    &~~+
    \current[\step]\sum_{\runalt\in\vt[\setout]}\product{\inter[\gvec]}{\inter[\gvec][\runalt]}
    + \product{\inter[\gvec]}{\current[\virtual]-\update[\virtual]}\\
    \notag
    &\le
    \frac{\vt[\stepalt]}{2}\left(\norm{\inter[\gvec]-\inter[\appr[\gvec]]}^2-\norm{\inter[\gvec]}^2-\norm{\inter[\appr]}^2\right)\\
    \label{eq:opt-regret-decomp-2}
    &~~+
    \frac{\vt[\step]}{2}\norm{\inter[\gvec]}^2
    +\frac{1}{2\vt[\step]}\norm{\current[\virtual]-\update[\virtual]}^2
    +\frac{\vt[\out]\vt[\step]}{2}\norm{\inter[\gvec]}^2
    +\frac{\vt[\step]}{2}\sum_{\runalt\in\vt[\setout]}\norm{\inter[\gvec][\runalt]}^2.
\end{align}
Combining \eqref{eq:opt-regret-decomp}, \eqref{eq:opt-regret-rec}, \eqref{eq:opt-regret-decomp-2} and summing from $\run=\start$ to $\nRuns$ yields
\begin{align}
    \notag
    \reg_{\nRuns}(\comp)
    &\le
    \update[\estseq][\nRuns](\update[\virtual][\nRuns])
    -\vt[\estseq][\start](\vt[\virtual][\start])
    +\sum_{\run=\start}^{\nRuns}
    \frac{\vt[\stepalt]}{2}\left(\norm{\inter[\gvec]-\inter[\appr[\gvec]]}^2-\norm{\inter[\appr]}^2\right)\\
    \label{eq:opt-regret-sum}
    &~~+
    \left(-\frac{\vt[\stepalt]}{2}+\frac{(\vt[\out]+1)\vt[\step]}{2}
    +\sum_{\run\in\vt[\setout][\runano]}\frac{\vt[\step][\runano]}{2}\right)\norm{\inter[\gvec]}^2.
\end{align}
Since the maximum delay is $\delaybound$, we have $\vt[\out]\le\outbound\le\delaybound$ and if $\run\in\vt[\setout][\runano]$ it holds $\runano>\run\ge\runano-\delaybound$ giving that $\card(\setdef{\runano}{\run\in\vt[\setout][\runano]})\le\delaybound$. Moreover, as $\seqinf{\step}{\run}$ is a decreasing sequence, $\run\in\vt[\setout][\runano]$ also implies $\vt[\step][\runano]\le\vt[\step]$. The last term of \eqref{eq:opt-regret-sum} can thus be bounded as following
\begin{equation}
    \left(-\frac{\vt[\stepalt]}{2}+\frac{(\vt[\out]+1)\vt[\step]}{2}
    +\sum_{\run\in\vt[\setout][\runano]}\frac{\vt[\step][\runano]}{2}\right)\norm{\inter[\gvec]}^2
    \le \frac{1}{2}((2\delaybound+1)\vt[\step]-\vt[\stepalt])\norm{\inter[\gvec]}^2\le0,
\end{equation}
where the second inequality leverages the condition $\vt[\stepalt]\ge(2\delaybound+1)\vt[\step]$.

To conclude, we use $\update[\estseq][\nRuns](\update[\virtual][\nRuns])\le\update[\estseq][\nRuns](\comp)$ and observe that $\vt[\estseq][\start](\vt[\virtual][\start])=\vt[\estseq][\start](\vt[\state][\start])=0$ by definition, so that
\begin{align}
    \notag
    \reg_{\nRuns}(\comp)
    &\le
    \frac{\norm{\comp-\vt[\state][\start]}^2}{2\update[\step][\nRuns]}
    +\sum_{\run=\start}^{\nRuns}
    \frac{\vt[\stepalt]}{2}\left(\norm{\inter[\gvec]-\inter[\appr[\gvec]]}^2-\norm{\inter[\appr]}^2\right).
\end{align}
Let $\update[\step]=\vt[\step]$ and we get the desired bound.
\end{proof}

\subsection{The Necessity of Scale Separation}
\label{apx:optimistic-regret-lower-bound}

\OptRegretLB*

\begin{proof}
Assume, for the sake of contradiction, that there exists $\step=\step(\radius,\nRuns,\delaybound,\vt[\variation^{\delaybound^+}][\nRuns])$ and a corresponding $\stepalt$ with $\stepalt\le\delaybound\step$ such that \eqref{eq:D-ODA} with $\inter[\appr]=\vt[\gvec][\run-\delaybound-1]$ guarantees a regret in $\smalloh(\max(\vt[\variation^{\delaybound^+}][\nRuns], \sqrt{\nRuns}))$.
Formally, we define a round of the algorithm as a composition a loss sequence, a delay mechanism, a initial point $\vt[\state][\start]$ and a competing vector $\comp$,
and denote by $\mathcal{R}(\radius,\nRuns,\delaybound,\vt[\variation^{\delaybound^+}][\nRuns])$ the set of all the rounds with time horizon $\nRuns$, $(\delaybound+1)$-variation $\vt[\variation^{\delaybound^+}][\nRuns]$, uniform delay $\delaybound$ and $\norm{\comp-\vt[\state][\start]}\le\radius$.
Then, fixing $\radius$ and $\delaybound$, for every $\eps>0$, we can find $N>0$ such that if $\max(\vt[\variation^{\delaybound^+}][\nRuns], \sqrt{\nRuns})\ge N$, the regret achieved by the algorithm for every instance in $\mathcal{R}(\radius,\nRuns,\delaybound,\vt[\variation^{\delaybound^+}][\nRuns])$ is smaller than $\eps\max(\vt[\variation^{\delaybound^+}][\nRuns], \sqrt{\nRuns})$.
The proof then consists in finding two instances of $\mathcal{R}(\radius,\nRuns,\delaybound,\vt[\variation^{\delaybound^+}][\nRuns])$ such that the regret achieved by the algorithm on these two instances can not be simultaneously smaller than $\eps\max(\vt[\variation^{\delaybound^+}][\nRuns], \sqrt{\nRuns})$.

For this, we fix the delay $\delaybound$, set $\radius=1$ without loss of generality
and explicit these two instances in the following ($\points=\R$):

\vskip 0.25em
1. Let $\nPeriods, \len > \delaybound$ be two positive integers.
We first consider a loss sequence of length $2\nPeriods\len+\delaybound+1$ (\ie $\nRuns=2\nPeriods\len+\delaybound+1$) as illustrated below:
\begin{equation}
\notag
\underbrace{\underbrace{-1\thinspace \ldots\thinspace -1}_{\len}\enspace
\underbrace{+1\thinspace \ldots\thinspace +1}_{\len}\quad
\ldots\quad
\underbrace{-1\thinspace \ldots\thinspace -1}_{\len}\enspace
\underbrace{+1\thinspace \ldots\thinspace +1}_{\len}}_{2\nPeriods\len ~\text{losses}}\enspace
\underbrace{-1\thinspace \ldots\thinspace -1}_{\delaybound+1}
\end{equation}
A period is defined as a subsequence of $2\len$ losses with $\ell$ consecutive $-1$s followed by $\ell$ consecutive $+1$s.
The whole loss sequence is then composed of $2\nPeriods$ periods followed by $\delaybound+1$ consecutive $-1$s.
We would like to compute the regret achieved by \eqref{eq:D-ODA} with $\step, \stepalt, \inter[\appr]$ as specified in the statement and $\vt[\state][\start]=\comp=0$.

For the first $\delaybound+1$ steps, the algorithm stays at $\vt[\state][\start]=\comp$ so the accumulative regret is $0$.
For the remaining of the round, the algorithm goes through the same trajectory for each period of delayed feedback vectors it receives and this happens $\nPeriods$ times.
To compute the regret, we just need to match the position of the iterate with the actual loss at each moment, which is done in \cref{fig:optimistic-lower-bound} (as the loss vectors of a single period sum to $0$, after receiving all the vectors from one period it is as if we started again from $\vt[\state][\start]=\comp=0$).
Notice that the algorithm uses the most recent vector it receives for extrapolation.

\begin{figure}[t]
\centering
\newcommand\tikzmark[1]{
  \tikz[remember picture,overlay,inner xsep=0pt]\node (#1) {};}
\begin{gather}
\notag
\arraycolsep=4pt\def\arraystretch{1.2}
{\footnotesize{
\begin{array}{c|c@{\hskip0pt}ccc@{\hskip12pt}ccc@{\hskip0pt}c c@{\hskip0pt}ccc@{\hskip10pt}ccc@{\hskip0pt}c}
\toprule
\boldsymbol{\run}&&\delaybound+2 & \dots & \len & \len+1 & \cdots & \len+\delaybound+1
&&& \len+\delaybound+2 & \ldots & 2\len & 2\len+1 & \dots
& 2\len+\delaybound+1&
\\
\boldsymbol{\vt}& \tikzmark{start} & \step &  \dots & (\len-\delaybound-1)\step & (\len-\delaybound)\step & \cdots & \len\step & \tikzmark{end}
\tikz[remember picture,overlay]{
  \draw[decorate,decoration={brace,mirror,raise=4pt}] (start) --node[below=5pt] {$\scriptstyle+\gamma$} (end);
}&
\tikzmark{start} & (\len-1)\step & \ldots & (\delaybound+1)\step
& \delaybound\step & \ldots  & 0 & \tikzmark{end}
\tikz[remember picture,overlay]{
  \draw[decorate,decoration={brace,mirror,raise=4pt}] (start) --node[below=5pt] {$\scriptstyle-\gamma$} (end);}
\\[0.6em]
\boldsymbol{\vt[\gvec]} && \multicolumn{3}{c|}{-1} & \multicolumn{8}{c}{+1} & \multicolumn{3}{|c}{-1}& \\
\bottomrule
\end{array}}}
\end{gather}
\vspace{-1em}
\caption{Illustration of the evolution of the optimistic algorithm for a period of feedback in the first example of the proof of \cref{thm:optimistic-regret-lower-bound}. The time is taken modulo $2\len$.}
\label{fig:optimistic-lower-bound}
\end{figure}

The regret for each period of feedback is thus
\begin{align}
\notag
    \reg_{per} &= 
    \frac{-(\len-\delaybound-1)(\len-\delaybound)\step}{2}
    -(\len-\delaybound-1)\stepalt
    +\frac{(\delaybound+1)(2\len-\delaybound)\step}{2}
    +(\delaybound+1)\stepalt\\
\notag
    &~~+\frac{(\len-\delaybound-1)(\len+\delaybound)\step}{2}
    -(\len-\delaybound-1)\stepalt
    -\frac{(\delaybound+1)\delaybound\step}{2}
    +(\delaybound+1)\stepalt\\
\notag
    &= (\delaybound+1)(\len-\delaybound)\step +
    (\len-\delaybound-1)\delaybound\step
    +2(2\delaybound-\len+2)\stepalt\\
\notag
    &= (\step+2\delaybound\step-2\stepalt)\len
    - 2\delaybound(\delaybound+1)\step
    + (4\delaybound+4)\stepalt.
\end{align}
Accordingly, the total regret is
\begin{equation}
    \notag
    \reg_1 = \nPeriods((\step+2\delaybound\step-2\stepalt)\len
    - 2\delaybound(\delaybound+1)\step
    + (4\delaybound+4)\stepalt)
    \ge \nPeriods (\len- 2\delaybound(\delaybound+1))\step,
\end{equation}
where for the inequality we use the fact that $\stepalt\le\delaybound\step$.

Moreover, for every $m\in\N_{0}$, from time $2m\len+\delaybound+2$ to $2m\len+2\len+\delaybound+1$ the $(\delaybound+1)$-variation increases by $8(\delaybound+1)$: there are $\delaybound+1$ switches both from\;$-1$ to\;$+1$ and from\;$+1$ to\;$-1$ with each switch contributing $4$ to the variation.
Remember also that in the definition of the $\vt[\variation^{\delaybound^+}][\nRuns]$ we compare the first $\delaybound+1$ losses with $0$. For the whole sequence we therefore count $\vt[\variation^{\delaybound^+}][\nRuns]=(8\nPeriods+1)(\delaybound+1)$.

\vskip 0.25em
2. We now construct another example with the same $\nRuns, \vt[\variation^{\delaybound^+}][\nRuns]$ as follows (with $\len>4\delaybound+4$):
\begin{equation}
\notag
\underbrace{\underbrace{0\thinspace \ldots\thinspace 0}_{\delaybound+1}\enspace
\underbrace{1\thinspace \ldots\thinspace 1}_{\delaybound+1}\quad
\ldots\quad
\underbrace{0\thinspace \ldots\thinspace 0}_{\delaybound+1}\enspace
\underbrace{1\thinspace \ldots\thinspace 1}_{\delaybound+1}}_{8\nPeriods(\delaybound+1) ~\text{losses}}\enspace
\underbrace{0\thinspace \ldots\thinspace 0}_{2\nPeriods\len-8\nPeriods(\delaybound+1)}\enspace
\underbrace{1\thinspace \ldots\thinspace 1}_{\delaybound+1}
\end{equation}
In particular, $2\nPeriods\len-8\nPeriods(\delaybound+1)>2\nPeriods>\delaybound+1$. It follows immediately $\vt[\variation^{\delaybound^+}][\nRuns]=(8\nPeriods+1)(\delaybound+1)$ and of course $\nRuns=2\nPeriods\len+\delaybound+1$.

Let $\vt[\state][\start]=0$ and $\comp=-1$.
In the sequence the loss $1$ appears $(4\nPeriods+1)(\delaybound+1)$ times while the remaining feedback are all $0$s.
Given that the last $\delaybound+1$ losses are never received by the algorithm,
we have indeed always $\vt\ge-4\nPeriods(\delaybound+1)\step-\stepalt$.
The regret can therefore be lower bounded as:
\begin{align}
    \notag
    \reg_2
    &= \sum_{\run=\start}^\nRuns\vt[\gvec](\vt+1)\\
    \notag
    &= \sum_{\run=\start}^\nRuns\vt[\gvec]\vt + (4\nPeriods+1)(\delaybound+1)\\
    \notag
    &\ge (4\nPeriods+1)(\delaybound+1) - 4\nPeriods(4\nPeriods+1)(\delaybound+1)^2\step
    - (4\nPeriods+1)(\delaybound+1)\stepalt\\
    \notag
    &\ge (4\nPeriods+1)(\delaybound+1) - (4\nPeriods+1)^2(\delaybound+1)^2\step,
\end{align}
where in the last inequality we use again $\stepalt\le\delaybound\step$.

\vskip 0.25em
\textbf{Conclude.}
We choose $\nPeriods,\len$ so that $\len=(16\nPeriods+9)(\delaybound+1)^2+2\delaybound(\delaybound+1) > 4\delaybound+4$.
Notice that $\nRuns$ and $\vt[\variation^{\delaybound^+}][\nRuns]$ can be made arbitrarily large.
We run the algorithm in question on the two problem instances described above. We have on one side
\begin{equation}
    \notag
    \reg_1
    \ge \nPeriods (\len- 2\delaybound(\delaybound+1))\step
    = (16\nPeriods^2 + 9\nPeriods)(\delaybound+1)^2\step.
\end{equation}
On the other side,
\begin{align}
    \notag
    \reg_2
    &\ge (4\nPeriods+1)(\delaybound+1) - (4\nPeriods+1)^2(\delaybound+1)^2\step\\
    \notag
    &\ge (4\nPeriods+1)(\delaybound+1) - (16\nPeriods^2 + 9\nPeriods)(\delaybound+1)^2\step.
\end{align}
Recalling that $\vt[\variation^{\delaybound^+}][\nRuns]=(8\nPeriods+1)(\delaybound+1)$, the above shows
\begin{equation}
    \notag
    \reg_1 + \reg_2 \ge (4\nPeriods+1)(\delaybound+1) \ge \vt[\variation^{\delaybound^+}][\nRuns]/2.
\end{equation}
Similarly, we have
$\nRuns=2\nPeriods\len+\delaybound+1\le(32\nPeriods^2+22\nPeriods)(\delaybound+1)^2$. As a consequence
\begin{equation}
    \notag
    \reg_1 + \reg_2 \ge (4\nPeriods+1)(\delaybound+1) \ge \sqrt{\nRuns}/2.
\end{equation}
To summarize, we have proven for some $\nRuns$ and $\vt[\variation^{\delaybound^+}][\nRuns]$ arbitrarily large, we can find two instances from $\mathcal{R}(\radius,\nRuns,\delaybound,\vt[\variation^{\delaybound^+}][\nRuns])$ so that the regrets achieved by the algorithm on these two instances satisfy
\begin{equation}
    \notag
    \max(\reg_1, \reg_2) \ge \max(\vt[\variation^{\delaybound^+}][\nRuns], \sqrt{\nRuns})/2.
\end{equation}
This is in contradiction with the initial hypothesis by choosing $\eps=1/2$.
\end{proof}

\subsection{A Lower Bound for Delayed Online Learning}
\label{apx:lowerboundopt}

\VarRegretLB*

\begin{proof}
Let $\len=\overline{\variation^{\delaybound}}/(4(\delaybound+1))$ be a positive integer and $\nRuns = (\tau+1)\len$.
We consider $\algo$ an arbitrary online algorithm compatible with delayed feedback.
From $\algo$ we define $\algo_{/\delaybound}$ another online algorithm as follows: For any sequence of losses with undelayed feedback, we repeat each loss $\delaybound+1$ times and only send the feedback after a delay of $\delaybound$.
In other words, for the loss sequence $\gvec_1, \gvec_2, \ldots$, at the end of iteration $\indg(\delaybound+1)$ to $\indg(\delaybound+1)+\delaybound$ we receive feedback $\gvec_{\indg-1}$ (with the convention $\gvec_0=0$) while we suffer a loss $\product{\gvec_\indg}{\vt}$ from iteration $p_\indg=(\indg-1)(\delaybound+1)+1$ to $\indg(\delaybound+1)$.
We then play $\algo$ on this new loss sequence with delayed feedback and after every $\delaybound+1$ iterations we return $\vt[\avg[\state]][\indg]=\sum_{\run=p_\indg}^{p_\indg+\delaybound}\vt[\state]/(\delaybound+1)$. This is a legitimate online algorithm because the knowledge of $\gvec_\indg$ is not required for playing $\vt[\avg[\state]][\indg]$.
Moreover, the regret achieved by $\algo$ on the constructed sequence is exactly $\delaybound+1$ times the regret achieved by $\algo_{/\delaybound}$ on the original sequence.

We now apply the the well known $\Omega(\sqrt{\len})$ lower bound for a horizon of $\len$ (see \eg \citealp{SS07}), and this proves the existence of a sequence of linear losses of length $\len$ and a corresponding $\comp$ with $\norm{\comp-\vt[\state][\start]}\le1$ such that the regret achieved by $\algo_{/\delaybound}$ is $\Omega(\sqrt{\len})$. Moreover, the loss vectors are either $1$ or $-1$.
Let us now considered the loss sequence constructed as in the previous paragraph.
The $(\delaybound+1)$-variation $\vt[\variation^{\delaybound^+}][\nRuns]$ is then bounded by $(\delaybound+1)+4(\delaybound+1)(\len-1)<\overline{\variation^{\delaybound}}$ and we have effectively $\nRuns = (\tau+1)\len$.
To finish, we observe that the regret achieved by $\algo$ on the constructed sequence is $\Omega((\delaybound+1)\sqrt{\len})$ and $(\delaybound+1)\sqrt{\len}\sim\sqrt{\delaybound\overline{\variation^\delaybound}}/2$ (where $\sim$ stands for asymptotically equivalent).
\end{proof}

\subsection{Delayed Online Learning with Slow Variation}
\label{apx:optimistic-slow-var}

\OptDelayRegretV*
\begin{proof}
The proof is immediate from \cref{thm:optimistic-delay-regret}.
Indeed,
\begin{equation}
    \notag
    \norm{\vt[\vecfield](\inter)-\vt[\appr[\vecfield]](\current)}^2
    \le
    2\norm{\vt[\vecfield](\inter)-\vt[\vecfield](\current)}^2
    +2\norm{\vt[\vecfield](\current)-\vt[\appr[\vecfield]](\current)}^2.
\end{equation}
Then, using the Lipschitz continuity of $\vt[\appr[\vecfield]]$ and the condition $2\vt[\stepalt^2]\lips^2\le1$, we have:
\begin{equation}
    \notag
    2\norm{\vt[\vecfield](\inter)-\vt[\vecfield](\current)}^2
    \le
    2\lips^2\norm{\inter-\current}^2
    =
    2\vt[\stepalt]^2\lips^2\norm{\vt[\appr[\vecfield]](\current)}^2
    \le\norm{\vt[\appr[\vecfield]](\current)}^2.
\end{equation}
In other words, we have proven $\norm{\inter[\gvec]-\inter[\appr[\gvec]]}^2-\norm{\inter[\appr]}^2\le2\norm{\vt[\vecfield](\current)-\vt[\appr[\vecfield]](\current)}^2$ and the bound follows.
\end{proof}

\subsection{More Flexible Learning Rates}
\label{apx:optimistic-adaptive}

In order prove \cref{prop:optimistic-adaptive}, we generalize both \cref{thm:optimistic-delay-regret} and \cref{thm:optimistic-delay-regret-V} to the case where the learning rate is non-increasing along a faithful permutation.

{
\addtocounter{theorem}{-1}
\renewcommand{\thetheorem}{\ref{thm:optimistic-delay-regret}$'$}
\begin{theorem}
\label{thm:optimistic-delay-regret-decen}
Assume that the maximum delay is bounded by $\delaybound$. 
Consider a faithful permutation $\Dpermu$ and let \acl{DOptDA} \eqref{eq:D-ODA} be run with learning rate sequences $\seqinf[\oneto{\nRuns}]{\step}{\run}$, $\seqinf[\oneto{\nRuns}]{\stepalt}{\run}$ satisfying $\vt[\step][\Dpermu(\run+1)]\le\vt[\step][\Dpermu(\run)]$ and $(4\delaybound+1)\max_{\setdef{\runalt}{|\runalt-\run|\le\delaybound}}\vt[\step][\runalt]\le\vt[\stepalt]$ for all $\run$. Then, the regret of the algorithm (evaluated at the points $\vt[\state][\frac{3}{2}],\ldots,\vt[\state][\nRuns+\frac{1}{2}]$) satisfies
\begin{equation}
\notag
    \reg_{\nRuns}(\comp)
    \le
    \frac{\norm{\comp-\vt[\state][\start]}^2}{2\vt[\step][\nRuns]}
    +\sum_{\run=\start}^{\nRuns}
    \frac{\vt[\stepalt]}{2}\left(\norm{\inter[\gvec]-\inter[\appr[\gvec]]}^2-\norm{\inter[\appr]}^2\right).
\end{equation}
\end{theorem}
}
\begin{proof}
We define the virtual iterates
\begin{equation}
    \notag
    \vt[\virtual] = \vt[\state][\start] - \vt[\step][\Dpermu(\run)]\sum_{\runalt=\start}^{\run-1}\inter[\gvec][\Dpermu(\runalt)].
\end{equation}
We then decompose
\begin{align}
\notag
    \vt[\obj](\inter) - \vt[\obj](\comp)
    \le \product{\inter[\gvec]}{\inter-\comp}
    = \product{\inter[\gvec]}{\vt[\state][\Dpermu(\run)+\frac{1}{2}]-\vt[\virtual][\run+1]}
    + \product{\inter[\gvec]}{\vt[\virtual][\run+1]-\comp}.
\end{align}
Following closely the proof of \cref{thm:optimistic-delay-regret}, we obtain
\begin{align}
    \notag
    \reg_{\nRuns}(\comp)
    &\le
    \frac{\norm{\comp-\vt[\state][\start]}^2}{2\vt[\step][\Dpermu(\nRuns)]}
    +\sum_{\run=\start}^{\nRuns}
    \frac{\vt[\stepalt]}{2}\left(\norm{\inter[\gvec]-\inter[\appr[\gvec]]}^2-\norm{\inter[\appr]}^2\right)\\
    \notag
    &~~+
    \left(-\frac{\vt[\stepalt][\Dpermu(\run)]}{2}+\frac{(\card(\vt[\setout^\Dpermu])+1)\vt[\step][\Dpermu(\run)]}{2}
    +\sum_{\Dpermu(\run)\in\vt[\setout^\Dpermu][\runano]}\frac{\vt[\step][\Dpermu(\runano)]}{2}\right)\norm{\vt[\gvec][\Dpermu(\run)+\frac{1}{2}]}^2.
\end{align}
Invoking \cref{prop:decent-delay-bound}, we know that $\setexclude{\oneto{\run}^{\Dpermu}}{\vt[\set][\Dpermu(\run)]}\subseteq\intinterval{\Dpermu(\run)-\delaybound}{\Dpermu(\run)+\delaybound}$.
Given that $\Dpermu(\run)\notin\oneto{\run-1}^{\Dpermu}$, this implies $\vt[\setout^{\Dpermu}]\subseteq\intinterval{\Dpermu(\run)-\delaybound}{\Dpermu(\run)-1}\union\intinterval{\Dpermu(\run)+1}{\Dpermu(\run)+\delaybound}$.
Therefore, $\card(\vt[\setout^{\Dpermu}])\le2\delaybound$ and if $\Dpermu(\run)\in\vt[\setout^\Dpermu][\runano]$ then $|\Dpermu(\run)-\Dpermu(\runano)|\le\delaybound$ while $\Dpermu(\run)\neq\Dpermu(\runano)$, which also shows $\card(\setdef{\runano}{\Dpermu(\run)\in\vt[\setout^\Dpermu][\runano]})\le2\delaybound$. Accordingly,
\begin{equation}
    \notag
    \frac{(\card(\vt[\setout^\Dpermu])+1)\vt[\step][\Dpermu(\run)]}{2}
    +\sum_{\Dpermu(\run)\in\vt[\setout^\Dpermu][\runano]}\frac{\vt[\step][\Dpermu(\runano)]}{2}
    \le \frac{(4\delaybound+1)\max_{\setdef{\runalt}{|\runalt-\Dpermu(\run)|\le\delaybound}}\vt[\step][\runalt]}{2}.
\end{equation}
With the assumption $\vt[\stepalt]\ge(4\delaybound+1)\max_{\setdef{\runalt}{|\runalt-\run|\le\delaybound}}\vt[\step][\runalt]$, we effectively deduce
\begin{equation}
    \notag
    \reg_{\nRuns}(\comp)
    \le
    \frac{\norm{\comp-\vt[\state][\start]}^2}{2\vt[\step][\Dpermu(\nRuns)]}
    +\sum_{\run=\start}^{\nRuns}
    \frac{\vt[\stepalt]}{2}\left(\norm{\inter[\gvec]-\inter[\appr[\gvec]]}^2-\norm{\inter[\appr]}^2\right).
\end{equation}
This proves the theorem.
\end{proof}

{
\addtocounter{theorem}{-1}
\renewcommand{\thetheorem}{\ref{thm:optimistic-delay-regret-V}$'$}
\begin{theorem}
\label{thm:optimistic-delay-regret-V-decent}
Let the maximum delay be bounded by $\delaybound$ and that \cref{asm:whole-vecfield} holds. Assume in addition that the vector fields  $\vt[\vecfield]$ are $\lips$-Lipschitz continuous.
Consider a faithful permutation $\Dpermu$ and take $\vt[\appr[\gvec]]=\vt[\appr[\vecfield]](\vt[\state])$, 
$\vt[\step][\Dpermu(\run+1)]\le\vt[\step][\Dpermu(\run)]$, $(4\delaybound+1)\max_{\setdef{\runalt}{|\runalt-\run|\le\delaybound}}\vt[\step][\runalt]\le\vt[\stepalt]$, and $2\vt[\stepalt^2]\lips^2\le1$. 
Then, the regret of \acl{DOptDA} \eqref{eq:D-ODA} (evaluated at the points $\vt[\state][\frac{3}{2}],\ldots,\vt[\state][\nRuns+\frac{1}{2}]$) satisfies
\begin{equation}
\notag
    \reg_{\nRuns}(\comp)
    \le
    \frac{\norm{\comp-\vt[\state][\start]}^2}{2\vt[\step][\Dpermu(\nRuns)]}
    +\sum_{\run=\start}^{\nRuns}
    \vt[\stepalt]\norm{\vt[\vecfield](\current)-\vt[\appr[\vecfield]](\current)}^2.
\end{equation}
\end{theorem}
}
\begin{proof}
Apply \cref{thm:optimistic-delay-regret-decen} and bound the term $\norm{\inter[\gvec]-\inter[\appr[\gvec]]}^2-\norm{\inter[\appr]}^2$ as in the proof of \cref{thm:optimistic-delay-regret-V}.
\end{proof}

\OptAdapt*
\begin{proof}
Let $\vt[\varappr]=\sum_{\runalt\in\vt[\set]}\norm{\vt[\vecfield][\runalt](\vt[\state][\runalt])-\vt[\appr[\vecfield]][\runalt](\vt[\state][\runalt])}^2$.
We consider a permutation $\Dpermu$ such that (\emph{i}) if $\vt[\varappr][\runalt]<\vt[\varappr]$ then $\Dpermu^{-1}(\runalt)<\Dpermu^{-1}(\run)$; (\emph{ii}) if $\vt[\varappr][\runalt]=\vt[\varappr]$ and $\runalt\in\vt[\set]$ then $\Dpermu^{-1}(\runalt)<\Dpermu^{-1}(\run)$. The sequence $(\vt[\varappr])_{\run}$ is non-decreasing along $\Dpermu$ (see \eg proof of \cref{prop:decen-decr-regret}) and 
accordingly the learning rate sequence $(\vt[\step])_{\run}$ is non-decreasing along $\Dpermu$.
Moreover, if $\runalt\in\vt[\set]$, we have $\vt[\set][\runalt]\subseteq\vwt[\setrec][\activeworker{\run}][\runalt]\subseteq\vt[\set]$ thanks to \cref{asm:in-order}.
This implies $\vt[\varappr][\runalt]\le\vt[\varappr]$; subsequently $\Dpermu^{-1}(\runalt)<\Dpermu^{-1}(\run)$.
The above shows that $\Dpermu$ is a faithful permutation.
The condition $2\vt[\stepalt^2]\lips^2\le1$ is automatically satisfied by the definition of $\vt[\stepalt]$.
To apply \cref{thm:optimistic-delay-regret-V-decent}, the last missing piece is to prove $(4\delaybound+1)\max_{\setdef{\runalt}{|\runalt-\run|\le\delaybound}}\vt[\step][\runalt]\le\vt[\stepalt]$.
This boils down to showing that
\begin{equation}
\label{eq:varappr-diff}
\vt[\varappr][\runalt]+4\gbound^2(3\delaybound+1)\ge\vt[\varappr]+4\gbound^2(\delaybound+1)
\end{equation}
for all $\runalt\in\oneto{\nRuns}\intersect\intinterval{\run-\delaybound}{\run+\delaybound}$.
The maximum delay being bounded by $\delaybound$, we have $|\card(\vt[\set][\runalt])-\card(\vt[\set])|\le|\runalt-\run|+\delaybound$.
By bounding each $\norm{\vt[\vecfield][\runano](\vt[\state][\runano])-\vt[\appr[\vecfield]][\runano](\vt[\state][\runano])}^2$ by $4\gbound^2$, we indeed prove \eqref{eq:varappr-diff} for $\runalt$ such that $|\runalt-\run|\le\delaybound$.

With all this at hand, applying \cref{thm:optimistic-delay-regret-V-decent} gives
\begin{equation}
\notag
    \reg_{\nRuns}(\comp)
    \le
    \frac{\norm{\comp-\vt[\state][\start]}^2}{2\vt[\step][\Dpermu(\nRuns)]}
    +\sum_{\run=\start}^{\nRuns}
    \vt[\stepalt]\norm{\vt[\vecfield](\current)-\vt[\appr[\vecfield]](\current)}^2.
\end{equation}
As the maximum delay is bounded by $\delaybound$ and the gradients are bounded by $\gbound$, we have $\vt[\varappr]+4\gbound^2(\delaybound+1)\ge\vt[\variation]$.
Invoking \cref{lem:adaptive} then gives
\begin{equation}
\label{eq:opt-adapt-regret-proof}
\begin{aligned}
    \reg_{\nRuns}(\comp)
    &\le
    \frac{\norm{\comp-\vt[\state][\start]}^2}{2\vt[\step][\Dpermu(\nRuns)]}
    +\frac{\radius\sqrt{4\delaybound+1}}{2}\sum_{\run=\start}^{\nRuns}
    \frac{1}{\sqrt{\vt[\variation]}}\norm{\vt[\vecfield](\current)-\vt[\appr[\vecfield]](\current)}^2\\
    &\le\frac{\radius^2}{2\vt[\step][\Dpermu(\nRuns)]}
    +\radius\sqrt{(4\delaybound+1)\vt[\variation][\nRuns]}.
\end{aligned}
\end{equation}
We bound the second term by
\begin{equation}
\label{eq:opt-adapt-regret-proof-2nd}
    \radius\sqrt{(4\delaybound+1)\vt[\variation][\nRuns]}
    \le
    \radius\sqrt{(4\delaybound+1)(\vt[\varappr][\nRuns]+4\gbound^2(3\delaybound+1))}
    \le\frac{\radius^2}{2\vt[\step][\nRuns]}
    \le\frac{\radius^2}{2\vt[\step][\Dpermu(\nRuns)]}.
\end{equation}
Combining \eqref{eq:opt-adapt-regret-proof} and \eqref{eq:opt-adapt-regret-proof-2nd} we get $\reg_{\nRuns}(\comp)\le\radius^2/\vt[\step][\Dpermu(\nRuns)]$.
We can conclude by using the definition of $\vt[\step][\Dpermu(\nRuns)]$ and $\vt[\varappr][\Dpermu(\nRuns)]\le\vt[\variation][\nRuns]$.
\end{proof} 
\section{Regret Bounds for Decentralized Delayed Dual Averaging}
\label{apx:global}
The proofs in this part leverage \cref{lem:regret-local-global}, which we recall below.

\RegLocalGlobal*

These proofs can thus be divided into two essential parts: a bound on the effective regret and a bound on the inter-agent discrepancies. For the first part we will utilize the change of index $\mapping(\worker,\run)=\vt[\nSamples][\run-1]+\worker$ introduced in \cref{subsec:DDDA}, where $\vt[\nSamples]=\sum_{\runalt=\start}^{\run}\vt[\nWorkers][\runalt]$ and $\nSamples=\vt[\nSamples][\nRuns]$.
We also recall the notations $\vt[\alt{\gvec}][\mapping(\worker,\run)]=\vwt[\gvec]$ and $\vt[\alt{\set}][\mapping(\worker,\run)]=\setdef{\mapping(\workeralt,\runalt)}{(\workeralt,\runalt)\in\vwt[\set]}$.

\subsection{Fixed Learning Rate}
\label{apx:globalfixed}

\GlobalRegretFixLR*

\begin{proof}
Let us start with \eqref{eq:global-reg-fix-lr-local}. Since the loss functions are $\gbound$-Lipschitz, the subgradients are bounded by $\gbound$.
\begin{align}
\notag
    \vt[\local[\reg]][\nRuns](\comp)
    &\le
    \frac{\hreg(\comp)}{\step}
    + \frac{1}{2}\sum_{\indsamp=\start}^{\nSamples}
    \step
    \left(\dnorm{\vt[\alt{\gvec}][\indsamp]}^2
        +2\dnorm{\vt[\alt{\gvec}][\indsamp]}
        \sum_{\runano\in\setexclude{\oneto{\indsamp-1}}{\vt[\alt{\set}][\indsamp]}}\dnorm{\vt[\alt{\gvec}][\runano]}\right)\\
    &\le
    \frac{\hreg(\comp)}{\step}
    + \frac{\step}{2}\sum_{\indsamp=\start}^{\nSamples}
    (1+2\card(\setexclude{\oneto{\indsamp-1}}{\vt[\alt{\set}][\indsamp]}))\gbound^2.
    \label{eq:delay-regret-local-prelim}
\end{align}
To bound $\card(\setexclude{\oneto{\indsamp-1}}{\vt[\alt{\set}][\indsamp]})$, we write $\indsamp=\mapping(\worker,\run)$.
On one hand, the subgradients
\[\{\vwt[\gvec][\worker-1][\run],\ldots,\vwt[\gvec][1][\run]\}=\{\vt[\alt{\gvec}][\indsamp-1],\ldots,\vt[\alt{\gvec}][\indsamp-\worker+1]\}\]
of instant $\run$ are necessarily unavailable when making the prediction $\vwt[\state][\worker][\run]=\vt[\alt{\state}][\indsamp]$.
On the other hand, the maximum delay assumption guarantees that all the subgradients received before time $\run-\delaybound$ are used in the computation of $\vwt[\state][\worker][\run]$.
This leads to the inequality
\begin{equation}
    \notag
    \card(\setexclude{\oneto{\indsamp-1}}{\vt[\alt{\set}][\indsamp]})
    \le \worker-1 + \sum_{\runalt=1}^{\delaybound}\vt[\nWorkers][\run-\runalt],
\end{equation}
with the convention $\vt[\nWorkers][\runano]=0$ if $\runano\le0$.
Subsequently, for any $\run\in\oneto{\nRuns}$,
\begin{equation}
    \sum_{\indsamp=\vt[\nSamples][\run-1]+1}^
    {\vt[\nSamples]}\card(\setexclude{\oneto{\indsamp-1}}{\vt[\alt{\set}][\indsamp]})
    \le \frac{\vt[\nWorkers](\vt[\nWorkers]-1)}{2}
    +\vt[\nWorkers]\sum_{\runalt=1}^{\delaybound}\vt[\nWorkers][\run-\runalt]
    \le \frac{(\delaybound+1)}{2}\vt[\nWorkers^2]
    +\frac{1}{2}\sum_{\runalt=1}^{\delaybound}\vt[\nWorkers^2][\run-\runalt].
    \label{eq:bound-outstanding}
\end{equation}
Substituting \eqref{eq:bound-outstanding} in \eqref{eq:delay-regret-local-prelim} then yields
\begin{align}
    \label{eq:regret-network-local}
    \vt[\local[\reg]][\nRuns](\comp)
    \le
    \frac{\hreg(\comp)}{\step}
    + \step(\delaybound+1)\gbound^2\sum_{\run=\start}^{\nRuns}
    \vt[\nWorkers^2].
\end{align}

We proceed to bound the difference $\norm{\vwt-\vwt[\state][\workeralt]}$ for all $\run\in\oneto{\nRuns}$ and $\worker,\workeralt\in\oneto{\vt[\nWorkers]}$.
In fact, we have $\vwt=\prox(-\vwt[\dvec])$ and $\vwt[\state][\workeralt]=\prox(-\vwt[\dvec][\workeralt])$ where $\vwt[\dvec]=\step\sum_{(\indg,\runalt)\in\vwt[\set]}\vwt[\gvec][\indg][\runalt]$ and $\vwt[\dvec][\workeralt]=\step\sum_{(\indg,\runalt)\in\vwt[\set][\workeralt]}\vwt[\gvec][\indg][\runalt]$.
From the maximum delay assumption we know that $\vwt[\set]$ and $\vwt[\set][\workeralt]$ differ by at most $\sum_{\runalt=1}^{\delaybound}\vt[\nWorkers][\run-\runalt]$ samples.
Using the $\gbound$-Lipshitz continuity of the loss functions and the non-expansiveness of the mirror map (\cref{lem:prox-nonexp}), we obtain
\begin{equation}
    \label{eq:regret-network-discrepancy}
    \sum_{\worker=1}^{\vt[\nWorkers]}
    \gbound\norm{\vwt-\vwt[\state][\workeralt]}
    \le \step\gbound^2\vt[\nWorkers]\sum_{\runalt=1}^{\delaybound}\vt[\nWorkers][\run-\runalt]
    \le \step\gbound^2\left(\frac{\delaybound\vt[\nWorkers^2]}{2}
    +\frac{1}{2}\sum_{\runalt=1}^{\delaybound}\vt[\nWorkers^2][\run-\runalt]\right).
\end{equation}

With \eqref{eq:regret-network-local} and \eqref{eq:regret-network-discrepancy}, invoking \cref{lem:regret-local-global} gives
\begin{align}
    \notag
    \vt[\glob[\reg]][\nRuns](\comp)
    \le
    \frac{\hreg(\comp)}{\step}
    + \step(2\delaybound+1)\gbound^2\sum_{\run=\start}^{\nRuns}
    \vt[\nWorkers^2].
\end{align}
The theorem follows immediately.
\end{proof}

\subsection{Learning Rates Based on the Number of Received Feedback Elements}
\label{apx:globalada}

As discussed in \cref{subsec:DDDA}, the learning rate proposed in \cref{prop:delay-regret-global} is generally not implementable in practice.
We will show below that a learning rate scheme similar to the one considered in \cref{subsec:nonada} equally guarantees low \acl{GPRg}.
To begin with, we rewrite \cref{asm:card-in-order} to accommodate the new notation.

{
\addtocounter{assumption}{-1}
\renewcommand{\theassumption}{\ref{asm:card-in-order}$'$}
\begin{assumption}
\label{asm:card-in-order-global}
If $(\workeralt,\runalt)\in\vwt[\set]$ then $\card(\vwt[\set][\workeralt][\runalt])<\card(\vwt[\set])$.
\end{assumption}
}

Under this assumption, we prove the following theorem which further extends the result of \cref{prop:decen-decr-regret}.

\begin{restatable}{proposition}{GlobalRegretVarLR}
\label{prop:delay-regret-global-oblivious}
Let \cref{asm:card-in-order-global} hold. Suppose that the maximum delay is bounded by $\delaybound$ and that all the loss functions are $\gbound$-Lipschitz.
Then, for any $\comp$ satisfying $\hreg(\comp)\le\radius^2$, \acl{D-DDA} \eqref{eq:D-DDA} with learning rates
\begin{equation}
\label{eq:delay-global-oblivious-stepsize}
\vwt[\step]=\frac{\radius}{\gbound\sqrt{(5\delaybound+3)(\card(\vwt[\set])+(\delaybound+1)\nWorkers_{\max})\nWorkers_{\max}}}
\end{equation}
guarantees a \acl{GPRg} in 
\begin{equation}
    \notag
    \vt[\glob[\reg]][\nRuns](\comp) = \bigoh(\sqrt{\delaybound\nSamples\nWorkers_{\max}}).
\end{equation}
\end{restatable}
\begin{proof}
With a slight abuse of notation, we will only work with the (worker, time) index pair in this proof, but it should be understood that the change of index $\mapping$ indeed intervenes implicitly when we apply the arguments of the previous sections (notably when we compare the indices).
Compared to \cref{prop:delay-regret-global}, the two additional difficulties here are:
\begin{enumerate*}
[\itshape i\upshape)]
    \item the non-monotonicity of learning rates which are solved by the introduction of a suitable faithful permutation;
    \item 
    the predictions of a time instant are not generated by the same learning rate, but we still manage to control the deviation since these learning rates are close enough.
\end{enumerate*}

To begin, we consider a permutation $\Dpermu$ satisfying $\Dpermu^{-1}(\workeralt,\runalt)<\Dpermu^{-1}(\worker,\run)$ if $\card(\vwt[\set][\workeralt][\runalt])<\card(\vwt[\set])$.
Such a $\Dpermu$ is necessarily faithful according to \cref{asm:card-in-order-global}.
We claim that $\card(\vt[\setout^{\Dpermu}][\Dpermu^{-1}(\worker,\run)])\le(\delaybound+1)\nWorkers_{\max}$ (where $\vt[\setout^{\Dpermu}][\Dpermu^{-1}(\worker,\run)]=\setexclude{\oneto{\Dpermu^{-1}(\worker,\run)-1}^{\Dpermu}}{\vwt[\set]}$).
Let $\runalt\in\intinterval{0}{\delaybound}$ such that $\vt[\nSamples][\run+\runalt-\delaybound]>\card(\vwt[\set])\ge\vt[\nSamples][\run+\runalt-\delaybound-1]$.
Then for any $\workeralt\in\oneto{\vt[\nWorkers][\run+\runalt+1]}$ it holds $\card(\vwt[\set][\workeralt][\run+\runalt+1])\ge\vt[\nSamples][\run+\runalt-\delaybound]>\card(\vwt[\set])$ and accordingly $\Dpermu^{-1}(\worker,\run)<\Dpermu^{-1}(\workeralt,\run+\runalt+1)$.
In other words, if $\Dpermu^{-1}(\indg,\runano)<\Dpermu^{-1}(\worker,\run)$ for some $\runano\in\oneto{\nRuns}$ and $\indg\in\oneto{\vt[\nWorkers][\runano]}$ then $\runano\le\run+\runalt$, and subsequently $\card(\oneto{\Dpermu^{-1}(\worker,\run)-1}^{\Dpermu})\le\vt[\nSamples][\run+\runalt]$.
We have therefore
\begin{equation}
\notag
    \card(\setexclude{\oneto{\Dpermu^{-1}(\worker,\run)-1}^{\Dpermu}}{\vwt[\set]})
    \le
    \vt[\nSamples][\run+\runalt]-\vt[\nSamples][\run+\runalt-\delaybound-1]
    =\sum_{\runano=0}^{\delaybound}\vt[\nWorkers][\run+\runalt-\runano]
    \le(\delaybound+1)\nWorkers_{\max}.
\end{equation}
Since $\vwt[\step]\le\vwt[\step][\workeralt][\runalt]$ if and only if $\card(\vwt[\set])\ge\card(\vwt[\set][\workeralt][\runalt])$, we have indeed $\vt[\step][\Dpermu((\worker,\run)+1)]\le\vt[\step][\Dpermu(\worker,\run)]$.
Invoking \cref{thm:delay-regret}, one has (notice that the sum is ordered differently as stated in the theorem)
\begin{align}
    \notag
    \vt[\local[\reg]][\nRuns](\comp)
    &\le
    \frac{\hreg(\comp)}{\vt[\step][\Dpermu(\vt[\nWorkers][\nRuns],\nRuns)]}
    + \frac{1}{2}
    \sum_{\run=\start}^{\nRuns}
    \sum_{\worker=1}^{\vt[\nWorkers][\run]}
    \vwt[\step]
    \left(\dnorm{\vwt[\gvec]}^2
        +2\dnorm{\vwt[\gvec]}
        \sum_{\runalt\in\vt[\setout^{\Dpermu}][\Dpermu^{-1}(\worker,\run)]}\dnorm{\vt[\gvec][\runalt]}\right)\\
    \label{eq:regret-local-oblivious}
    &\le
    \frac{\hreg(\comp)}{\min_{\run\in\oneto{\nRuns},\worker\in\oneto{\vt[\nWorkers]}}\vwt[\step]}
    + \frac{1}{2}
    \sum_{\run=\start}^{\nRuns}
    \left(\max_{\worker\in\oneto{\vt[\nWorkers]}}\vwt[\step]\right)\gbound^2(2\delaybound+3)\vt[\nWorkers]\nWorkers_{\max}.
\end{align}

In the second step we bound the difference $\norm{\vwt[\state]-\vwt[\state][\workeralt]}$ for $\worker,\workeralt\in\oneto{\vt[\nWorkers]}$.
Similar to the proof of \cref{prop:delay-regret-global}, we write $\vwt=\prox(-\vwt[\dvec])$ and $\vwt[\state][\workeralt]=\prox(-\vwt[\dvec][\workeralt])$ where $\vwt[\dvec]=\vwt[\step]\sum_{(\indg,\runalt)\in\vwt[\set]}\vwt[\gvec][\indg][\runalt]$ and $\vwt[\dvec][\workeralt]=\vwt[\step][\workeralt]\sum_{(\indg,\runalt)\in\vwt[\set][\workeralt]}\vwt[\gvec][\indg][\runalt]$.
By the non-expansiveness of the mirror map (\cref{lem:prox-nonexp}) it is then sufficient to bound $\norm{\vwt[\dvec]-\vwt[\dvec][\workeralt]}$.
For ease of notation, in the rest of the proof we will denote by $\setinter$ the intersection of $\vwt[\set]$ and $\vwt[\set][\workeralt]$, \ie $\setinter=\vwt[\set]\intersect\vwt[\set][\workeralt]$.
It follows
\begin{equation}
\label{eq:discrepancy-oblivious}
\begin{aligned}[b]
\norm{\vwt[\dvec]-\vwt[\dvec][\workeralt]}
&= \|(\vwt[\step]-\vwt[\step][\workeralt])\sum_{(\indg,\runalt)\in\setinter}\vwt[\gvec][\indg][\runalt]
+\vwt[\step]\sum_{(\indg,\runalt)\in\setexclude{\vwt[\set]}{\setinter}}\vwt[\gvec][\indg][\runalt]
-\vwt[\step][\workeralt]\sum_{(\indg,\runalt)\in\setexclude{\vwt[\set][\workeralt]}{\setinter}}\vwt[\gvec][\indg][\runalt]\|\\
&\le|\vwt[\step]-\vwt[\step][\workeralt]|\sum_{(\indg,\runalt)\in\setinter}\norm{\vwt[\gvec][\indg][\runalt]}
+\vwt[\step]\sum_{(\indg,\runalt)\in\setexclude{\vwt[\set]}{\setinter}}\norm{\vwt[\gvec][\indg][\runalt]}
+\vwt[\step][\workeralt]\sum_{(\indg,\runalt)\in\setexclude{\vwt[\set][\workeralt]}{\setinter}}\norm{\vwt[\gvec][\indg][\runalt]}\\
&\le\gbound(|\vwt[\step]-\vwt[\step][\workeralt]|\card(\setinter)+\max(\vwt[\step],\vwt[\step][\workeralt])\card(\symdiff{\vwt[\set]}{\vwt[\set][\workeralt]}))\\
&\le\gbound(|\vwt[\step]-\vwt[\step][\workeralt]|\vt[\nSamples][\run-1]+\max(\vwt[\step],\vwt[\step][\workeralt])\delaybound\nWorkers_{\max}).
\end{aligned}
\end{equation}
In the last inequality we use the fact that if one element belongs to one set but not the other then it must come from the last $\delaybound$ time steps.

To control $|\vwt[\step]-\vwt[\step][\workeralt]|$, we note that for any $b>a>0$, it holds
\begin{equation}
\notag
    \frac{1}{\sqrt{a}}-\frac{1}{\sqrt{b}}
    =\frac{b-a}{\sqrt{ab}(\sqrt{a}+\sqrt{b})}
    \le\frac{b-a}{2a\sqrt{a}}.
\end{equation}
For every $\indg\in\oneto{\vt[\nWorkers]}$, we have $\card(\vwt[\set][\indg])+(\delaybound+1)\nWorkers_{\max}\ge\vt[\nSamples]>\vt[\nSamples][\run-1]$.
Therefore, with the stepsize rule \eqref{eq:delay-global-oblivious-stepsize}, we get
\begin{equation}
\label{eq:lr-rate-difference}
    |\vwt[\step]-\vwt[\step][\workeralt]|
    \le \frac{\radius\thinspace|\card(\vwt[\set])-\card(\vwt[\set][\workeralt])|}{2\gbound\vt[\nSamples][\run-1]\sqrt{(5\delaybound+3)\vt[\nSamples]\nWorkers_{\max}}}
    \le \frac{\radius\delaybound\nWorkers_{\max}}{2\gbound\vt[\nSamples][\run-1]\sqrt{(5\delaybound+3)\vt[\nSamples]\nWorkers_{\max}}}.
\end{equation}
Let us denote $\vt[\step]=\radius/(\gbound\sqrt{(5\delaybound+3)\vt[\nSamples]\nWorkers_{\max}})$; then $\vwt[\step]\le\vt[\step]$ for all $\worker\in\oneto{\vt[\nWorkers]}$.
We also take
\[
\ubar{\step}=\frac{\radius}{\gbound\sqrt{(5\delaybound+3)(\nSamples\nWorkers_{\max}+(\delaybound+1)\nWorkers_{\max}^2)}}
\]
so that $\vwt[\step]\ge\ubar{\step}$ for all $\run\in\oneto{\nRuns},\worker\in\oneto{\vt[\nWorkers]}$.
We conclude with the help of \cref{lem:regret-local-global,lem:prox-nonexp,lem:adaptive}, and the inequalities \eqref{eq:regret-local-oblivious}, \eqref{eq:discrepancy-oblivious} and \eqref{eq:lr-rate-difference}:
\begin{equation}
\notag
    \begin{aligned}
    \vt[\glob[\reg]][\nRuns](\comp)
    &\le
    \frac{\hreg(\comp)}{\ubar{\step}}
    + \frac{1}{2}
    \sum_{\run=\start}^{\nRuns}
    \left(\vt[\step]\gbound^2(4\delaybound+3)\vt[\nWorkers]\nWorkers_{\max}
    +\frac{\radius\gbound\delaybound\vt[\nWorkers]\nWorkers_{\max}}{\sqrt{(5\delaybound+3)\vt[\nSamples]\nWorkers_{\max}}}
    \right)\\
    &=
    \frac{\hreg(\comp)}{\ubar{\step}}
    + \frac{1}{2}
    \sum_{\run=\start}^{\nRuns}
    \frac{\radius\gbound(5\delaybound+3)\vt[\nWorkers]\nWorkers_{\max}}{\sqrt{(5\delaybound+3)\vt[\nSamples]\nWorkers_{\max}}}\\
    &\le\radius\gbound\sqrt{(5\delaybound+3)(\nSamples\nWorkers_{\max}+(\delaybound+1)\nWorkers_{\max}^2)}
    +\radius\gbound\sqrt{(5\delaybound+3)\nSamples\nWorkers_{\max}}.
    \end{aligned}
\end{equation}
Accordingly, $\vt[\glob[\reg]][\nRuns](\comp) = \bigoh(\sqrt{\delaybound\nSamples\nWorkers_{\max}})$.
\end{proof}

Note that the bound of \cref{prop:delay-regret-global-oblivious} directly features the total number of actions taken in the full process; it is thus (at least partly) adaptive to the number of agents.
More importantly, since $\card(\vwt[\set])$ is obviously available to each agent at time $\run$, the learning rate \eqref{eq:delay-global-oblivious-stepsize} is indeed implementable by every single agent as long as the constants $\gbound$, $\delaybound$, and $\nWorkers_{\max}$ are known.
We leave the design of fully adaptive methods in the sense of \eqref{eq:adadelayDist} as an open question.

\footnotesize
\setlength{\bibsep}{\smallskipamount}
\bibliography{references,Bibliography-PM}

\end{document}